\documentclass{article}

\usepackage[final]{neurips_2025}
\usepackage{url}            %
\usepackage{booktabs}       %
\usepackage{amsfonts}       %
\usepackage{wrapfig}
\usepackage{caption}
\usepackage[mathscr]{euscript}
\usepackage{float}
\usepackage[nodisplayskipstretch]{setspace}

\usepackage{rotating}

\usepackage{times}
\usepackage{adjustbox}
\usepackage{dsfont}
\usepackage{amsmath, amssymb}
\usepackage{amsthm, thmtools, thm-restate}
\usepackage{bbm}
\usepackage{graphicx}
\usepackage{latexsym}
\usepackage{mathtools}
\usepackage{multirow}
\usepackage{paralist}

\usepackage{xcolor}
\usepackage{xspace}
\usepackage{enumitem}
\usepackage[most]{tcolorbox}
\tcbuselibrary{skins}

\newtheorem{theorem}{Theorem}%
\newtheorem{proposition}[theorem]{Proposition}
\newtheorem{lemma}[theorem]{Lemma}
\newtheorem{corollary}[theorem]{Corollary}

\newtheorem{definition}{Definition}
\newtheorem{assumption}[definition]{Assumption}

\newcommand{\myparagraph}[1]{\paragraph{#1.}\hspace{-0.8em}}  %

\usepackage{tikz}
\usepackage{pgfplots}
\usetikzlibrary{shapes}
\usetikzlibrary{positioning}
\usetikzlibrary{plotmarks}
\usetikzlibrary{patterns}
\usetikzlibrary{intersections,shapes.arrows}
\usetikzlibrary{pgfplots.fillbetween}
\definecolor{darkpink}{rgb}{0.91, 0.33, 0.5}
\pgfplotsset{compat=1.18}

\definecolor{puorange}{rgb}{0.80,0.40,0}
\definecolor{bluegray}{rgb}{0.04,0,0.7}
\definecolor{greengray}{rgb}{0.05,0.50,0.15}
\definecolor{darkbrown}{rgb}{0.40,0.2,0.05}
\definecolor{darkcyan}{rgb}{0,0.4,1}
\definecolor{black}{rgb}{0,0,0}
\definecolor{grey}{rgb}{0.93,0.93,0.93}
\definecolor{royalazure}{rgb}{0.0, 0.22, 0.66}

\newcommand{\red}[1]{{\color{purple}#1}}

\newcommand{\blue}[1]{{\color{royalazure}#1}}

\usepackage[colorlinks,citecolor=bluegray,linkcolor=darkbrown,urlcolor=blue,breaklinks]{hyperref}

\usepackage{minitoc} %

\usepackage[capitalise]{cleveref}

\crefname{section}{Sec.}{Sections}

\crefname{theorem}{Thm.}{Thms.}
\crefname{lemma}{Lem.}{Lems.}
\crefname{corollary}{Cor.}{Cors.}
\crefname{proposition}{Prop.}{Props.}
\crefname{assumption}{Asm.}{Asms.}
\crefname{property}{Propt.}{Propts.}
\crefname{algorithm}{Alg.}{Algs.}
\crefname{appendix}{Appx.}{Appxs.}

\crefname{figure}{Fig.}{Figs.}
\crefname{table}{Tab.}{Tabs.}

\definecolor{pastelgreen}{RGB}{227, 240, 227}
\colorlet{lightpastelgreen}{pastelgreen!60!white}
\definecolor{softblue}{RGB}{204, 225, 240}
\colorlet{lightersoftblue}{softblue!70!white}
\newtcolorbox{examplebox}[1][]{colback=gray!10, colframe=gray!10, sharp corners, boxrule=0pt, enhanced, breakable, #1}
\newtcolorbox{defibox}[1][]{colback=lightersoftblue, colframe=lightersoftblue, sharp corners, boxrule=0pt, enhanced, breakable, #1}

\newcommand{\abs}[1]{\left| #1 \right|}
\newcommand{\norm}[1]{\lVert #1 \rVert}

\newcommand{\Tr}{\operatorname{\bf Tr}}

\newcommand{\ip}[1]{{\langle #1 \rangle}}

\newcommand{\Var}{\operatorname{\mathbb{V}ar}}

\newcommand{\Cov}{\operatorname{\mathbb{C}ov}}

\DeclareMathOperator*{\argmin}{arg\,min}

\newcommand{\bigO}{\mathcal{O}}

\newcommand{\calA}{\mathcal{A}}
\newcommand{\calB}{\mathcal{B}}

\newcommand{\calD}{\mathcal{D}}
\newcommand{\calE}{\mathcal{E}}
\newcommand{\calF}{\mathcal{F}}
\newcommand{\calG}{\mathcal{G}}
\newcommand{\calH}{\mathcal{H}}

\newcommand{\calL}{\mathcal{L}}

\newcommand{\calN}{\mathcal{N}}

\newcommand{\calS}{\mathcal{S}}
\newcommand{\calW}{\mathcal{W}}
\newcommand{\calX}{\mathcal{X}}
\newcommand{\calY}{\mathcal{Y}}
\newcommand{\calZ}{\mathcal{Z}}

\newcommand{\ind}{\mathds{1}}

\newcommand{\wrt}{{\em w.r.t.~}}

\newcommand{\iid}{{\em i.i.d.~}}

\newcommand{\br}[1]{\ensuremath{\left\{#1\right\}}}
\renewcommand{\P}[2]{\ensuremath{{\mathbb P}_{#1}\left[#2\right]}}

\newcommand{\msc}[1]{\ensuremath{\mathscr{#1}}}
\newcommand{\prob}{\mathbb{P}}

\newcommand{\mc}[1]{\mathcal{#1}}
\newcommand{\sse}{\subseteq}
\newcommand \grad {\nabla}
\newcommand{\R}{\mathbb{R}}
\newcommand{\Ex}[2]{\ensuremath{{\mathbb E}_{#1}\left[#2\right]}}
\newcommand\E{\mathbb{E}}
\newcommand{\sbr}[1]{\ensuremath{\left[#1\right]}}
\newcommand{\p}[1]{\ensuremath{\left(#1\right)}}
\renewcommand{\d}{\mathrm{d}}

\newcommand{\pow}[1]{^{(#1)}}
\newcommand{\spow}[1]{^{\scriptscriptstyle(#1)}}

\newcommand{\gnorm}[1]{\ensuremath{\Vert#1\Vert_{\mathcal{G}}}}
\newcommand{\G}{\ensuremath{\mathcal{G}}}
\newcommand{\D}{\ensuremath{\mathrm{D}}}
\newcommand{\tr}[1]{\ensuremath{\text{tr}\left(#1\right)}}

\newcommand{\loss}{\ell} %
\newcommand{\ploss}{L}  %
\newcommand{\hg}{\hat{g}} %
\newcommand{\score}{S_\theta} %
\newcommand{\noscore}{S_{\text{no}}} %
\newcommand{\noscorehat}{\hat S_{\text{no}}} %
\newcommand{\sconv}{\mu} %
\newcommand{\smooth}{M} %
\newcommand{\pdir}{\Gamma} %
\newcommand{\pdirhat}{\hat\Gamma} %

\newcommand{\sgd}{\theta} %

\newcommand{\dsgd}{\theta} %

\newcommand{\dsgdbar}{\bar \theta} %

\newcommand{\Kconst}{K} %
\newcommand{\kappaconst}{\kappa} %

\newcommand{\secsmooth}{\alpha} %

\newcommand{\highsmooth}{\beta} %

\newcommand{\hessianconst}{\sconv_{\text{no}}} %
\newcommand{\hconst}{\smooth_{\text{no}}}

\newcommand{\Gr}{\G_{r}\p{g_0}} %
\newcommand{\Fro}{\mathrm{Fro}} %
\newcommand{\op}{\mathrm{op}}
\newcommand{\Q}{\mathbb{Q}}

\newcommand{\algoname}{OSGD\xspace}

\newcommand{\as}{\textit{a.s.~}}
\newcommand{\dno}{\delta_{\text{no}}}
\newcommand{\vno}{v_{\text{no}}}

\doparttoc
\faketableofcontents

\title{Stochastic Gradients under Nuisances}

\author{%
  Facheng Yu\thanks{email: \texttt{fachengy@uw.edu}.} \qquad  Ronak Mehta \qquad Alex Luedtke \qquad 
Zaid Harchaoui \vspace{0.3cm}\\
  University of Washington\\
}

\begin{document}

\maketitle

\begin{abstract}
    Stochastic gradient optimization is the dominant learning paradigm for a variety of scenarios, from classical supervised learning to modern self-supervised learning. We consider stochastic gradient algorithms for learning problems whose objectives rely on unknown nuisance parameters, and establish non-asymptotic convergence guarantees. Our results show that, while the presence of a nuisance can alter the optimum and upset the optimization trajectory, the classical stochastic gradient algorithm may still converge under appropriate conditions, such as Neyman orthogonality. Moreover, even when Neyman orthogonality is not satisfied, we show that an algorithm variant with approximately orthogonalized updates (with an approximately orthogonalized gradient oracle) may achieve similar convergence rates. Examples from orthogonal statistical learning/double machine learning and causal inference are discussed. 

\end{abstract}

\section{Introduction}\label{sec:intro}
Machine learning, statistics, and causal inference rely on risk minimization problems of the form
\begin{align}
    \min_{\theta \in \Theta} \: 
\big[
    \ploss_0(\theta) := \Ex{Z \sim \prob}{\loss_0(\theta; Z)}
   \big],
    \label{eq:erm}
\end{align}
where $\Theta \sse \R^d$ is a parameter space, $Z$ is a $\calZ$-valued random variable, and $\ell_0: \Theta \times \calZ \rightarrow \R$ is a loss function. The quantity $\ell_0(\theta; z)$ describes the performance of a model parametrized by $\theta \in \Theta$ on a test example $z \in \calZ$. Given only an oracle that provides a stochastic gradient estimate of the objective~\eqref{eq:erm}, practitioners are able to train models ranging from linear functions on tabular data to billion-parameter neural networks on vision and language data.

The success of stochastic gradient descent (SGD) algorithms \citep{amari1993backpropagation,bottou2005line,bottou2007tradeoffs,ward2020adagrad} has motivated an abundance of work on their theoretical properties under various algorithmic and risk conditions, such as class separability \citep{soundry2018theimplicit}, random reshuffling \citep{gurbuzbalan2021whyrandom}, decomposable objectives \citep{schmidt2017minimizing,vaswani2019painless}, quantization noise \citep{gorbunov20aunified}, and noise dominance \citep{sclocchi2024onthedifferent}. This success has been fueled by machine learning and AI software libraries such as JAX, PyTorch, TensorFlow, and others, which offer a wide range of SGD variants, as long as a loss function can be clearly specified. The gradient is then evaluated automatically on a mini-batch of datapoints and used for stochastic updates.

Though powerful, this recipe takes one thing for granted: that the learner can always compute the risk (or an unbiased estimate thereof). 
Indeed, many complex learning problems rely on a risk function that is only partially specified up to a class
\begin{align}
    \calL := \br{\ploss(\,\cdot\,, g): g \in \calG},\label{eq:loss_class}
\end{align}
where $\calG$ is a possibly infinite-dimensional set and $\ploss: \Theta \times \G \rightarrow \R$ is a function of both the target parameter $\theta \in \Theta$ and an unknown \emph{nuisance parameter}~$g \in \calG$. 

This framework originates from semiparametric estimation and inference~\citep{levit1979infinite,linnik2008statistical,bickel1993efficient,van2000asymptotic}, where the risk is a Kullback-Leibler (KL) divergence and $g$ provides information about the true data-generating distribution $\prob$, but is not of primary scientific interest. 
However, the partially specified loss formulation from~\eqref{eq:loss_class} is not limited to semiparametric estimation and inference problems. This framework connects to many areas of interest, including profile likelihood based learning~\citep{murphy2000onprofile, pavlichin2019approximate, hao2019thebroad} and distributionally robust learning~\citep{shapiro2017distributionally, Levy2020Large-Scale, mehta2024distributionally}. 

For instance, profile likelihood based learning reduces~\eqref{eq:loss_class} by applying a pointwise minimum over $g \in \calG$ to then construct a problem that can be solved in $\theta \in \Theta$. Another example arises in applications with distribution shifts, for which $g$ represents an unknown test data distribution that may differ from the one from which the training data were drawn. Distributionally robust learning reduces~\eqref{eq:loss_class} by instead taking a pointwise maximum over $\calG$ and solving the resulting problem. 
Although the pointwise minimum and maximum are natural reductions, it is often the case that there is a ``true'' $g_0 \in \calG$ and the loss class is reduced by first estimating $g_0$ with auxiliary data to produce some $\hg$, which we refer to as \emph{double/debiased machine learning}, or DML, following \citet{chernozhukov2018double}. The problem~\eqref{eq:erm} is then thought to be derived via $\ploss_0(\theta) \equiv \ploss(\theta, g_0)$ in this case (see examples in \Cref{sec:osl}). This is the focus of this paper. 

Despite the prominence of SGD and DML individually, the convergence guarantees of SGD to recover the risk minimizer with a misspecified nuisance parameter remain unknown. 
Indeed, after producing $\hg$, the user typically solves a (full batch) empirical risk minimization problem, i.e.~minimizing a sample average approximation of $\ploss(\cdot, \hg)$. In this paper, we aim to fill this gap by proving convergence guarantees on the sequence $(\theta\pow{n})_{n\geq 1}$ generated by updates of the form
\begin{align}
    \theta\pow{n} = \theta\pow{n-1} - \eta S(\theta\pow{n-1}, \hg;Z_n),\label{eq:sgd}
\end{align}
where $\eta > 0$ is a learning rate, $\calD_n := (Z_i)_{i=1}^n$ is a stream of independent data drawn from $\prob$, $\hg$ is a nuisance parameter estimate, and $S: \Theta \times \calG \times \calZ \rightarrow \R^d$ is a stochastic gradient oracle satisfying of $\E_{Z \sim \prob}[S(\theta, g; Z)] = \grad_\theta \ploss(\theta, g)$ for all $(\theta, g) \in \Theta \times \calG$.
In particular, when $\calG$ lies within a Banach space equipped with norm $\norm{\cdot}_\calG$, we wish to compare $\theta\pow{n}$ to 
\begin{align}\label{minimization problem}
   \theta_\star = \argmin_{\theta \in \Theta} \ploss(\theta, g_0),
\end{align} 
given conditions on the degree of misspecification $\norm{\hg - g_0}_{\calG}$ and (approximate) \emph{Neyman orthogonality} of the risk $L$ \citep{neyman1959probability}.

Intuitively, Neyman orthogonal classes of objectives are instances of~\eqref{eq:loss_class} whose curvature with respect to $\theta$ is insensitive to the choice of $g$ (see \Cref{sec:osl} for the formal description). When Neyman orthogonality is satisfied, the double machine learning framework is also known as \emph{orthogonal statistical learning (OSL)} \citep{zadik2018orthogonal,liu2022orthogonal,foster2023orthogonal}. In addition to the obvious computational considerations, we argue that the SGD perspective in this paper also sheds light on the methodological opportunities in DML/OSL. Indeed, while loss functions are typically specified by the chosen architecture, Neyman orthogonality is often achieved by specialized analytic calculations on the part of the user. Although this property is generally seen as a second-order property of the loss, it can also be viewed as a first-order property of the gradient oracle $S$. As we detail in \Cref{sec:optimization}, it may be easier and more aligned with the spirit of modern machine learning, to craft Neyman orthogonal gradient oracles instead of losses.

\myparagraph{Contributions}
We prove the first theoretical convergence guarantees for SGD under an unknown nuisance model.
We find that $\sgd\pow{n}$ converges linearly to a neighborhood of $\theta_\star$---the optimum in the well-specified case---with a radius that has a fourth-power (resp.~squared) dependence on $\Vert \hat g - g_0 \Vert_\G$ when Neyman orthogonality is (resp.~is not) satisfied. Our analysis can also apply to two-stream settings in which the nuisance parameter is learned online alongside the target. We further analyze a new algorithm, called orthogonalized SGD (\algoname), wherein the gradient oracle of a possibly non-orthogonal loss can be iteratively made orthogonal using an ``approximately orthogonalized'' gradient oracle, which is based on a separate estimation procedure. This algorithm enjoys a convergence guarantee that interpolates between the $\Vert \hat g - g_0 \Vert_\G^4$ (nuisance insensitive) and $\Vert \hat g - g_0 \Vert_\G^2$ (nuisance sensitive) regimes depending on the quality of the orthogonalizing operator. 

We provide an introduction to the OSL/DML setting in \Cref{sec:osl}. The SGD and OSGD algorithms are described and analyzed in \Cref{sec:optimization}. We discuss related work in \Cref{sec:discussion}; additional discussion can be found in \Cref{appx:discussion}. All proofs and numerical illustrations can be found in the Appendix.

\section{Orthogonal Statistical Learning}\label{sec:osl}

We first introduce various examples of risk functions in the form of~\eqref{eq:loss_class}, then formally introduce Neyman orthogonality and its implications. As is common in learning settings, the risk will be in the form of an expectation,
\begin{align*}
    \ploss(\theta, g) = \Ex{Z \sim \prob}{\loss(\theta, g; Z)},
\end{align*}
where $\loss: \Theta \times \calG \times \calZ \rightarrow \R$ is an instance-level loss function. Various assumptions used in the analysis in \Cref{sec:optimization} (e.g.~convexity) may be placed on either the loss $\loss$ or the risk $\ploss$. In each example, we provide the structure of the data point $Z$, the set $\calG$, and the loss $\loss$, and the true $g_0 \in \calG$ to fully specify the problem. Here, we interpret ``true'' to mean that $g_0$ is a parameter of the data-generating distribution (e.g.~a propensity score in causal inference), or that $g_0$ satisfies a cost-minimizing or utility-maximizing criterion (as in the profile likelihood or distributional robustness examples from \Cref{sec:intro}).

\begin{table}[ht]
    \centering

    \begin{adjustbox}{max width=1.0\linewidth}
    \renewcommand{\arraystretch}{1.15}
    \begin{tabular}{ccc}
    \toprule
        {\bf Example} & { $\ell(\theta, g; z)$} &{ $g_0$} \\
        \midrule
        PLM & $\frac{1}{2}(y - g_Y(w) - \ip{\theta, x - g_X(w)})^2$ & $(\Ex{\prob}{Y \mid W}, \Ex{\prob}{X \mid W})$\\
        CATE & $\frac{1}{2}\p{g\spow{1}(x) - g\spow{0}(x) + \frac{(w-g^{\text{prop}}(x))(y - g\spow{w}(x))}{g^{\text{prop}}(x)(1-g^{\text{prop}}(x))} - \ip{\theta, x}}^2$ & $(\Ex{\prob}{Y \mid W=1, X}, \Ex{\prob}{X \mid W=0, X}, \Ex{\prob}{W \mid X})$\\
        CRR & $- \sbr{\mu_g\spow{1}(z) \log p_\theta(x) + \mu_g\spow{0}(z) \log(1 - p_\theta(x))}$ & $(\Ex{\prob}{Y \mid W=1, X}, \Ex{\prob}{X \mid W=0, X}, \Ex{\prob}{W \mid X})$
        \\
        \bottomrule
    \end{tabular}
    \end{adjustbox}
    \vspace{6pt}
    \caption{Examples of Neyman Orthogonal Losses.}
    \label{tab:eg main}
\end{table}

\begin{examplebox}
    \textbf{Example 1} (Partially Linear Model)\textbf{.} Let $Z = (X, Y, W) \sim \prob$, where $X$ is an $\R^d$-valued input, $Y$ is a real-valued outcome, and $W$ is a $\calW$-valued control or confounder. The space $\calG$ is a nonparametric class containing functions of the form
    \begin{align*}
        g = (g_Y, g_X): \calW \rightarrow \R \times \R^d. 
    \end{align*}
    Following the construction of \citet{robinson1988root}, this $g$ is supplied to the loss
    \begin{align*}
        \loss_{\text{PLM}}(\theta, g; z) = \frac{1}{2}(y - g_Y(w) - \ip{\theta, x - g_X(w)})^2.
    \end{align*}
    To ensure $\theta_\star$ can be interpreted via the projection of $\E_{\prob}[Y|X,W]$ onto partially linear additive functions, the true nuisance is given by $g_0 = (g_{0, X}, g_{0, Y})$, where
    \begin{align*}
        g_{0, Y}(w) := \Ex{\prob}{Y \mid W = w} \text{ and } g_{0, X}(w) := \Ex{\prob}{X \mid W = w}.
    \end{align*}
\end{examplebox}

 The next example concerns a quantity widely studied in causal inference \citep{kennedy2023towards}.
\begin{examplebox}
    \textbf{Example 2} (Conditional Average Treatment Effect)\textbf{.}  We observe $Z = (X, Y, W)\sim \prob$, where $W$ is a binary treatment assignment. The functions in $\calG$ are of the form
    \begin{align*}
        g = (g\spow{0}, g\spow{1}, g^{\text{prop}}): \R^d \rightarrow \R \times \R \times (0, 1),
    \end{align*}
    and are evaluated (see \citet[Thm. 1]{van2014targeted}) at the loss
    \begin{align*}
        \loss_{\text{CATE}}(\theta, g; z) = \frac{1}{2}\p{g\spow{1}(x) - g\spow{0}(x) + \frac{w-g^{\text{prop}}(x)}{g^{\text{prop}}(x)(1-g^{\text{prop}}(x))}(y - g\spow{w}(x)) - \ip{\theta, x}}^2.
    \end{align*}
    For $g_0 = (g\spow{0}_0, g\spow{1}_0, g^{\text{prop}}_0)$ nuisance functions $g\spow{0}_0$ and $g\spow{1}_0$ represent the outcome regressions
    \begin{align*}
        g\spow{0}_0(x) := \E_{\prob}[Y \mid W=1, X = x] \text{ and } g\spow{1}_0(x) := \E_{\prob}[Y \mid W=0, X = x],
    \end{align*}
    whereas $g^{\text{prop}}_0(x) := \Ex{\prob}{W \mid X = x}$ denotes the propensity score. The minimizer $\theta_\star$ indexes a projection of the conditional average treatment effect $g\spow{1}_0-g\spow{0}_0$ onto linear functions.
\end{examplebox}

Finally, we maintain the data structure from the previous example, but consider a loss corresponding to a different target parameter according to \cite{van2024combining}.
\begin{examplebox}
    \textbf{Example 3} (Conditional Relative Risk)\textbf{.}
    We retain all components of the previous example, changing only the loss and assuming that the outcome $Y$ is binary/non-negative. First, consider the ``label'' function
    \begin{align*}
        \mu_g\spow{s}(z) = g\spow{s}(x) + \frac{\ind(w = s)}{sg^{\text{prop}}(x) + (1-s) (1-g^{\text{prop}}(x))}(y - g\spow{s}(x)),
    \end{align*}
    where $\ind(\cdot)$ denotes the indicator function, and the logit-linear predictor $ p_\theta(x) = e^{\ip{\theta, x}}/(1 + e^{\ip{\theta, x}})$.
    To obtain a linear approximation of the log-relative risk $\log (g\spow{1}_0/g\spow{0}_0)$, we employ the cross entropy-type loss function
    \begin{align*}
        \loss_{\text{CRR}}(\theta,  g; z) = - \sbr{\mu_g\spow{1}(z) \log p_\theta(x) + \mu_g\spow{0}(z) \log(1 - p_\theta(x))}.
    \end{align*}
\end{examplebox}

While the choices of the loss function in Examples 2 and 3 might look opaque to readers outside of causal inference and statistics, they are carefully designed to be \emph{Neyman orthogonal}. To motivate its definition, notice that, invariably, $g_0$ is unknown to the user. In DML, the user may produce or access some $\hg \in \calG$, which is an estimate of $g_0$ based on independent training data other than the stream $(Z_i)_{i=1}^n$ used to produce $\theta\pow{n}$. It is of clear interest how stochastic optimization algorithms (and their resulting minimizers) behave in light of the misspecification of $g_0$, and what precise theoretical conditions govern this behavior. Moreover, as we demonstrate in \Cref{sec:optimization}, these same conditions can be used to analyze procedures for which the user may access additional data to progressively improve the estimate $\hg$ and learn $\theta_\star$ simultaneously. We now formally introduce Neyman orthogonality, and by extension, the orthogonal statistical learning (OSL) variant of DML.

\myparagraph{Neyman Orthogonality}
For a definition that accounts for a possibly infinite-dimensional function class $\calG$, we introduce the \emph{directional derivative}, or equivalently, the \emph{derivative operator}.
\begin{defibox}
    \begin{definition}[Derivative Operator]\label{def derivative}
        For a functional $F$ mapping from a vector space $\calF$ to $\R$, we define the \textit{(directional) derivative operator} $\D$ as $\D F(f)[h] := \frac{\d}{\d t} F(f + th) \mid_{t=0}$ for any $f, h \in \calF$. For a vector-valued $F: \calF \mapsto \R^d$, this derivative operator can be generalized by taking derivatives coordinate-wise. We define the second-order derivative as $\D^2 F(f)[h, h'] := \D(\D F(f)[h])[h']$ for $h, h' \in \calF$ and higher-order derivatives similarly. For functionals of multiple variables $F: \calF \times \calG \rightarrow \R$, we use the subscript notation $\D_f F(f, g)[h]$ to indicate the directional derivative of $f \mapsto F(f, g)$ with $g \in \calG$ fixed.
    \end{definition}
\end{defibox}
We denote by $\score(\theta, g;z) = \nabla_\theta \loss(\theta, g;z)$ the gradient of the loss function \wrt the target parameter $\theta \in \Theta$. 
Borrowing terminology from statistics, we call this the \emph{score}, whether $\loss$ is based on a likelihood or not.\footnote{This notion of (Fisher) score differs from the ``score'' used in score-based generative modeling \citep{song2021scorebased}. If $\loss$ is based on a log-likelihood, then $S_\theta$ is the gradient \wrt the parameter $\theta \in \Theta$, \emph{not} the input $z \in \calZ$.} This constitutes one particular example of a stochastic gradient oracle $S$ used in~\eqref{eq:sgd}. Overloading notation, the \emph{population gradient oracle} is defined as $\score(\theta, g) = \E_{Z\sim\prob}[\nabla_\theta \loss(\theta, g;Z)]$. 

\begin{defibox}
\begin{definition}[Neyman Orthogonality]\label{def neyman orthogonal}
    The population gradient oracle $\score$ is \emph{Neyman orthogonal} at $(\theta_\star, g_0)$ over $\G' \sse \calG$ if
    \begin{align}\label{eq neyman orthogonal 1}
          \D_g \score(\theta_\star, g_0)[g - g_0] = 0  \quad \text{ for all } g \in \G'.
    \end{align}
    For $\Theta' \sse \Theta$, the population loss $\ploss$ is \emph{Neyman orthogonal} at $(\theta_\star, g_0)$ over $\Theta' \times \G'$ if 
    \begin{align} \label{eq neyman orthogonal 2}
          \D_g \D_\theta \ploss(\theta_\star, g_0)[\theta - \theta_\star ,g - g_0] = 0  \quad \text{ for all } (\theta, g) \in \Theta' \times \G'.
    \end{align}
\end{definition}
\end{defibox}
In Definition \ref{def neyman orthogonal}, we allow $\Theta' \times \G' \subseteq \Theta \times \G$ to be a proper subset, which not only provides a weaker condition, but also accounts for localization-style arguments. Moreover, since $\D_\theta \ploss(\theta_\star, g_0)[\theta - \theta_\star] = \ip{\score(\theta_\star, g_0), \theta - \theta_\star}$, if the population risk $\ploss$ satisfies \eqref{eq neyman orthogonal 1}, then \eqref{eq neyman orthogonal 2} holds for any target parameter class $\Theta' \subseteq \R^d$. 
As mentioned above, the risk functions in Examples 1–3 are all Neyman orthogonal at their respective value of $(\theta_\star, g_0)$. In the next section, we will discuss a procedure to make a non-orthogonal gradient oracle ``approximately'' orthogonal. We illustrate the intended outcome intuitively in \Cref{fig:osl} by comparing a generic loss and its orthogonalized counterpart.

\begin{figure}
    \centering
    \includegraphics[width=\linewidth, trim=50 10 80 10, clip]{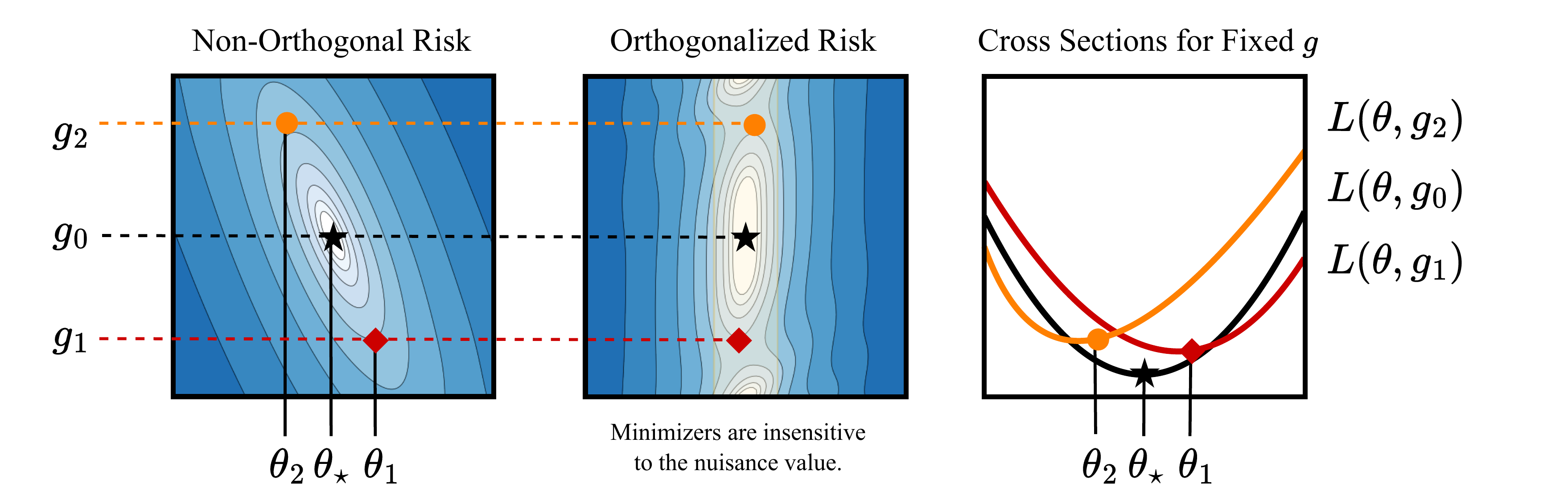}
    \caption{{\bf Illustration of Neyman Orthogonalization.} The first two panels are contour plots of the risk function $L(\theta, g)$, where $\theta$ varies on the $x$-axis and $g$ varies on the $y$-axis. For the orthogonalized risk (center) the contours are approximately axis-aligned. The right plot shows the cross sections of the non-orthogonal risk when fixing $g = g_0, g_1, g_2$. Due to non-orthogonality, the minimizers $\theta_1$ and $\theta_2$ shown in the first and third plots may drift significantly from $\theta_\star$. In contrast, the minimizers in the center plot are less sensitive to the choice of $g$.}
    \label{fig:osl}
\end{figure}

\section{Stochastic Gradient Optimization}\label{sec:optimization}
In this section, we propose two stochastic gradient algorithms, which rely on different choices of the stochastic gradient oracle $S$ used in~\eqref{eq:sgd}. The first is the familiar stochastic gradient oracle that provides a sample estimate of the gradient $\grad \ploss(\cdot, \hg)$ for a fixed estimate $\hg$. The second employs an \emph{approximately orthogonalized gradient oracle}, or \algoname oracle, to achieve a notion of approximate Neyman orthogonality (in a manner we make precise in this section). We analyze the first algorithm under both non-orthogonal and orthogonal settings, achieving an illustrative breakdown of ``nuisance sensitive'' and ``nuisance insensitive'' regimes regarding the theoretical convergence guarantee. For the \algoname algorithm, we prove a convergence guarantee that interpolates between the two regimes, depending on the accuracy of the oracle. 

\myparagraph{Notation and Assumptions}
For readers’ convenience, a table of all the notation we introduce throughout the paper is collected in \Cref{sec:a:notation}. We maintain the prototypical bias/variance conditions on the stochastic gradient oracle $S$, that is, an \blue{unbiasedness condition} and a \red{second-moment growth condition} (see \Cref{assumption of population risk}(d)).
To dispel confusion, note that by ``unbiased'', we mean specifically that $\blue{\Ex{Z \sim \prob}{S(\theta, g; Z)} = \grad_\theta \ploss(\theta, g)}$ for all $(\theta, g)$, as opposed to the ``bias'' of replacing $\hg$ with $g_0$, a terminology sometimes used in DML/OSL.
Our analysis will rely partly on the initial distance $r$ of the nuisance estimate $\hat g$ to $g_0$ in $\G$, which defines the ball
\begin{align}
    \Gr = \br{g \in \G: \gnorm{g - g_0} \leq r}.
\end{align}
Various assumptions on the risk will be required to hold only locally, that is, within $\Gr$ as opposed to the entire linear space $\G$. Thus, the assumptions become weaker as the estimate $\hg$ improves. 
\begin{assumption}\label{assumption of population risk}
The following conditions hold:
\begin{enumerate}[label=(\alph*)]
    \item \emph{Differentiability:} For any $(z, g) \in \calZ \times \G$, $\theta \mapsto \loss(\theta, g;z)$ is twice continuously differentiable. For any $(\theta, g), (\bar \theta, \bar g) \in \Theta \times \G$, (i) $\D_g^2\D_\theta \loss(\bar \theta, \bar g;z)[ \theta - \theta_\star,  g - g_0,  g - g_0]$ exists and is continuous, (ii) $\D_\theta L(\bar \theta, \bar g)[\theta - \theta_\star]$ and $\D_g L(\bar \theta, \bar g)[g - g_0]$ exist, and (iii) $\Ex{\prob}{\D_\theta \ell(\bar \theta, \bar g; Z)[\theta - \theta_\star]} = \D_\theta L(\bar \theta, \bar g)[\theta - \theta_\star]$, $\Ex{\prob}{\D_g \ell(\bar \theta, \bar g; Z)[g - g_0]} = \D_g L(\bar \theta, \bar g)[g - g_0]$.
    \item \emph{First-order optimality:}  The pair $(\theta_\star, g_0)$ satisfies $\score(\theta_\star, g_0) = 0$.
    \item \emph{Smoothness and strong convexity:} There exist constants $\smooth \geq \sconv>0$ such that for all $g \in \G_{r}\left(g_0\right)$, the population risk $\ploss(\cdot, g)$ is $\smooth$-smooth and $\sconv$-strongly convex for $\theta \in \Theta$.
    \item \emph{Second-moment growth:} There exist constants $\red{\Kconst_1}, \red{\kappaconst_1} \geq 0$ such that
    \begin{align*}
        \E_{Z \sim \prob}[\Vert\score(\theta, g; Z) - \score(\theta, g)\Vert_2^2] \leq \red{\Kconst_1} + \red{\kappaconst_1}\norm{\theta - \theta_\star}_2^2 \quad \forall\theta \in \Theta, g \in \G_{r}\left(g_0\right).
    \end{align*}
    \item \emph{Second-order smoothness:} There exists a constant $\secsmooth_1 \geq 0$ such that
    $$
    \left|\mathrm{D}_g \mathrm{D}_\theta L\left(\theta_{\star}, \bar{g}\right)\left[\theta-\theta_{\star}, g-g_0\right]\right| \leq \secsmooth_1\left\Vert\theta-\theta_{\star}\right\Vert_{2}\left\Vert g-g_0\right\Vert_{\G} \quad \forall \theta \in \Theta, g, \bar{g} \in \G_{r}\left(g_0\right).
    $$
\end{enumerate}
\end{assumption}
\Cref{assumption of population risk} does not require Neyman orthogonality at $(\theta_\star, g_0)$. Instead, \Cref{assumption of population risk}(a) is a standard differentiability condition.  \Cref{assumption of population risk}(b) and (c) implies that $\theta_\star$ is a unique global minimizer. 
\Cref{assumption of population risk}(d) generalizes the uniformly bounded second moment condition in stochastic optimization (e.g.~\citet{cutler2023stochastic}) by adding a quadratic form in $\theta$, which allows us to consider an unbounded feasible set $\Theta$ and more loss classes. Finally, \Cref{assumption of population risk}(d) and (e) can be satisfied when the Hessian of the population risk is a bounded operator. Usually, $\Kconst_1$, $\kappaconst_1$, and  $\secsmooth_1$ would depend on the initial nuisance estimation distance $r$. We provide in~\Cref{sec:a:examples}~estimates of the constants in~\Cref{assumption of population risk} for each motivating example. We proceed to the main results regarding the convergence of SGD and \algoname.

\myparagraph{Stochastic Gradient Algorithm}
Here, we use the standard single-sample stochastic gradient estimate $S = S_\theta$ in~\eqref{eq:sgd}.
This leads to the update
\begin{align}\label{baseline SGD procedure}
    \sgd\pow{n} = \sgd\pow{n-1} - \eta S_\theta(\sgd\pow{n-1}, \hat g; Z_n), \quad \sgd\pow{0} \in \Theta.
\end{align}
While the SGD procedure can be easily extended to using a batch of unbiased gradient estimates, we keep our single-observation construction to highlight the most important aspects of the analysis.
In order to achieve quantitative guarantees in the Neyman orthogonal setting, which essentially removes certain second-order terms that include $\theta$ and $g$, we will consider the following higher-order condition in some cases.
\begin{assumption}[Higher-Order Smoothness] \label{assumption orthogonality main}
     The risk $L$ satisfies \Cref{def neyman orthogonal} at $(\theta_\star, g_0)$, and there exists some constant $ \highsmooth_1 > 0$ such that
\begin{align*}
    \left|\mathrm{D}_g^2 \mathrm{D}_\theta L\left(\theta_{\star}, \bar{g}\right)\left[\theta-\theta_{\star}, g-g_0, g-g_0\right]\right| \leq \highsmooth_1\left\Vert \theta-\theta_{\star}\right\Vert_{2}\left\Vert g-g_0\right\Vert_{\G}^2 \quad \forall \theta \in \Theta, g, \bar{g} \in \G_{r}\left(g_0\right).
\end{align*}
\end{assumption}
When satisfied, \Cref{assumption orthogonality main} results in the nuisance insensitivity alluded to at the beginning of this section.
Notice that Neyman orthogonality is not necessary to construct a stochastic optimizer, and it is still possible to obtain a nuisance sensitive rate under only \Cref{assumption of population risk}. We demonstrate this in \Cref{theorem convergence rate baseline SGD}.
\begin{theorem}\label{theorem convergence rate baseline SGD}
    Define $\mc{D}_n = (Z_1, \ldots, Z_n)$, sampled from the product measure $\prob^n$. Suppose that \Cref{assumption of population risk} holds, $\hat g \in \Gr$ is estimated independently of $\calD_n$, and $\sgd\pow{0}, \ldots, \sgd\pow{n} \in \Theta$ almost surely.  The iterates of~\eqref{baseline SGD procedure} satisfy:
    \begin{enumerate}
        \item \red{\emph{Nuisance sensitive:}} If $\eta \leq \sconv/2(M\sconv + \kappaconst_1)$, then
        \begin{align*}
             \E_{\mc{D}_{n} \sim \prob^n}[\Vert\sgd\pow{n} - \theta_\star\Vert_2^2]  \leq& \p{1-\frac{\sconv \eta}{2}}^n\Vert\theta\pow{0} - \theta_\star\Vert_2^2+ \red{\frac{2\secsmooth_1^2}{\sconv^2} \Vert\hat g - g_0\Vert_\G^2} + \frac{4\Kconst_1 \eta}{\sconv}.
        \end{align*}
        \item \blue{\emph{Nuisance insensitive:}} If \Cref{assumption orthogonality main} also holds, then, for $\eta \leq \sconv/2(M\sconv + \kappaconst_1)$,
        \begin{align*}
             \E_{\mc{D}_{n} \sim\prob^n}[\Vert\sgd\pow{n} - \theta_\star\Vert_2^2]   \leq& \p{1-\frac{\sconv\eta}{2}}^n\Vert\theta\pow{0} - \theta_\star\Vert_2^2 + \blue{\frac{\highsmooth_1^2}{2\sconv^2} \Vert\hat g - g_0\Vert_\G^4} + \frac{4\Kconst_1 \eta}{\sconv}.
        \end{align*}
    \end{enumerate}
\end{theorem}

Note that the assumption that the iterates remain in $\Theta$ is satisfied in common cases. It is satisfied trivially for the first two examples in \Cref{sec:osl} because $\Theta = \R^d$. Another case is when the loss decomposes into the sum of a $G$-Lipschitz continuous component and the $\ell_2^2$-norm regularizer, i.e. $\loss(\theta, g; z) = h(\theta, g; z) + \frac{\mu}{2}\norm{\theta}_2^2$. Then, the iterates and the optimum remain in $\br{\theta: \norm{\theta}_2 \leq G/\mu}$ (see, e.g., \citet[Appx. C]{mehta2023stochastic}), so \Cref{def neyman orthogonal} can be restricted to this compact set.

\Cref{theorem convergence rate baseline SGD} states that SGD converges linearly to a ball around $\theta_\star$ with a radius that depends on the bias (due to the replacement of $g_0$ with $\hg$) and the variance due to gradient noise. Moreover, the variance component decays proportionally to the learning rate $\eta$. Under \Cref{assumption orthogonality main}, the bias component can have a significantly more favorable scaling with the error in the nuisance estimate $\Vert\hat g - g_0\Vert_\G$---specifically, $\Vert\hat g - g_0\Vert_\G^4$ instead of $\Vert\hat g - g_0\Vert_\G^2$. A similar breakdown into two regimes of the bias scaling occurs in the works of both \citet{foster2023orthogonal} and \citet{liu2022orthogonal} under \Cref{assumption orthogonality main} (called ``slow rate'' and ``fast rate'' there). Importantly, their bounds are based on an exact, offline empirical risk minimization procedure for a fixed training set, i.e.~they provide excess risk bounds on the quantity $L(\hat{\theta}_n, g_0) - L(\theta_\star, g_0)$, where
\begin{align*}
    \hat{\theta}_n = \argmin_{\theta \in \Theta} \frac{1}{n}\sum_{i=1}^n \ell(\theta, \hg; Z_i).
\end{align*}
In contrast, \Cref{theorem convergence rate baseline SGD} accounts for both the expected distance to optimum and the interplay between bias incurred by $\hg$ and the progress achieved at each step. In particular, even when using a constant learning rate, the bias does not accrue on each iterate and is in fact constant in $n$. When $\hat{\theta}_n$ is designed to be doubly robust, using SGD can achieve \emph{double robustness}; see \Cref{appx discuss sec 55} for an example.

\myparagraph{Orthogonalized Stochastic Gradient Algorithm}
Given the marked improvement in the rate of decay of the bias term when an orthogonal loss is used, it is clearly beneficial to do so when possible. %
We now describe how we can induce orthogonality by adjusting the stochastic gradient oracle using the solution of an auxiliary problem. %

The construction of Neyman orthogonal losses has historically been motivated in semiparametric theory and statistical learning as a means to build efficient -- minimum asymptotic variance -- \emph{full batch} statistical estimators~\citep{tsiatis2006semiparametric, van2000asymptotic, foster2023orthogonal, chernozhukov2018plug}. The approach we follow is inspired by the construction reviewed in~\citet[Section 2.2]{chernozhukov2018double}; see also~\citet{luedtke2024simplifying}. We also give an intuitive explanation based on least-squares estimation, instead of the usual differential/information geometry one.

While our construction holds in general spaces, let us first consider the illustrative case when $\G = \R^k$. 
At the true parameters $(\theta_\star, g_0)$, consider the problem of finding the best predictor of the $\R^d$-valued target variable $S_\theta(\theta_\star, g_0; Z) = \grad_\theta \ell(\theta_\star, g_0; Z)$ given the $\R^k$-valued predictor $\grad_g \ell(\theta_\star, g_0; Z)$ variable in the space $\calL(\G, \Theta)$ containing all continuous and linear operators from $\G$ to $\Theta$:
\begin{align}
     \pdir_0 = \argmin_{\pdir \in \calL(\G, \Theta)} \Ex{\prob}{\norm{S_\theta(\theta_\star, g_0; Z) - \pdir \nabla_g \loss(\theta_\star, g_0; Z)}_2^2}. \label{eq:regression}
\end{align}
In the special case where $\loss(\theta, g; z) = -\log p_{\theta,g}(z)$ for a density $p_{\theta, g}$ on $\calZ$ that governs the random variable $Z$, the projection direction solving~\eqref{eq:regression} can be shown to satisfy $\pdir_0 = H_{\theta g}^\top H_{g g}^{-1}$, where $H_{\theta g} = \nabla_g\score(\theta_\star, g_0) \in \R^{k \times d}$ is the transposed Jacobian and $H_{gg} = \nabla_g^2 \ploss(\theta_\star, g_0) \in \R^{k \times k}$ is the Hessian.
The prediction $\pdir_0 \nabla_g \loss(\theta_\star, g_0; Z)$ accounts for the covariance between $\grad_\theta \ell(\theta_\star, g_0; Z)$ (the gradient \wrt $\theta$) and $\grad_g \ell(\theta_\star, g_0; Z)$ (the gradient \wrt $g$). It stands to reason that as $\theta \rightarrow \theta_\star$, the random vector
\begin{align}
    S(\theta, g_0; Z) &:= S_\theta(\theta, g_0; Z) - \pdir_0 \nabla_g \loss(\theta, g_0; Z) =S_\theta(\theta, g_0; Z) -  H_{\theta g}^\top H_{g g}^{-1} \nabla_g \loss(\theta, g_0; Z) \label{eq:no_score_finite}
\end{align}
would be less sensitive to perturbations of $g_0$, as the component of $S_\theta(\theta, g_0; Z)$ that is predictable through changes in $g_0$ is subtracted out. Furthermore, if we are aware that the expectation of $S$ is made zero at $\theta_\star$, then a stochastic gradient scheme based on~\eqref{eq:no_score_finite} could conceivably achieve a nuisance insensitive rate guarantee in lieu of \Cref{theorem convergence rate baseline SGD}. 
From a variance reduction viewpoint, the correction term in \eqref{eq:no_score_finite} subtracts the regression of the $\theta$ gradient of the loss on the $g$ “gradient” of the loss. By the law of total variance, the variance of the gradient reduces and improves the trajectory of stochastic optimization; see \Cref{appx discuss sec 4} for more details. This variational description \eqref{eq:no_score_finite} hints at how such an operator can be computed algorithmically, instead of the historical approach of deriving the operator via calculation by hand on case by case basis.

Supported by this illustration, we define a generalization that will provide a modified stochastic gradient oracle to use for optimization purposes. Without assuming that $\loss$ is a negative log-likelihood, we generalize the formulas for $\nabla_g \loss(\theta, g; z) \in \R^k$, $H_{\theta g} \in \R^{k \times d}$ and $H_{g g} \in \R^{k \times k}$ for when $\G \equiv (\G, \ip{\cdot, \cdot}_\G)$ is an infinite-dimensional Hilbert space. Under regularity conditions on the directional derivatives of $\ploss$, we have that $\nabla_g \loss(\theta, g; z) \in \G$ for all $z \in \calZ$, $H_{\theta g} = (H_{\theta g}\spow{1}, \ldots, H_{\theta g}\spow{d}) \in \G^d$, and $H_{gg}: \G \rightarrow \G$ is a bounded and self-adjoint operator. The formal details of their construction are contained in \Cref{sec:orthogonalization}. Just as in~\eqref{eq:no_score_finite}, we may consider the operator $\pdir_0: \G \rightarrow \R^d$, defined element-wise by $[\pdir_0 g]_j = \ip{H_{\theta g}\spow{j},H_{gg}^{-1} g}_\calG$, where the invertibility of $H_{gg}$ is satisfied by our assumptions preceding \Cref{theorem debiased SGD convergence rate}. As shown in \eqref{eq:regression}, the orthogonalizing $\Gamma_0$ is defined by both the true nuisance $g_0$ and the target $\theta_\star$, where $g_0$ can usually be learned as some conditional expectation and $\theta_\star$ can be learned by our proposed methods.
We then construct the central object of the upcoming \Cref{theorem debiased SGD convergence rate}: the \emph{Neyman orthogonalized (NO) gradient oracle} 
\begin{align}\label{no score of infinite dim}
    \noscore(\theta, g; z) = \score(\theta, g; z) - \pdir_0 \nabla_g \loss(\theta, g; z).
\end{align}

\begin{lemma}\label{lem NO gradient oracle main text}
    Suppose that \Cref{assumption of population risk}(a) holds and $\D_g^2 \ploss(\theta_\star, g_0)[\cdot,\cdot]: \G \times \G \mapsto \R$ is a bounded and symmetric bilinear form. Then the NO gradient oracle $\noscore(\theta, g;z)$ is Neyman orthogonal at $(\theta_0, g_0)$.
\end{lemma}
We refer readers to \Cref{lemma construct neyman orthogonal score} for the proof. In this context, we refer to the operator $\pdir_0$ as the ``orthogonalizing operator''.
As a natural sanity check, we note that for a risk function that is already Neyman orthogonal at $(\theta_\star, g_0)$, the NO score $\noscore$ is exactly equal to score function $\score$ itself since $\Gamma_0 = 0$. To construct $\noscore$ for the non-orthogonal loss, we provide the following example in partially linear model where the corresponding derivations of $\Gamma_0$ and $\noscore$ are included in \Cref{subsubsec: non-ortho plm}.
\begin{examplebox}
    \textbf{Example 4} (Partially Linear Model)\textbf{.} In addition to Example 1, suppose that $Z = (X, Y, W) \sim \prob$ satisfies
\begin{align*}
    Y = \ip{\theta_\star, X} + g_0(W) + \epsilon,
\end{align*}
where $\theta_\star \in \R^d$ is the true parameter, $g_0: \calW \mapsto \R$ is the true nuisance function, and $\Ex{\prob}{\epsilon \mid X, W} = 0$. The space $\calG \in L_2(\prob)$ with inner product $\ip{g_1, g_2}_{\calG} = \E_\prob[ g_1(W) g_2(W)]$ for any $g_1, g_2 \in \calG$ is a nonparametric class containing functions of the form
    \begin{align*}
        g : \calW \rightarrow \R. 
    \end{align*}
    Consider the following non-orthogonal squared loss function:
\begin{align*}
    \tilde\loss_{\text{PLM}}(\theta, g; z) = \frac{1}{2}[y  - g(w) - \ip{\theta, x}]^2.
\end{align*}
The orthogonalizing operator for this non-orthogonal loss is 
\begin{align*}
    \Gamma_0: g \mapsto \E_\prob[\E_\prob[X \mid W] g(W)],
\end{align*}
and the NO gradient oracle is obtained as
\begin{align*}
    \noscore(\theta,g;z) = -(y  - g(w) - \ip{\theta, x})(x - \E_\prob[X \mid W = w]).
\end{align*}
\end{examplebox}

Motivated by the advantage of a Neyman orthogonal score, we now construct our \algoname algorithm using an estimated the NO score $\noscore$. 
While $\pdir_0$ (like $g_0$) is unknown to the user in general, using an arbitrary estimate $\pdirhat$, we can define the \emph{estimated NO score} $\noscorehat$ oracle via
\begin{align}
    \noscorehat(\theta, g; z) = \score(\theta, g; z) - \pdirhat \nabla_g \loss(\theta, g; z).\label{eq:noscorehar}
\end{align}
Usually, one can obtain such an estimate $\pdirhat$ using the same data stream of $\hg$; we discuss possible strategies in \Cref{appx discuss sec 3}. Finally, using $\noscorehat$ as the stochastic gradient oracle $S$ in \eqref{eq:sgd}, we derive the \algoname update
\begin{align}
    \dsgd\pow{n} = \dsgd\pow{n-1} - \eta \noscorehat(\dsgd\pow{n-1}, \hat g; Z_n), \quad \dsgd\pow{0} \in \Theta.\label{eq:dsgd}
\end{align}

To measure the quality of $\pdirhat$ in our analysis, we use the Frobenius norm $\norm{\pdir}^2_{\Fro} = \sum_{j=1}^d \norm{\pdir\pow{j}}^2_{\op}$ where
$\pdir: \G \rightarrow \R^d, \ \pdir\pow{j}: g \mapsto [\pdir g]_j$
and $\norm{\cdot}_{\op}$ denotes the usual operator norm for linear functionals. As an example, by the uniqueness of Riesz representations, $\norm{\pdir_0}^2_{\Fro} = \sum_{j=1}^d \norm{H_{gg}^{-1} H_{\theta g}\spow{j}}_\calG^2$.

Using this modified oracle~\eqref{eq:noscorehar} requires similar assumptions to those used in \Cref{theorem convergence rate baseline SGD}. For ease of presentation, we defer the formal assumption statement to \Cref{sec:a:debiased}, but note that the result depends on the constants $(\hessianconst, \hconst, \alpha_2, \beta_2, \Kconst_2)$, which are exactly analogous to $(\sconv, \smooth, \alpha_1, \beta_1, \Kconst_1)$ from \Cref{assumption of population risk}.
\begin{theorem}\label{theorem debiased SGD convergence rate}
    Consider the setting of \Cref{theorem convergence rate baseline SGD}, with the addition of \Cref{asm:noscore}. When $\Vert \pdirhat - \pdir_0\Vert_{\Fro} < \hessianconst/(4\alpha_1)$ and
    \begin{align*}
         \eta \leq \frac{\hessianconst - 4\secsmooth_1 \norm{\pdirhat - \pdir_0}_\Fro}{12\hconst^2-3\hessianconst^2/2 + 4(\kappaconst_1 + \kappaconst_2\norm{\hat \Gamma}_{\Fro}^2)},
    \end{align*} 
    the iterates of~\eqref{eq:dsgd} satisfy:
        \begin{align}
        \E_{\mathcal{D}_n \sim \prob^n}[\Vert\dsgd\pow{n} -\theta_\star\Vert_2^2] &\leq \p{1 - \frac{\hessianconst \eta}{2}}^n\norm{\dsgd\pow{0} -\theta_\star}_2^2 + \frac{4(\Kconst_1 + \Kconst_2\norm{\hat \Gamma}_{\Fro}^2)\eta}{\hessianconst}\notag\\
         + \frac{3}{\hessianconst^2}&\p{\highsmooth_2^2\gnorm{\hat g - g_0}^4 + 4\secsmooth_2^2\gnorm{\hat g - g_0}^2\cdot\norm{\pdirhat - \pdir_0}_\Fro^2}.\label{eq:dsgd:thm2}
    \end{align}
\end{theorem}

Compared with \Cref{theorem convergence rate baseline SGD}, \Cref{theorem debiased SGD convergence rate} shows that \algoname can outperform the nuisance sensitive rate through the correction term $\Vert\hat g - g_0\Vert_\G^2 \cdot \Vert\pdirhat - \pdir_0\Vert_{\Fro}^2$, and can align with the nuisance insensitive rate when $\Vert\pdirhat - \pdir_0\Vert_{\Fro}$ is of the order $\bigO(\Vert\hat g - g_0\Vert_\G)$. With slightly different assumptions, \Cref{theorem debiased SGD convergence rate} can further simplified -- see \Cref{sec:a:debiased} for details.

\myparagraph{Interleaving Target and Nuisance Estimation}
The results seen thus far have considered for simplicity the estimate $\hg$ to be a fixed element of $\calG$, and included terms that depend on the discrepancy $\norm{\hg - g_0}_\calG$. Part of the convenience of these results is that if $\hg \equiv \hg\pow{m}$ is the result of a learning procedure with $m$ independent data points, then statistical bounds on $\norm{\hg\pow{m} - g_0}_\calG$ (either in expectation or high probability, depending on the situation) can be plugged in to quantify the bias. While the results naturally account for full batch learning procedures, they are also amenable to analyzing staggered procedures in which two data sources are queried to estimate $\theta_\star$ and $g_0$, respectively. To our knowledge, this is the first theoretical analysis of such an orthogonal stochastic learning method.

To be precise, suppose that we update the nuisance estimator for $m$ times, leading to the sequence $\hg\pow{1}, \dots, \hg\pow{m}$ on a stream of $\calW$-valued data $W_1, \ldots, W_m$, sampled i.i.d.~from a probability measure $\Q$. We define $\sgd\pow{0, n} = \sgd\pow{0} \in \Theta$, and for the update of  $\hg\pow{i}$ for $1 \leq i \leq m$, we define $\sgd\pow{i, 0} = \sgd\pow{i-1, n}$ and produce the sequence $\sgd\pow{i, 1}, \ldots, \sgd\pow{i, n}$ using $n$ steps of the SGD update~\eqref{baseline SGD procedure} initialized at $\sgd\pow{i, 0}$. Consider, for example, the case in which $\G$ is a reproducing kernel Hilbert space (RKHS) with kernel $k(\cdot,\cdot)$. With the assumption that the eigenvalues $(\lambda_j)_{j\geq1}$ of covariance operator $\E_{\Q}[k(W,\cdot) \otimes k(W,\cdot)]$ decay polynomially at order $j^{-\alpha}$, the nonparametric stochastic gradient algorithm of \cite{dieuleveut2016nonparametric} satisfies $\Ex{\Q^m}{\gnorm{\hat g^{(m)} - g_0}^2} = \bigO(m^{-(2\alpha - 1)/(2\alpha)})$. This leads to the following nuisance sensitive rate for a non-Neyman orthogonal loss, by \Cref{prop. adaptive}:
\begin{align*}
    \E_{\prob^{mn}\otimes \Q^m}[\Vert\sgd^{(m,n)} - \theta_\star\Vert_2^2] = \bigO\p{\p{1-\sconv\eta/2}^{mn} + m^{-\frac{2\alpha - 1}{2\alpha}} + n^{-1} + \eta }.
\end{align*}
As another example, suppose that, in addition, we can estimate $\pdirhat \equiv \pdirhat\pow{m}$ using the nonparametric stochastic gradient algorithm of \cite{dieuleveut2016nonparametric} and using the same data stream $(W_1, \dots, W_m)$. If there are high probability bounds for $\gnorm{\hat g^{(m)} - g_0}^2$ and $\norm{\pdirhat^{(m)} - \pdir_0}_{\Fro}^2$ of the same order as $\bigO(m^{-(2\alpha - 1)/(2\alpha)})$ and $\Vert\sgd^{(m,n)} - \theta_\star\Vert_2^2$ decays as described in \Cref{theorem debiased SGD convergence rate}, then we have in \Cref{prop. prob bound} that $\Vert\sgd^{(m,n)} - \theta_\star\Vert_2^2 = \bigO_p\p{\p{1-\sconv\eta/2}^{mn} + m^{-(2\alpha - 1)/\alpha} + n^{-1} + \eta}$
where the $O_p(m^{-(2\alpha - 1)/\alpha})$ nuisance bias term decays quadratically faster than the one for a non-Neyman orthogonal loss. We refer the reader to \Cref{appx discuss sec 3} for further details of this analysis.

\section{Related Work}\label{sec:discussion}

We summarize in this section our discussion of the related work. Additional discussions, as well as calculations supporting them, can be found in \Cref{appx:discussion}. Possible extensions to SGD variants such as SGD with momentum, averaged SGD, and Adam, are explored in \Cref{appx:ASGD}.

From an optimization perspective, it is helpful to know how our convergence bounds perform in the idealized case of a known nuisance, which is equivalent to \eqref{eq:erm}. In this case, \Cref{theorem convergence rate baseline SGD} gives the convergence rate $\E_{\calD_n \sim \prob^n}[\Vert \sgd^{(n)} - \theta_\star\Vert_2^2] = \bigO(\p{1 - \sconv\eta/2}^n + \eta)$, which aligns with the non-asymptotic SGD convergence rates, in mean-square error~\citep{moulines2011non} and in high-probability~\citep{cutler2023stochastic}. Our result requires a smaller learning rate $\eta < \sconv/2(M\sconv + 2\kappaconst_1)$ when compared to the requirement $\eta < 1/(2\smooth)$ from \citet{cutler2023stochastic}. This is entirely due to our bounded moment assumption (see \Cref{assumption of population risk}(d)), which contrasts with a uniform boundedness assumption over all $\Theta \times \Gr$. 
In addition, when the uniform moment bound holds true, $\kappaconst_1$ becomes zero, and our learning rate requirement becomes $\eta < 1/(2\smooth)$.

The comparison with unbiased SGD, and biased SGD, respectively, is also valuable. 
In the biased SGD literature, the ``bias'' refers to the fact that $\Ex{Z \sim \prob}{S(\theta, \hg; Z)} \neq \grad_\theta \ploss(\theta, g_0)$ in general. The convergence radius then depends on the average value of $\norm{\Ex{Z \sim \prob}{S(\theta\pow{n}, \hg; Z)} - \grad_\theta \ploss(\theta\pow{n}, g_0)}_2^2$. Results along this line result in a radius that may not scale with $\eta$; see \citet[Thm. 3]{demidovich2023aguide}. Although this form of bias may be related to $\norm{\hg - g_0}_\calG$ under Lipschitzness conditions on the oracle, it is unclear how to effectively incorporate Neyman orthogonality into these general-purpose approaches. Our approach naturally leverages Neyman orthogonality whenever it holds. 

In the general case of an unknown nuisance, \citet{foster2023orthogonal, chernozhukov2018plug} consider full batch learning methods based on analytically crafted Neyman orthogonal risk functions in various scenarios. For regression functionals, the procedure from \citet{chernozhukov2022automatic} using random forests or neural networks can ensure that the bias term $\norm{\hg - g_0}_\calG^2$ is asymptotically negligible for large samples, in the sense that classical statistical confidence sets for $\theta_\star$ are asymptotically valid. These papers are focused on algorithm-independent statistical properties. 

Our work fills this gap, by providing non-asymptotic convergence guarantees for stochastic gradient algorithms under unknown nuisances. 
Moreover, the modified stochastic gradient oracle moreover offers a flexible solution to deal with general risk functions.  If deriving an orthogonalized risk by hand is difficult or impossible, then the strategy we propose can be applied, and \Cref{theorem debiased SGD convergence rate} demonstrates that, when the learning rate 
$\eta$ is set appropriately, 
the convergence rate using the modified stochastic gradient oracle
can be improved to  
\begin{align*}
    \bigO\Big(\p{1 - \frac{\hessianconst\eta}{2}}^n + \underbrace{\Vert\hat g - g_0\Vert_\G^4 + \Vert\hat g - g_0\Vert_\G^2 \cdot \norm{\pdirhat - \pdir_0}_{\Fro}^2}_{\text{improvement over $\Vert\hat g - g_0\Vert_\G^2$}} + \eta\Big).
\end{align*}
When we have the true orthogonalizing $\pdir_0$, the improved rate recovers the nuisance insensitive one from~\Cref{theorem convergence rate baseline SGD}.  
Besides, when $\hg$ converges but $\Vert\pdirhat - \pdir_0\Vert_\G^2 = \bigO_p(1)$, the improved rate resembles the nuisance sensitive rate of \Cref{theorem convergence rate baseline SGD}, plus a $\bigO(\eta)$ bias term. Thus, the quality of the estimated orthogonalizing operator governs how the optimization interpolates between these two rates.

Having understood the performance of SGD when using an estimated orthogonalizing operator, one question is how to compute or approximate such an operator.~\citet{luedtke2024simplifying} recently demonstrated %
that an orthogonalizing operator can be derived using algorithmic/reverse mode functional differentiation in many interesting cases. This can also be effective in our stochastic setting.
In \Cref{sec:optimization}, using least-squares regression as an illustration, we developed a control variate~\citep{johnstone1985efficient} interpretation of the variance reduction. This viewpoint offers another venue to develop approximate orthogonalizing operators.

\paragraph{Conclusion.}
We established non-asymptotic convergence guarantees for SGD algorithms under nuisances. We showed how the Neyman orthogonality of the loss function can mitigate the sensitivity of SGD algorithms to the effect of nuisances, and obtained results that align with recent ones from the DML/OSL literature in the batch setting. We also presented an iteratively orthogonalized SGD algorithm, whose convergence rate aligns with the rate in the nuisance insensitive regime. Extensions to hypothesis testing and reinforcement learning are interesting venues for future work.

\clearpage
\paragraph{Acknowledgments.}
The authors would like to thank L. Liu, V. Roulet, and J. Wellner for valuable comments and suggestions. This work was supported by NSF DMS-2023166, DMS-2134012, DMS-2210216, DMS-2502281, CCF-2019844, NIH, and IARPA 2022-22072200003. Part of this work was performed while R. Mehta and Z. Harchaoui were visiting the Simons Institute for the Theory of Computing, and A. Luedtke was visiting the Institute of Statistical Mathematics, and with the University of Washington.

\bibliographystyle{abbrvnat}
\bibliography{bib}

\clearpage
\appendix
\addcontentsline{toc}{section}{Appendix} %
\part{Appendix} %
\parttoc %

\clearpage

\section{Notation}\label{sec:a:notation}
\begin{table}[ht]
    \centering

    \begin{adjustbox}{max width=1.0\linewidth}
    \renewcommand{\arraystretch}{1.15}
    \begin{tabular}{cc}
    \toprule
        {\bf Symbol} & {\bf Description}\\
        \midrule
        $\Theta \subseteq \R^d$ & Finite-dimensional parameter class.\\
        $(\calG, \|\cdot\|_\G)$ & Possibly infinite-dimensional nuisance space.\\
        $(\G, \ip{\cdot, \cdot}_\G)$ & The nuisance space as a Hilbert space.\\
        \midrule
        $\prob$& The unknown distribution of interest.\\
        $Z \in \mathcal{Z}$ & The random variable under $\prob$.\\
        $\theta_\star$ & The target of interest.\\
        $g_0$ & The true nuisance parameter.\\ 
        \midrule
        $\loss(\theta, g; z)$ & The prespecified loss function.\\
        $\ploss(\theta, g)$ & The population loss $\E_{Z \sim \prob}\sbr{\loss(\theta, g; Z)}$.\\
         $\score(\theta, g; z)$ & The score function $\nabla_\theta \loss(\theta, g; z)$.\\
       $\score(\theta, g)$ & The population score $\E_{Z \sim \prob}\sbr{\nabla_\theta \loss(\theta, g; z)}$.\\
       $\noscore(\theta, g; z)$ & The Neyman orthogonalized score.\\
       $\noscore(\theta, g)$ & The population Neyman orthogonalized score $\E_{Z \sim \prob}[\noscore(\theta, g; Z)]$.\\
        \midrule
    $(\nabla_\theta, \nabla_g)$ & The gradient \wrt $\theta$ and $g$\\
    $(\D_\theta, \D_g)$ & The derivative operator \wrt $\theta$ and $g$.\\
               $H_{\theta g}$ & The transposed Jacobian defined by $\nabla_g \score(\theta_\star, g_0) \in \G^d$\\
       $H_{g g}$ & The nuisance Hessian operator defined by $\nabla_g^2 \ploss(\theta_\star, g_0)$\\
       $\pdir_0$ & Linear operator defined by $[\pdir_0 g]_j = \ip{H_{\theta g}\spow{j},H_{gg}^{-1} g}_\calG$.\\
        
        \midrule
        $\sconv$ & The strong convexity constant of $\ploss$.\\
        $\smooth$ & The smoothness constant of $\ploss$.\\ 
        $(\Kconst_1, \kappaconst_1)$ & Constants to bound the second moment of $\score(\theta, g; Z)$.\\
        $(\secsmooth_1,\secsmooth_2)$ & The second order smoothness constant of $\ploss$.\\
        $\highsmooth_1$ & The higher order smoothness constant of a Neyman orthogonal $\ploss$.\\
        \midrule
        $\hessianconst$ & The strong convexity constant of  $\nabla_\theta\noscore(\theta_\star, g_0)$.\\
        $\hconst$ & The smoothness constant of  $\nabla_\theta\noscore(\theta_\star, g_0)$.\\
        $(\Kconst_2, \kappaconst_2)$ & Constants to bound the second moment of $\noscore(\theta, g; Z)$.\\
        $\highsmooth_2$ & The higher order smoothness constant of $\noscore$.\\

        \midrule 
        $\eta$ & The learning rate of stochastic optimization. \\
        $n$ & The iteration of stochastic gradient.\\
        $m$ & The iteration of nuisance estimation.\\
        \bottomrule
    \end{tabular}
    \end{adjustbox}
    \vspace{6pt}
    \caption{Notation used throughout the paper.}
    \label{tab:notation}
\end{table}

\clearpage

\section{Detailed Examples}\label{sec:a:examples}

In this section, we describe in detail how the three examples in \Cref{sec:osl} from the main text satisfy \Cref{assumption of population risk} and \Cref{assumption orthogonality main}. We first talk about the partially linear model (PLM) in \Cref{appx PLM}, and then introduce the conditional averaged treatment effect (CATE) based on the potential outcomes framework in \Cref{appx PLM}. Under the same framework, finally we talk about the conditional relative risk (CRR) in \Cref{appx CRR}. In addition, we also study a non-orthogonal loss usually used for PLM in \Cref{subsubsec: non-ortho plm} and an unrestricted loss function for CATE in \Cref{sec unres cate}. The constants for all examples are concluded in \Cref{tab:examples} and proofs of lemmas in this section are provided in \Cref{sec: ex proof}.

\begin{sidewaystable}
    \centering

    \begin{adjustbox}{max width=0.9\linewidth}
    \renewcommand{\arraystretch}{1.4}
    \begin{tabular}{ccccccc}
    \toprule
        Example &  $\sconv$ &  $\smooth$  & $\Kconst_1$ & $\kappaconst_1$ & $\secsmooth_1$ & $\highsmooth_1$\\
        \midrule
        (1) Orthogonal PLM &   $\lambda_0$&  $C_X^2 + r^2$ & $18C_X^2\sigma^2 + \bigO(r^2)$ & $18C_X^4 + \bigO(r^2)$ &$2(1 + \norm{\theta_\star}_2)r$ & $2(1 + \norm{\theta_\star}_2)$
\\
        (2) Non-Orthogonal PLM &   $\lambda_0$&  $C_X^2$ & $6C_X^2(\sigma^2 + 2r^2)$ & $2C_X^4$ & $C_X$ & --
\\
        (3) Unrestricted CATE & $c_0^2\lambda_0$ & $C_X^2(1+r^2)$ & $12C_X^2(\sigma^2 + 4(C_X\norm{\theta_\star}_2 + C_\tau)^2) + \bigO(r^2)$ & $27C_X^4 + \bigO(r^2)$& $\bigO(r)$ & $C_X(1 + 4(C_X\norm{\theta_\star}_2 + C_\tau))$\\

        (4) Restricted CATE & $\lambda_0$ & $C_X^2$  &$27C_X^2(\sigma^2+2C^2) + \bigO(r^2) $& $3C_X^4$ &$2C_X(2r + 3)r$& $4c_0^{-2}C_X(1 +r)$\\ 

        (5) CRR &$C^2\delta\lambda_0$ & $C_X^2(1 + 2\delta^{-1}r^2)$ & $24C_X^2(\delta^{-2} + 4c_0^{-1}C_Y^2) + \bigO(r^2)$ &$6C_X^4(1+6(\delta^{-2} + 4c_0^{-1}C_Y^2)) + \bigO(r^2)$ & $2C_X(c_0^{-1}+ 1)r$ & $4c_0^{-2}C_X(1 + r)$\\
        \bottomrule
    \end{tabular}
    \end{adjustbox}
    \vspace{6pt}
    \caption{Constants for All Examples.}
    \label{tab:examples}

\end{sidewaystable}

\subsection{Partially Linear Model}\label{appx PLM}

\subsubsection{Orthogonal Loss}\label{subsubsec: ortho PLM}
We revisit Example 1 from the main text where we consider the target of interest as a solution of a partially linear model. Let $Z = (X, Y, W)$, where $X$ is an $\R^d$-valued input, $Y$ is a real-valued outcome, and $W$ is a $\calW$-valued control or confounder. The space $\calG$ is a nonparametric class containing functions of the form
    \begin{align*}
        g = (g_Y, g_X): \calW \rightarrow \R \times \R^d. 
    \end{align*}
    Following the construction of \citet{robinson1988root}, this $g$ is supplied to the loss
    \begin{align}
        \loss(\theta, g; z) = \frac{1}{2}[y - g_Y(w) - \ip{\theta, x - g_X(w)}]^2.
    \end{align}

    To ensure $\theta_\star$ can be interpreted via the projection of $\E_{\prob}[Y|X,W]$ onto partially linear additive functions, the true nuisance is given by $g_0 = (g_{0, X}, g_{0, Y})$, where
    \begin{align*}
        g_{0, Y}(w) := \Ex{\prob}{Y \mid W = w} \text{ and } g_{0, X}(w) := \Ex{\prob}{X \mid W = w}.
    \end{align*}
We define the residual $\epsilon$ at $(\theta_\star, g_0)$ as 
\begin{align*}
    \epsilon = Y - g_{0,Y}(W) - \ip{\theta_\star, X - g_{0,X}(W)}.
\end{align*}
    
\begin{lemma}\label{lemma: example ortho PLM}
    Let $\tilde Y = Y - g_{0, Y}(w)$ and $\tilde X = X - g_{0, X}(w)$. We assume the following conditions:
\begin{enumerate}[label=(\alph*)]
    \item $\lambda_{\min}(\E_{\prob}[\tilde X\tilde X^\top])\geq \lambda_0$ for some constant $\lambda_0 > 0$.
    \item $\norm{\tilde X}_2 \leq C_X$ \as and $\Ex{\prob}{\epsilon^4} \leq \sigma^4$ for some constants $C_X, \sigma > 0$.
\end{enumerate}
Then \Cref{assumption of population risk} and \Cref{assumption orthogonality main} are satisfied. The target $\theta_\star$ is the minimizer of the squared loss:
\begin{align*}
    \theta_\star = \argmin_{\theta \in \R^d } \E_{\prob}[(\tilde Y - \tilde X^\top \theta)^2].
\end{align*}

\end{lemma}
The proof of \Cref{lemma: example ortho PLM} is provided in \Cref{sec: proof ex1}.

\subsubsection{Non-orthogonal Loss}\label{subsubsec: non-ortho plm}
Suppose that the outcome $Y$ is generated under the partially linear model: 
\begin{align}
    Y = \ip{\theta_0, X} + g_0(W) + \epsilon,
\end{align}
where $\theta_0 \in \R^d$ is the true parameter, $g_0: \calW \mapsto \R$ is the true nuisance function and $\Ex{\prob}{\epsilon \mid X, W} = 0$. The space $\calG$ is a nonparametric class containing functions of the form
    \begin{align*}
        g: \calW \mapsto \R. 
    \end{align*}
We can also consider the following non-orthogonal squared loss function:
\begin{align}\label{ex2 loss}
    \loss(\theta, g; z) = \frac{1}{2}[y  - g(w) - \ip{\theta, x}]^2.
\end{align}
We define the residual $\epsilon$ at $(\theta_\star, g_0)$ as 
\begin{align*}
    \epsilon = Y - g_{0}(W) - \ip{\theta_\star, X}.
\end{align*}
\begin{lemma}\label{lemma: example non-ortho PLM}
We assume the following conditions:
\begin{enumerate}[label=(\alph*)]
    \item $\lambda_{\min}(\Ex{\prob}{XX^\top}) \geq \lambda_0$ for some constant $\lambda_0 > 0$.
    \item $\norm{X}_{\infty}\leq C_X$  \as and $\Ex{\prob}{\epsilon^2} \leq \sigma^2$ for some constants $C_X, \sigma> 0$.
\end{enumerate}
Then \Cref{assumption of population risk} is satisfied and the target $\theta_\star$ is the true parameter, i.e., $\theta_\star = \theta_0$.

\end{lemma}
The proof of \Cref{lemma: example non-ortho PLM} is provided in \Cref{sec: proof ex2}.

\myparagraph{Orthogonalization}
We can perform our orthogonalization method to obtain the Neyman orthogonal gradient oracle for this non-orthogonal loss. For any $h_1, h_2 \in \G$, we define the inner product of $\G$ as 
\begin{align}\label{inner product}
    \ip{h_1, h_2}_\G = \Ex{\prob}{h_1(W)h_2(W)}.
\end{align}

For any $(\theta, g, z) \in \Theta \times G \times \calZ$ By \Cref{def derivative} the derivative of non-orthogonal loss \eqref{ex2 loss} along the direction of $h_1$ is given by
\begin{align}\label{eq Dg1}
    \D_g \ell(\theta, g; z)[h_1] = \frac{\d}{\d t}\p{\frac{1}{2}[y  - (g + t h_1)(w) - \ip{\theta, x}]^2} = -(y - g(w) - \ip{\theta, x})h_1(w).
\end{align}
Do derivative on $\D_g \ell(\theta, g; z)[h_1]$ along the direction of $h_2$ and we have
\begin{align}\label{eq Dg2}
    \D_g^2 \ell(\theta, g; z)[h_1, h_2] = \frac{\d}{\d t}\p{-(y - (g+t h_2)(w) - \ip{\theta, x})h_1(w)} = h_1(w)h_2(w),
\end{align}
which implies 
\begin{align*}
    \D_g^2 L(\theta_\star, g_0)[h_1, h_2] = \Ex{\prob}{ \D_g^2 \ell(\theta_\star, g_0; Z)[h_1, h_2]} = \Ex{\prob}{h_1(W)h_2(W)}.
\end{align*}
By the definition in \eqref{def H gg}, we have $H_{gg} = \mathbf{I}$ the identity operator. In addition, do derivative on the score along the direction of $h \in \G$ and we have
\begin{align*}
    \D_g S_\theta(\theta, g; z)[h] = \frac{\d}{\d t}\p{-(y - (g+th)(w) - \ip{\theta, x})x} = h(w)x,
\end{align*}
which implies that 
\begin{align*}
    \D_g S_\theta(\theta, g)[h] = \Ex{\prob}{ S_\theta(\theta, g; Z)[h]} = \Ex{\prob}{h(W)\Ex{\prob}{X\mid W}}.
\end{align*}
By the definition in \eqref{def H theta g}, we have $H_{\theta g} = \Ex{\prob}{X\mid W}$. Thus, by \eqref{appx gamma} we have
\begin{align}\label{sim gamma0}
    \Gamma_0: g \mapsto \ip{\Ex{\prob}{X\mid W},  g}_\G = \Ex{\prob}{\Ex{\prob}{X\mid W}g(W)}.
\end{align}
Thus, the Neyman orthogonalized gradient oracle defined in \eqref{def no score appx} is given by
\begin{align}
    \noscore(\theta, g; z) &= S_\theta(\theta, g; z) - \D_g \ell(\theta, g; z)[\Ex{\prob}{X\mid W=w}]\notag \\
    &= -(y - g(w) - \ip{\theta, x})(x - \Ex{\prob}{X\mid W=w}).\label{ex2 noscore}
\end{align}

\begin{lemma}\label{lem eg osgd}
Consider the bounded linear operator $\pdirhat: \G \mapsto \R^d$ such that $[\pdirhat g]_j = \ip{\hat \gamma^{(j)}, g}_\G, \forall g \in \G$ for some $\hat \gamma^{(j)} \in \G, j=1,\dots,d$. Let $\tilde Y = Y - g_{0, Y}(w)$ and $\tilde X = X - g_{0, X}(w)$. We assume the following conditions:
    \begin{enumerate}[label=(\alph*)]
        \item $\lambda_{\min}(\E_{\prob}[\tilde X\tilde X^\top])\geq \lambda_0$ for some constant $\lambda_0 > 0$.
        \item $\norm{\tilde X}_2 \leq C_X$ \as and $\Ex{\prob}{\epsilon^4} \leq \sigma^4$ for some constants $C_X, \sigma > 0$.
    \end{enumerate}
    Then \Cref{asm:noscore} is satisfied.
\end{lemma}
The proof of \Cref{lem eg osgd} is provided in \Cref{sec: proof osgd}.

\subsection{Conditional Averaged Treatment Effect}\label{appx CATE}

We now introduce examples in causal inference which are established based on the potential outcomes framework. The potential outcomes framework \citep{rubin1974estimating} has been widely used in causal inference. Let $Z = (W, X, Y) \in \br{0,1}\times \R^d \times \R$ under some distribution $\prob$. We posit the existence of potential outcomes $Y(1),Y(0) \in \R$. The conditional averaged treatment effect (CATE) is then defined as 
\begin{align*}
    \tau_0(x) = \E_{\prob}[Y(1) - Y(0) \mid X=x].
\end{align*}
To identify $\tau_0(x)$ and the following causal assumptions are required:
\begin{assumption}\label{asm causal}
    The following conditions hold:
    \begin{enumerate}[label=(\alph*)]
    \item (consistency) $Y = Y(W)$.
        \item (unconfoundedness) $Y(w) \perp W \mid X$ for all $w\in \{0,1\}$. 
        \item (positive overlap) $c_0 \leq \prob\p{W = 1 \mid X}\leq 1-c_0$ \as for some $c_0>0$.
    \end{enumerate}
\end{assumption}
Under \Cref{asm causal}, $\tau_0$ can be identified by observed data since
\begin{align*}
    \tau_0(x) &= \E_{\prob}[Y(1) - Y(0) \mid X=x]\\
    &=\E_{\prob}[Y(1) \mid W=1, X=x] - \E_{\prob}[Y(0) \mid W=0, X=x]\\
    &=\E_{\prob}[Y \mid W=1, X=x] - \E_{\prob}[Y \mid W=0, X=x].
\end{align*}

\subsubsection{Unrestricted Nuisance}\label{sec unres cate}
We observe $Z = (X, Y, W)$, where $W$ is a binary treatment assignment. The functions in $\calG$ are of the form
    \begin{align*}
        g = (g^{\text{out}}, g^{\text{prop}}): \R^d \rightarrow \R \times \R,
    \end{align*}
and are evaluated (see \citet[Eq.~(2)]{nie2021quasi}) at the loss
    \begin{align}
        \loss\p{\theta, g; z} = \frac{1}{2}\p{y - g^{\text{out}}(x) - \p{w - g^{\text{prop}}(x)}\ip{\theta, x}}^2.
    \end{align}
For $g_0 = (g^{\text{out}}_0, g^{\text{prop}}_0)$ nuisance functions $g^{\text{out}}_0$ and $g^{\text{prop}}_0$ represent the outcome regression and the propensity score, respectively:
    \begin{align*}
        g^{\text{out}}_0(x) := \E_{\prob}[Y \mid X = x] \text{ and } g^{\text{prop}}_0(x) := \E_{\prob}[W \mid X = x].
    \end{align*}
We define the residual $\epsilon$ under the true model as 
\begin{align*}
    \epsilon = Y - g_0^{\text{out}}(X) - \p{W- g_0^{\text{prop}}(X)}\tau_0(X).
\end{align*}
\begin{lemma}\label{lemma: example unres CATE}
We assume \Cref{asm causal} and the following conditions hold:
\begin{enumerate}[label=(\alph*)]
    \item $\lambda_{min}(\Ex{\prob}{XX^\top})\geq \lambda_0$ for some constant $\lambda_0>0$.
    \item $\norm{X}_{2}\leq C_X$ and $\abs{\tau_0(X)} \leq C_\tau$ \as and $\Ex{\prob}{\epsilon^4} \leq \sigma^4$  for some constants $C_X, C_\tau, \sigma> 0$. 
\end{enumerate}
Then \Cref{assumption of population risk} and \Cref{assumption orthogonality main} are satisfied. The target $\theta_\star$ is the minimizer of the squared loss:
\begin{align*}
    \theta_\star = \argmin_{\theta \in \R^d } \Ex{\prob}{(W-g^{\text{prop}}_0(X))^2(\tau_0(X) - X^\top\theta)^2}.
\end{align*}

\end{lemma}
The proof of \Cref{lemma: example unres CATE} is provided in \Cref{sec: proof ex3}.

\subsubsection{Restricted Nuisance}
We observe $Z = (X, Y, W)$, where $W$ is a binary treatment assignment. Here we restrict the propensity model as $g^{\text{prop}}: \R^d \mapsto (0,1)$. The functions in $\calG$ are of the form
    \begin{align*}
        g = (g\spow{0}, g\spow{1}, g^{\text{prop}}): \R^d \rightarrow \R \times \R \times (0, 1),
    \end{align*}
    and are evaluated (see \citet[Thm.~1]{van2014targeted}) at the loss
    \begin{align}
        \loss(\theta, g; z) = \frac{1}{2}\p{g\spow{1}(x) - g\spow{0}(x) + \frac{w-g^{\text{prop}}(x)}{g^{\text{prop}}(x)(1-g^{\text{prop}}(x))}(y - g\spow{w}(x)) - \ip{\theta, x}}^2.
    \end{align}
   This loss also appears in \citet[Eq.~(23)]{foster2023orthogonal}. For $g_0 = (g\spow{0}_0, g\spow{1}_0, g^{\text{prop}}_0)$ nuisance functions $g\spow{0}_0$ and $g\spow{1}_0$ represent the outcome regressions
    \begin{align*}
        g\spow{0}_0(x) := \E_{\prob}[Y \mid W=1, X = x] \text{ and } g\spow{1}_0(x) := \E_{\prob}[Y \mid W=0, X = x].
    \end{align*}
 We define the residual $\epsilon$ as
\begin{align*}
    \epsilon = \frac{W-g_0^{\text{prop}}(X)}{g_0^{\text{prop}}(X)(1-g_0^{\text{prop}}(X))}(Y - g_0\spow{W}(X)).
\end{align*}
\begin{lemma}\label{lemma: example res CATE}

We assume \Cref{asm causal} and the following conditions hold:
\begin{enumerate}[label=(\alph*)]
    \item $\lambda_{min}(\Ex{\prob}{XX^\top})\geq \lambda_0$ for some constant $\lambda_0>0$.
    \item $\Ex{\prob}{\epsilon^2} \leq \sigma^2$, $\norm{X}_{2}\leq C_X$, and $\abs{Y - g_0\spow{w}(X)} \leq C_Y, w=0,1$ \as for some constants $\sigma, C_X, C_Y> 0$. 
\end{enumerate}
Then \Cref{assumption of population risk} and \Cref{assumption orthogonality main} are satisfied. The target $\theta_\star$ is the minimizer of the squared loss:
\begin{align*}
    \theta_\star = \argmin_{\theta \in \R^d } \Ex{\prob}{(\tau_0(X) - X^\top\theta)^2}.
\end{align*}

\end{lemma}

The proof of \Cref{lemma: example res CATE} is provided in \Cref{sec: proof ex4}.

\subsection{Conditional Relative Risk}\label{appx CRR}
    We retain all components of the previous example, changing only the loss and assuming that the outcome $Y$ is binary/non-negative. First, consider the ``label'' function
    \begin{align*}
        \mu_g\spow{s}(z) = g\spow{s}(x) + \frac{\ind(w = s)}{sg^{\text{prop}}(x) + (1-s) (1-g^{\text{prop}}(x))}(y - g\spow{s}(x)),
    \end{align*}
    where $\ind(\cdot)$ denotes the indicator function, and the log-linear predictor $ p_\theta(x) = e^{\ip{\theta, x}}/(1 + e^{\ip{\theta, x}})$.
    Following Example~2 in \cite{van2024combining}, we then employ the cross entropy-type loss function
    \begin{align}
        \loss(\theta,  g; z) = - \sbr{\mu_g\spow{1}(z) \log p_\theta(x) + \mu_g\spow{0}(z) \log(1 - p_\theta(x))}.
    \end{align}

    \begin{lemma}\label{lemma: example CRR}
We assume the following conditions:
\begin{enumerate}[label=(\alph*)]
    \item $\lambda_{min}(\Ex{\prob}{XX^\top})\geq \lambda_0$ for some constant $ \lambda_0 >0$.
    \item $\norm{X}_2 \leq C_X$ and $Y(w) - g_0\spow{w}(X) \leq C_Y, w = 0,1$ \as for some constants $C_X, C_Y>0$.
    \item $\delta \leq g_0\spow{0}(X) + g_0\spow{1}(X) \leq \delta^{-1}$ \as for some constant $\delta > 0$.
\end{enumerate}
Then \Cref{assumption of population risk} and \Cref{assumption orthogonality main} are satisfied. The target $\theta_\star$ is the minimizer of the weighted cross entropy loss:
\begin{align*}
    \theta_\star = \argmin_{\theta \in \R^d}  -\Ex{\prob}{g\spow{1}_0(X) \log p_\theta(X) + g\spow{0}_0(X) \log(1 - p_\theta(X))}.
\end{align*}
\end{lemma}
The proof of \Cref{lemma: example CRR} is provided in \Cref{sec: proof ex5}.

\subsection{Proofs}\label{sec: ex proof}

\subsubsection{Proof of Lemma \ref{lemma: example ortho PLM}}\label{sec: proof ex1}

\begin{proof}
    We consider the following loss:
    \begin{align*}
        \loss(\theta, g; z) = \frac{1}{2}(y - g_Y(w) - \ip{\theta, x - g_X(w)})^2,
    \end{align*}
    with the corresponding risk function defined as
\begin{align*}
    \ploss(\theta, g) = \frac{1}{2}\Ex{\prob}{(Y - g_Y(W) - \ip{\theta, X - g_X(W)})^2}.
\end{align*}
Let $\tilde Y = Y - g_{0, Y}(w)$ and $\tilde X = X - g_{0, X}(w)$. By definition, the target $\theta_\star$ is the minimizer of the squared loss:
\begin{align}\label{eq: ex1 target}
    \theta_\star = \argmin_{\theta \in \R^d } \Ex{\prob}{(\tilde Y - \tilde X^\top \theta)^2} = \Ex{\prob}{\tilde X\tilde X^\top}^{-1}\Ex{\prob}{\tilde Y\tilde X}.
\end{align}
    Differentiating $\loss(\theta, g; z)$ with respect to $\theta$, we obtain the gradient and Hessian \wrt $\theta$ as
\begin{align*}
    \score(\theta, g; z) &= -(y - g_Y(w) - \ip{\theta, x - g_X(w)})(x - g_X(w)),\\
    H_{\theta\theta}(\theta, g; z) &= (x - g_X(w))(x - g_X(w))^\top.
\end{align*}
The expected gradient and expected Hessian are then obtained as
\begin{align*}
    \score(\theta, g) &= -\Ex{\prob}{(Y - g_Y(W) - \ip{\theta, X - g_X(W)})(X - g_X(W))},\\
    H_{\theta\theta}(\theta, g) &= \Ex{\prob}{(X - g_X(W))(X - g_X(W))^\top}.
\end{align*}
We consider the nuisance neighborhood such that for $g \in \Gr$,
\begin{align}\label{eq: ex1 Gr}
   \gnorm{g - g_0} := \max\br{\Ex{\prob}{\norm{g_X(W) - g_{0,X}(W)}_2^4}^{\frac{1}{4}}, \Ex{\prob}{(g_Y(W) - g_{0,Y}(W))^4}^{\frac{1}{4}}} \leq r.
\end{align}
We now verify that the loss function $\loss$ satisfies \Cref{assumption of population risk}.

(a) We assume that $g_X(w): \calW \mapsto \R^d$ and $g_Y(w):\calW \mapsto \R$ are continuous functions, thus \Cref{assumption of population risk}(a) is satisfied.

(b) By \eqref{eq: ex1 target}, it follows from KKT conditions that 
\begin{align}\label{eq: ex1 a2}
    \score(\theta_\star, g_0) = -\Ex{\prob}{(\tilde Y - \ip{\theta_\star, \tilde X})\tilde X} = 0.
\end{align}

(c) Since $\E_{\prob}[\tilde X \mid W] = 0$ and $\E_{\prob}[\tilde Y \mid W] = 0$, we have
\begin{align*}
    H_{\theta\theta}(\theta, g) = \Ex{\prob}{\tilde X \tilde X^\top} + \Ex{\prob}{(g_X(W) - g_{0,X}(W))(g_X(W) - g_{0,X}(W))^\top}.
\end{align*}
For any $g \in \G_r$,
when $\lambda_{\min}(\E_{\prob}[\tilde X\tilde X^\top])\geq \lambda_0$ and $\norm{\tilde X}_2 \leq C_X$ \as, we have
\begin{align}\label{eq: ex1 a3}
   \lambda_0\mathbf{I} \preccurlyeq H_{\theta\theta}(\theta, g) \preccurlyeq ( C_X^2 + r^2)\mathbf{I} \implies \sconv = \lambda_0 \text{ and } \smooth = C_X^2 + r^2.
\end{align}

(d) Consider the Taylor expansion around $\theta_\star$, we have
\begin{align*}
    \score(\theta, g; Z) - \score(\theta, g) = \score(\theta_\star, g; Z) - \score(\theta_\star, g) + (H_{\theta\theta}(\theta_\star, g; Z) - H_{\theta\theta}(\theta_\star, g))(\theta - \theta_\star).
\end{align*}
Let $\epsilon = \tilde Y - \ip{\theta_\star, \tilde X}$. Note that $X - g_X(W) = \tilde X - (g_X -g_{0,X})(W)$ and
\begin{align}\label{eq ex1 good}
    Y - g_Y(W) - \ip{\theta_\star, X - g_X(W)} = \epsilon - (g_Y - g_{0,Y})(W) + \ip{\theta_\star, (g_X - g_{0,X})(W)}.
\end{align}
Since $\Ex{\prob}{\epsilon \mid W} = 0$, $\E_{\prob}[\tilde X  \mid W] = 0$ by definition and $\E_{\prob}[\epsilon \tilde X] = 0$ by \eqref{eq: ex1 a2}, then for any $g \in \Gr$,
\begin{align*}
    \Vert &\score(\theta_\star, g)\Vert_2 = \norm{\Ex{\prob}{(g_Y - g_{0,Y})(g_X -g_{0,X})(W) - \ip{\theta_\star, (g_X - g_{0,X})(W)}(g_X -g_{0,X})(W)}}_2\\
    &\leq \p{\Ex{\prob}{\p{(g_Y - g_{0,Y})(W)}^2}\Ex{\prob}{\norm{(g_X - g_{0,X})(W)}_2^2}}^{1/2} + \Ex{\prob}{\norm{(g_X - g_{0,X})(W)}_2^2}\norm{\theta_\star}_2\\
    &\leq r^2(1 + \norm{\theta_\star}).
\end{align*}
Similarly, we have
\begin{align*}
    \Vert &\score(\theta_\star, g;Z)\Vert_2^2 \leq \p{\epsilon - (g_Y - g_{0,Y})(W) + \ip{\theta_\star, (g_X - g_{0,X})(W)}}^2\norm{\tilde X - (g_X -g_{0,X})(W)}_2^2\\
    &\leq 3\p{\epsilon^2 + ((g_Y - g_{0,Y})(W))^2 +\norm{(g_X - g_{0,X})(W)}_2^2\norm{\theta_\star}_2^2}\p{C_X + \norm{(g_X -g_{0,X})(W)}_2}^2\\
    &\leq 6\p{\epsilon^2 + ((g_Y - g_{0,Y})(W))^2 +\norm{(g_X - g_{0,X})(W)}_2^2\norm{\theta_\star}_2^2}\p{C_X^2 + \norm{(g_X -g_{0,X})(W)}_2^2},
\end{align*}
which implies that for $g \in \Gr$, when $\E_{\prob}[\epsilon^4] \leq \sigma^4$,
\begin{align*}
    \E_{\prob}[\Vert &\score(\theta_\star, g;Z)\Vert_2^2] \\
    &\leq 6C_X^2\p{\Ex{\prob}{\epsilon^2} + \Ex{\prob}{((g_Y - g_{0,Y})(W))^2} + \Ex{\prob}{\norm{(g_X - g_{0,X})(W)}_2^2}\norm{\theta_\star}_2^2}\\
    &\quad + 6\Ex{\prob}{\epsilon^2\norm{(g_X -g_{0,X})(W)}_2^2} + 6\Ex{\prob}{((g_Y - g_{0,Y})(W))^2\norm{(g_X -g_{0,X})(W)}_2^2}\\
    &\quad + 6\Ex{\prob}{\norm{(g_X -g_{0,X})(W)}_2^4}\norm{\theta_\star}_2^2\\
    &\leq 6C_X^2(\sigma^2 + r^2 + r^2\norm{\theta_\star}_2^2) + 6r^4\norm{\theta_\star}_2^2 + 6\p{\Ex{\prob}{\epsilon^4}\Ex{\prob}{\norm{(g_X -g_{0,X})(W)}_2^4}}^{1/2}\\
    &\quad + 6\p{\Ex{\prob}{((g_Y - g_{0,Y})(W))^4}\Ex{\prob}{\norm{(g_X -g_{0,X})(W)}_2^4}}^{1/2}\\
    &\leq 6C_X^2\sigma^2 + 6\br{\sigma^2+C_X^2(1 + \norm{\theta_\star}_2^2)}r^2 + 6(1+\norm{\theta_\star}_2^2)r^4.
\end{align*}
Thus, for any $g \in \Gr$,
\begin{align*}
    \E_\prob[&\norm{\score(\theta_\star, g; Z) - \score(\theta_\star, g)}_2^2] \leq 2\Ex{\prob}{\norm{\score(\theta_\star, g; Z)}_2^2} + 2\norm{\score(\theta_\star, g)}_2^2\\
    &\leq 12C_X^2\sigma^2 + \br{12\sigma^2+2(1 + \norm{\theta_\star}_2)+12C_X^2(1 + \norm{\theta_\star}_2^2)}r^2 + 12(1+\norm{\theta_\star}_2^2)r^4\\
    &=12C_X^2\sigma^2 + \bigO(r^2).
\end{align*}

On the other hand, since
\begin{align*}
    \norm{H_{\theta\theta}(\theta_\star, g; Z)}_2 &= \norm{(\tilde X - (g_X - g_{0,X})(W))(\tilde X - (g_X - g_{0,X})(W))^\top}_2\\
    &\leq \norm{\tilde X - (g_X - g_{0,X})(W)}_2^2\\
    &\leq 2\norm{\tilde X}_2^2 + 2\norm{(g_X - g_{0,X})(W)}_2^2 \leq 2C_X^2 + 2\norm{(g_X - g_{0,X})(W)}_2^2,
\end{align*}
by \eqref{eq: ex1 a3} we have
\begin{align*}
    \norm{H_{\theta\theta}(\theta_\star, g; Z) - H_{\theta\theta}(\theta_\star, g) }_2 &\leq \norm{H_{\theta\theta}(\theta_\star, g; Z)}_2 + \norm{H_{\theta\theta}(\theta_\star, g)}_2\\
    &\leq 3C_X^2+ r^2 + 2\norm{(g_X - g_{0,X})(W)}_2^2,
\end{align*}
which implies that
\begin{align*}
    \E_{\prob}[\Vert (H_{\theta\theta}(\theta_\star, g; Z) -& H_{\theta\theta}(\theta_\star, g))(\theta - \theta_\star) \Vert_2^2] \\
    &\leq \Ex{\prob}{(3C_X^2+ r^2 + 2\norm{(g_X - g_{0,X})(W)}_2^2)^2}\norm{\theta - \theta_\star}_2^2\\
    &=(9C_X^4+ \bigO(r^2))\norm{\theta - \theta_\star}_2^2.
\end{align*}
Thus,
\begin{align*}
\Ex{\prob}{\norm{\score(\theta, g; Z) - \score(\theta, g)}_2^2} \leq&  2\Ex{\prob}{\norm{\score(\theta_\star, g; Z) - \score(\theta_\star, g)}_2^2} \\
    &+ 2\Ex{\prob}{\norm{(H_{\theta\theta}(\theta_\star, g; Z) - H_{\theta\theta}(\theta_\star, g))(\theta - \theta_\star)}_2^2}\\
\leq&24C_X^2\sigma^2 + \bigO(r^2) + (18C_X^4+ \bigO(r^2))\norm{\theta - \theta_\star}_2^2,
\end{align*}
which implies 
\begin{align}\label{eq ex1 a4}
    \Kconst_1 = 24C_X^2\sigma^2 + \bigO(r^2) \text{ and } \kappaconst_1 = 18C_X^4+ \bigO(r^2).
\end{align}

(e) For any $\theta \in \Theta$ and $g, \bar g \in \Gr$, by \eqref{eq ex1 good} we have
\begin{align*}
    \D_g\D_\theta &L(\theta_\star, \bar g)[\theta - \theta_\star, g - g_0] \\
    &= \Ex{\prob}{(-(g_Y - g_{0,Y})(W) + \ip{\theta_\star, (g_X - g_{0,X})(W)})\ip{\theta - \theta_\star, (\bar g_X - g_{0,X})(W)}}\\
    &\quad + \Ex{\prob}{(-(\bar g_Y - g_{0,Y})(W) + \ip{\theta_\star, (\bar g_X - g_{0,X})(W)})\ip{\theta - \theta_\star, ( g_X - g_{0,X})(W)}}.
\end{align*}
Since $\bar g \in \Gr$,
\begin{align*}
    &\abs{\Ex{\prob}{(-(g_Y - g_{0,Y})(W) + \ip{\theta_\star, (g_X - g_{0,X})(W)})\ip{\theta - \theta_\star, (\bar g_X - g_{0,X})(W)}}}\\
    &\leq \Ex{\prob}{\abs{(g_Y - g_{0,Y})(W)}\norm{(\bar g_X - g_{0,X})(W)}_2}\norm{\theta - \theta_\star}_2 \\
    &\quad + \norm{\theta_\star}_2\Ex{\prob}{\norm{(g_X - g_{0,X})(W)}_2\norm{(\bar g_X - g_{0,X})(W)}_2}\norm{\theta - \theta_\star}_2\\
    &\leq \Ex{\prob}{\norm{(\bar g_X - g_{0,X})(W)}_2^2}^{1/2}\Ex{\prob}{((g_Y - g_{0,Y})(W))^2}^{1/2}\norm{\theta - \theta_\star}_2\\
    &\quad + \norm{\theta_\star}_2\Ex{\prob}{\norm{(\bar g_X - g_{0,X})(W)}_2^2}^{1/2}\Ex{\prob}{\norm{(g_X - g_{0,X})(W)}_2^2}^{1/2}\norm{\theta - \theta_\star}_2\\
    &\leq (1 + \norm{\theta_\star}_2)r\gnorm{g - g_0}\norm{\theta - \theta_\star}_2.
\end{align*}
Similarly,
\begin{align*}
    &\abs{\Ex{\prob}{(-(\bar g_Y - g_{0,Y})(W) + \ip{\theta_\star, (\bar g_X - g_{0,X})(W)})\ip{\theta - \theta_\star, ( g_X - g_{0,X})(W)}}}\\
    &\leq (1 + \norm{\theta_\star}_2)r\gnorm{g - g_0}\norm{\theta - \theta_\star}_2.
\end{align*}
Thus, 
\begin{align*}
   |\D_g\D_\theta &L(\theta_\star, \bar g)[\theta - \theta_\star, g - g_0]| \\
    &\leq \abs{\Ex{\prob}{(-(g_Y - g_{0,Y})(W) + \ip{\theta_\star, (g_X - g_{0,X})(W)})\ip{\theta - \theta_\star, (\bar g_X - g_{0,X})(W)}}}\\
    &\quad + \abs{\Ex{\prob}{(-(\bar g_Y - g_{0,Y})(W) + \ip{\theta_\star, (\bar g_X - g_{0,X})(W)})\ip{\theta - \theta_\star, ( g_X - g_{0,X})(W)}}}\\
    &\leq 2(1 + \norm{\theta_\star}_2)r\gnorm{g - g_0}\norm{\theta - \theta_\star}_2.
\end{align*}
which implies 
\begin{align}\label{eq: ex1 a5}
    \secsmooth_1 = 2(1 + \norm{\theta_\star}_2)r.
\end{align}

In addition, the risk $L$ is Neyman orthogonal at $(\theta_\star, g_0)$ since
\begin{align}
    \D_g\D_\theta &L(\theta_\star, g_0)[\theta - \theta_\star, g - g_0] = 0.
\end{align}
Note that 
\begin{align*}
    \D_g^2\D_\theta &L(\theta_\star, \bar g)[\theta - \theta_\star, g - g_0, g - g_0] \\
    &= 2\Ex{\prob}{(-(g_Y - g_{0,Y})(W) + \ip{\theta_\star, (g_X - g_{0,X})(W)})\ip{\theta - \theta_\star, (g_X - g_{0,X})(W)}}.
\end{align*}
By identical proof of \eqref{eq: ex1 a5}, we have that $L$ satisfies \Cref{assumption orthogonality main} since
\begin{align*}
    \D_g^2\D_\theta &L(\theta_\star, \bar g)[\theta - \theta_\star, g - g_0, g - g_0] \leq 2(1 + \norm{\theta_\star}_2)\gnorm{g - g_0}^2\norm{\theta - \theta_\star}_2,
\end{align*}
which implies 
\begin{align}\label{eq: ex1 assump 4}
    \highsmooth_1 = 2(1 + \norm{\theta_\star}_2).
\end{align}
\end{proof}

\subsubsection{Proof of Lemma \ref{lemma: example non-ortho PLM}}\label{sec: proof ex2}
\begin{proof}
We consider the following loss:
\begin{align*}
    \loss(\theta, g; z) = \frac{1}{2}(y  - g(w) - \ip{\theta, x})^2,
\end{align*}
with the corresponding risk function defined as
\begin{align*}
    \ploss(\theta, g) = \frac{1}{2}\Ex{\prob}{(Y  - g(W) - \ip{\theta, X})^2}
\end{align*}
Under the true nuisance, the target is the minimizer of the following squared loss:
\begin{align*}
    \theta_\star = \argmin_{\theta \in \R^d}\frac{1}{2}\Ex{\prob}{(Y  - g_0(W) - \ip{\theta, X})^2}.
\end{align*}
Since $\epsilon = Y  - g_0(W) - \ip{\theta_0, X}$ satisfies $\Ex{\prob}{\epsilon\mid X,W} = 0$ under the true model, by bias-variance decomposition, we have
\begin{align}\label{eq: ex2 target}
    \theta_\star = \argmin_{\theta \in \R^d}\frac{1}{2}\Ex{\prob}{(\ip{\theta_0, X} - \ip{\theta, X})^2} = \theta_0.
\end{align}
    Differentiating $\loss(\theta, g; z)$ with respect to $\theta$, we obtain the gradient and Hessian \wrt $\theta$ as
\begin{align*}
    \score(\theta, g; z) = -(y - g(w) - \ip{x, \theta})x \text{ and } H_{\theta\theta}(\theta, g; z) = xx^\top.
\end{align*}
The expected gradient and expected Hessian are then obtained as
\begin{align*}
    \score(\theta, g) = -\Ex{\prob}{(Y - g(W) - \ip{X, \theta})X} \text{ and } H_{\theta\theta}(\theta, g) = \Ex{\prob}{XX^\top}.
\end{align*}
We consider the nuisance neighborhood such that for $g \in \Gr$,
\begin{align}\label{eq: ex2 Gr}
   \gnorm{g - g_0} := \Ex{\prob}{(g(W) - g_0(W))^2}^{1/2} \leq r.
\end{align}

We now verify that the loss function $\loss$ satisfies \Cref{assumption of population risk}.

(a) We assume that $g: \calW \mapsto \R$ is continuous, thus \Cref{assumption of population risk}(a) is satisfied.

(b) Since $\theta_\star = \theta_0$ by \eqref{eq: ex2 target}, we have
\begin{align}\label{eq: ex2 a3(b)}
    \score(\theta_\star, g_0) = -\Ex{\prob}{\epsilon X} = 0.
\end{align}

(c) When $\lambda_{\min}(\Ex{\prob}{XX^\top}) \geq \lambda_0 > 0$ and $\norm{X}_2 \leq C_X$ \as, $L(\theta, g)$ is $\lambda_0$-strongly convex and $C_X^2$-smooth since
\begin{align}\label{eq: ex2 a3(c)}
    \lambda_0\mathbf{I} \preccurlyeq H_{\theta\theta}(\theta, g) \preccurlyeq C_X^2\mathbf{I} \implies \sconv = \lambda_0 \text{ and } \smooth = C_X^2.
\end{align}

(d) Consider the Taylor expansion around $\theta_\star$, we have
\begin{align*}
    \score(\theta, g; Z) - \score(\theta, g) = \score(\theta_\star, g; Z) - \score(\theta_\star, g) + (H_{\theta\theta}(\theta_\star, g; Z) - H_{\theta\theta}(\theta_\star, g))(\theta - \theta_\star).
\end{align*}
Since $\score(\theta_\star, g; Z) = ((g - g_0)(w) - \epsilon)X$ and $\norm{X}_2 \leq C_X$ \as, we have
 \begin{align*}
     \norm{\score(\theta_\star, g; Z) - \score(\theta_\star, g)}_2 &= \norm{((g - g_0)(W) - \epsilon)X - \Ex{\prob}{((g - g_0)(W))X}}_2\\
     &\leq C_X\p{\abs{(g - g_0)(W)}  + \Ex{\prob}{\abs{(g - g_0)(W)}}+ \abs{\epsilon}}.
 \end{align*}
On the other hand, 
\begin{align*}
    H_{\theta\theta}(\theta_\star, g; Z) - H_{\theta\theta}(\theta_\star, g) = XX^\top - \E\sbr{XX^\top}  \preccurlyeq 2C_X^2\mathbf{I},
\end{align*}
which implies that 
\begin{align*}
    \norm{(H_{\theta\theta}(\theta_\star, g; Z) - H_{\theta\theta}(\theta_\star, g))(\theta - \theta_\star)}_2 \leq 2C_X^2\norm{\theta - \theta_\star}_2.
\end{align*}
For $g \in \Gr$, when $\Ex{\prob}{\epsilon^2}\leq \sigma^2$ we have
\begin{align*}
    \Ex{\prob}{\norm{\score(\theta, g; Z) - \score(\theta, g)}_2^2} \leq&  2\Ex{\prob}{\norm{\score(\theta_\star, g; Z) - \score(\theta_\star, g)}_2^2}\\
    &+ 2\Ex{\prob}{\norm{(H_{\theta\theta}(\theta_\star, g; Z) - H_{\theta\theta}(\theta_\star, g))(\theta - \theta_\star)}_2^2}\\
    \leq& 2C_X^2\Ex{\prob}{\p{\abs{(g - g_0)(W)}  + \Ex{\prob}{(g - g_0)(W)}+ \abs{\epsilon}}^2} \\
    &+ 2C_X^4\norm{\theta - \theta_\star}_2^2\\
    \leq& 6C_X^2\p{2\Ex{\prob}{((g - g_0)(W))^2} + \Ex{\prob}{\epsilon^2}} + 2C_X^4\norm{\theta - \theta_\star}_2^2\\
    \leq& 6C_X^2\p{2r^2 + \sigma^2} + 2C_X^4\norm{\theta - \theta_\star}_2^2,
\end{align*}
which implies that 
\begin{align}\label{eq: ex2 a3(d)}
    \Kconst_1 = 6C_X^2\p{2r^2 + \sigma^2} \text{ and } \kappaconst_1 = 2C_X^4.
\end{align}

(e) For any $\theta \in \Theta$ and $g, \bar g \in \Gr$, we have
\begin{align*}
    \abs{\D_g\D_\theta L(\theta, \bar g)[\theta - \theta_\star, g - g_0]} &= \abs{\Ex{\prob}{(g - g_0)(W)\ip{X, \theta - \theta_\star}}}\\
    &\leq \Ex{\prob}{\abs{(g - g_0)(W)\ip{X, \theta - \theta_\star}}}\\
    &\leq C_X\norm{\theta - \theta_\star}_2\Ex{\prob}{((g - g_0)(W))^2}^{1/2},
\end{align*}
which implies that 
\begin{align}\label{eq: ex2 a3(e)}
    \secsmooth_1 = C_X.
\end{align}
\end{proof}

\subsubsection{Proof of Lemma \ref{lem eg osgd}}\label{sec: proof osgd}
\begin{proof}
We consider the following loss:
\begin{align*}
    \loss(\theta, g; z) = \frac{1}{2}(y  - g(w) - \ip{\theta, x})^2,
\end{align*}
with the corresponding risk function defined as
\begin{align*}
    \ploss(\theta, g) = \frac{1}{2}\Ex{\prob}{(Y  - g(W) - \ip{\theta, X})^2}.
\end{align*}
First by the same proof as \Cref{sec: proof ex2}, we have $\theta_\star = \theta_0$. Define the inner product $\ip{\cdot, \cdot}_\G$ as \eqref{inner product} and define the norm $\norm{\cdot}_\G$ such that $\norm{g}_\G^2 = \ip{g, g}_\G \forall g \in \G$. Consider a uniformly bounded neighborhood $\Gr$ such that
\begin{align}\label{osgd gr}
    \Gr = \br{g \in \G: \abs{g(W) - g_0(W)} \leq r \textit{ almost surely}}.
\end{align}
The NO gradient oracle for this non-orthogonal loss is derived as \eqref{ex2 noscore} such that
\begin{align*}
    \noscore(\theta, g; z) 
    &= -(y - g(w) - \ip{\theta, x})(x - \Ex{}{X\mid W=w}).
\end{align*}
We now verify that \Cref{asm:noscore} is satisfied.

(a) Since $\epsilon = Y  - g_0(W) - \ip{\theta_0, X}$ satisfies $\Ex{\prob}{\epsilon\mid X,W} = 0$ under the true model, by \eqref{ex2 noscore} we first have 
\begin{align}
    \noscore(\theta_\star, g_0) &= \Ex{\prob}{\noscore(\theta_\star, g_0; Z)} \notag \\
    &= -\Ex{\prob}{\epsilon(X - \Ex{}{X\mid W})} \notag \\
    &= -\Ex{\prob}{\Ex{\prob}{\epsilon \mid X, W}(X - \Ex{}{X\mid W})} = 0. \label{eq osgd asm a1}
\end{align}
Let $\gamma_0\pow{j} = H_{gg}^{-1}H_{\theta g}\pow{j}$ for $j = 1,\dots, d$. By \eqref{appx gamma}, we have $[\pdir_0 g]_j = \ip{ \gamma_0^{(j)}, g}_\G, \forall g \in \G$. Thus, by \eqref{eq Riesz Dg} we have
\begin{align*}
    [(\pdirhat - \pdir_0)\nabla_g L(\theta_\star, g_0)]_j = \ip{\hat\gamma\pow{j} -\gamma_0\pow{j}, \nabla_g L(\theta_\star, g_0)}_\G = \D_g L(\theta_\star, g_0)[\hat\gamma\pow{j} -\gamma_0\pow{j}],
\end{align*}
which, by \eqref{eq Dg1}, implies that 
\begin{align}\label{eq osgd asm a2}
    [(\pdirhat - \pdir_0)\nabla_g L(\theta_\star, g_0)]_j = -\Ex{\prob}{\epsilon [\pdirhat - \pdir_0]_j(W)} = \Ex{\prob}{\Ex{\prob}{\epsilon\mid W} [\pdirhat - \pdir_0]_j(W)} = 0.
\end{align}
Thus, \Cref{asm:noscore}(a) holds true due to \eqref{eq osgd asm a1} and \eqref{eq osgd asm a2}.

(b) By \eqref{ex2 noscore}, for any $(\theta, g) \in \Theta \times \G$,
\begin{align}
    \nabla_\theta \noscore(\theta, g) 
    &= \Ex{\prob}{X(X - \Ex{\prob}{X\mid W})^\top} \notag\\
    &= \Ex{\prob}{XX^\top} - \Ex{\prob}{\Ex{\prob}{X\mid W}\Ex{\prob}{X\mid W}^\top}\notag\\
    &=\Ex{\prob}{\tilde X \tilde X^\top},
\end{align}
which implies that
\begin{align*}
    \lambda_{\min}{\nabla_\theta \noscore(\theta, g) + \nabla_\theta \noscore(\theta, g)^\top} = \lambda_{\min}{2\Ex{\prob}{\tilde X \tilde X^\top}} \geq 2\lambda_0.
\end{align*}
Thus, \Cref{asm:noscore}(b) holds true for $\hessianconst = \lambda_0$.

(c) For any $(\theta, g, \bar g) \in \Theta \times \G \times \Gr$, by \eqref{eq Dg1},
\begin{align*}
    &\Ex{\prob}{(\D_g L(\theta, \bar g; Z)[g] - \D_g L(\theta, \bar g)[g])^2}\\ &\leq \Ex{\prob}{(\D_g L(\theta, \bar g; Z)[g])^2}\\
    &= \Ex{\prob}{(Y - \bar g(W) - \ip{\theta, X})^2\p{g(W)}^2}\\
    &= \Ex{\prob}{(\epsilon - (\bar g - g_0)(W) - \ip{\theta - \theta_\star, X})^2\p{g(W)}^2}\\
    &\leq 3\Ex{\prob}{\epsilon^2\p{g(W)}^2} + 3\Ex{\prob}{(\bar g - g_0)(W))^2\p{g(W)}^2} + 3\Ex{\prob}{\ip{\theta - \theta_\star, X}^2\p{g(W)}^2}.
\end{align*}
Assume that $\Ex{\prob}{\epsilon^2\mid W} \leq \sigma^2$ and $\norm{X}_\infty \leq C_X$ \as. By \eqref{osgd gr} we have
\begin{align}
    \Ex{\prob}{(\D_g L(\theta, \bar g; Z)[g] - \D_g L(\theta, \bar g)[g])^2} \leq 3(\sigma^2 + r^2 + C_X^2\norm{\theta - \theta_\star}_2^2)\gnorm{g}^2.
\end{align}
Thus, \Cref{asm:noscore}(c) holds true for $\Kconst_2= 3(\sigma^2 + r^2)$ and $\kappaconst_2 = 3C_X^2$.

(d) For any $(\theta, \bar g, g_1, g_2) \in \Theta \times \Gr \times \G \times \G$, by \eqref{eq Dg2}, 
\begin{align}
    \abs{\D_g^2 L(\theta, g)[g_1, g_2]} &= \abs{\Ex{\prob}{g_1(W)g_2(W)}} \notag \\
    &\leq \Ex{\prob}{(g_1(W))^2}^{1/2}\Ex{\prob}{(g_2(W))^2}^{1/2} = \gnorm{g_1}\gnorm{g_2}. \label{eq asm5 d1}
\end{align}
In addition, for any $(\theta, \bar \theta, g) \in \Theta \times \Theta \times \G$, by \eqref{eq Dg1} we have
\begin{align}
    \abs{\D_\theta\D_g\ploss(\bar \theta, g_0)[g, \theta - \theta_\star]} &= \abs{\D_{\bar \theta}\Ex{\prob}{(Y - g_0(W) - \ip{\bar \theta, X})g(W)}[\theta - \theta_\star]}\notag\\
    &=\abs{\Ex{\prob}{\ip{\theta - \theta_\star, X} g(W)}}\notag\\
    &\leq \Ex{\prob}{\ip{\theta - \theta_\star, X}^2}^{1/2}\Ex{\prob}{(g(W))^2}^{1/2}\notag\\
    &\leq C_X\norm{\theta - \theta_\star}_2\gnorm{g}.\label{eq asm5 d2}
\end{align}
Thus, \Cref{asm:noscore}(d) holds true for $\secsmooth_2= 1$ due to \eqref{eq asm5 d1} and $\secsmooth_1 = C_X$ due to \eqref{eq asm5 d2}.

(e) Note that 
\begin{align*}
    \noscore(\theta, g; z) 
    &= -(y - g(w) - \ip{\theta, x})(x - \Ex{\prob}{X\mid W=w}),
\end{align*}
which implies that for any $g_1, g_2 \in \G$,
\begin{align}
    D_g^2 \noscore(\theta, g; z)[g_1, g_2] = 0. 
\end{align}
Thus, \Cref{asm:noscore}(e) holds true for $\highsmooth_2 = 0$.
\end{proof}

\subsubsection{Proof of Lemma \ref{lemma: example unres CATE}}\label{sec: proof ex3}
\begin{proof}
    We consider the following loss:
\begin{align*}
    \loss(\theta, g; z) = \frac{1}{2} \p{y - g^{\text{out}}(x) - \p{w - g^{\text{prop}}(x)}\ip{\theta, x}}^2,
\end{align*}
with the corresponding risk function defined as
\begin{align*}
    \ploss(\theta, g) = \frac{1}{2}\Ex{\prob}{\p{Y - g^{\text{out}}(X) - \p{W - g^{\text{prop}}(X)}\ip{\theta, X}}^2}.
\end{align*}
Note that $\epsilon = Y - g_0^{\text{out}}(X) - \p{W- g_0^{\text{prop}}(X)}\tau_0(X)$. Under \Cref{asm causal}, we have $\Ex{\prob}{\epsilon \mid W,X} = 0$, which implies that 
\begin{align*}
    \ploss(\theta, g_0) &= \frac{1}{2}\Ex{\prob}{\p{\epsilon + \p{W - g_0^{\text{prop}}(X)}(\tau_0(X)-\ip{\theta, X})}^2}\\
    &= \frac{1}{2}\Ex{\prob}{\p{W - g_0^{\text{prop}}(X)}^2(\tau_0(X)-\ip{\theta, X})^2} + \frac{1}{2}\Ex{\prob}{\epsilon^2}.
\end{align*}
Thus, the target is the minimizer of the following squared loss:
\begin{align}\label{eq: ex3 target}
    \theta_\star = \argmin_{\theta \in \R^d}\Ex{\prob}{\p{W - g_0^{\text{prop}}(X)}^2(\tau_0(X)-\ip{\theta, X})^2}.
\end{align}

    Differentiating $\loss(\theta, g; z)$ with respect to $\theta$, we obtain the gradient and Hessian \wrt $\theta$ as
\begin{align*}
    \score(\theta, g; z) &= - \p{y - g^{\text{out}}(x) - \p{w - g^{\text{prop}}(x)}\ip{\theta, x}}\p{w - g^{\text{prop}}(x)}x, \\
    H_{\theta\theta}(\theta, g; z) &= \p{w - g^{\text{prop}}(x)}^2xx^\top.
\end{align*}
The expected gradient and expected Hessian are then obtained as
\begin{align*}
    \score(\theta, g) &= -\Ex{\prob}{\p{Y - g^{\text{out}}(X) - \p{W - g^{\text{prop}}(X)}\ip{\theta, X}}\p{W - g^{\text{prop}}(X)}X}, \\
    H_{\theta\theta}(\theta, g) &= \Ex{\prob}{\p{g_0^{\text{prop}}(1 - g_0^{\text{prop}})(X) + ((g^{\text{prop}} - g_0^{\text{prop}})(X))^2}XX^\top}.
\end{align*}
We consider the nuisance neighborhood such that for $g \in \Gr$,
\begin{align*}
   \gnorm{g - g_0} := \max\br{\Ex{\prob}{(g^{\text{out}}(X) - g_0^{\text{out}}(X))^4}^{\frac{1}{4}}, \Ex{\prob}{(g^{\text{out}}(X) - g_0^{\text{out}}(X))^4}^{\frac{1}{4}}} \leq r.
\end{align*}

We now verify that the loss function $\loss$ satisfies \Cref{assumption of population risk}.

(a) We assume that $g^{\text{out}}: \R^d \mapsto \R$ and $g^{\text{prop}}: \R^d \mapsto \R$ are continuous, thus \Cref{assumption of population risk}(a) is satisfied.

(b) Since $\theta_\star$ is a global minimizer of \eqref{eq: ex3 target}, we have
\begin{align}\label{eq: ex3 a3(b)}
    \score(\theta_\star, g_0) = 0.
\end{align}

(c) We assume that $c_0 \leq g_0^{\text{prop}}(X) \leq 1-c_0$ \as for some $c_0 > 0$. When $\lambda_{\min}(\Ex{\prob}{XX^\top}) \geq \lambda_0 > 0$ and $\norm{X}_2 \leq C_X$ \as, we have 
\begin{align}\label{eq: ex3 a3(c)}
    c_0^2\lambda_0\mathbf{I} \preccurlyeq H_{\theta\theta}(\theta, g) \preccurlyeq (1+r^2)C_X^2\mathbf{I} \implies \sconv = c_0^2\lambda_0 \text{ and } \smooth = (1+r^2)C_X^2.
\end{align}

(d) Consider the Taylor expansion around $\theta_\star$, we have
\begin{align*}
    \score(\theta, g; Z) - \score(\theta, g) = \score(\theta_\star, g; Z) - \score(\theta_\star, g) + (H_{\theta\theta}(\theta_\star, g; Z) - H_{\theta\theta}(\theta_\star, g))(\theta - \theta_\star).
\end{align*}
Note that
\begin{align*}
    \score(\theta_\star, g; Z) =& - \p{Y - g^{\text{out}}(X) - \p{W - g^{\text{prop}}(X)}\ip{\theta_\star, X}}\p{W - g^{\text{prop}}(X)}X\\
    =&-\p{\epsilon - (g^{\text{out}} - g_0^{\text{out}})(X) + (g^{\text{prop}} - g_0^{\text{prop}})(X)\tau_0(X)}\p{W - g^{\text{prop}}(X)}X\\
    &+\p{W - g^{\text{prop}}(X)}^2( \ip{\theta_\star, X} - \tau_0(X))X.
\end{align*}
We assume that $\tau_0: \R^d \mapsto \R$ is continuous. Then when $\norm{X}_2 \leq C_X$ \as, $\abs{\tau_0(X)} \leq C_\tau$ for some $C_\tau > 0$. It follows that
\begin{align*}
    \norm{\score(\theta_\star, g; Z)}_2 \leq& C_X\abs{\p{\epsilon - (g^{\text{out}} - g_0^{\text{out}})(X) + (g^{\text{prop}} - g_0^{\text{prop}})(X)\tau_0(X)}\p{W - g^{\text{prop}}(X)}}\\
    &+ C_X\p{W - g^{\text{prop}}(X)}^2\abs{\ip{\theta_\star, X} - \tau_0(X)}\\
    \leq&C_X\p{\abs{\epsilon} + \abs{(g^{\text{out}} - g_0^{\text{out}})(X)} + C_\tau\abs{(g^{\text{prop}} - g_0^{\text{prop}})(X)}} \abs{W - g^{\text{prop}}(X)}\\
    &+ C_X(C_X\norm{\theta_\star}_2 + C_\tau)\p{W - g^{\text{prop}}(X)}^2.
\end{align*}
Since $\Ex{\prob}{(W - g^{\text{prop}}(X))^2\mid X} = g_0^{\text{prop}}(1-g_0^{\text{prop}})(X) + ((g^{\text{prop}}- g_0^{\text{prop}})(X))^2$ and $(W - g^{\text{prop}}(X))^2 \leq 2 + 2((g^{\text{prop}} - g_0^{\text{prop}})(X))^2$, we have
\begin{align*}
    \Ex{\prob}{\norm{\score(\theta_\star, g; Z)}_2^2} \leq& 4C_X^2\Ex{\prob}{\epsilon^2\p{1+((g^{\text{prop}}- g_0^{\text{prop}})(X))^2}} \\
    &+ 4C_X^2\Ex{\prob}{((g^{\text{out}} - g_0^{\text{out}})(X))^2\p{1+((g^{\text{prop}}- g_0^{\text{prop}})(X))^2}}\\
    &+4C_X^2C_\tau^2\Ex{\prob}{((g^{\text{prop}} - g_0^{\text{prop}})(X))^2\p{1+((g^{\text{prop}}- g_0^{\text{prop}})(X))^2}}\\
    &+ 4C_X^2(C_X\norm{\theta_\star}_2 + C_\tau)^2\Ex{\prob}{4+4((g^{\text{prop}}- g_0^{\text{prop}})(X))^4}.
\end{align*}
When $\Ex{\prob}{\epsilon^4} \leq \sigma^4$, by H{\" o}lder inequality, 
\begin{align*}
     \Ex{\prob}{\norm{\score(\theta_\star, g; Z)}_2^2} \leq& 4C_X^2(\sigma^2 + 4(C_X\norm{\theta_\star}_2 + C_\tau)^2) \\
     &+ 4C_X^2(1+\sigma^2+C_\tau^2)r^2 + 4C_X^2( 1 + C_\tau^2 + 4(C_X\norm{\theta_\star}_2 + C_\tau)^2)r^4.
\end{align*}

Similarly, use the fact that $\Ex{\prob}{\epsilon \mid W,X} = 0$ and by the stationary condition of \eqref{eq: ex3 target}, we have
 \begin{align*}
     \score(\theta_\star, g) =& -\Ex{\prob}{\p{  (g^{\text{out}} - g_0^{\text{out}})(X) - (g^{\text{prop}} - g_0^{\text{prop}})(X)\tau_0(X)}\p{g^{\text{prop}} - g_0^{\text{prop}}}(X)X}\\
    &+\Ex{\prob}{\p{(W - g_0^{\text{prop}}(X))^2+((g^{\text{prop}} -g_0^{\text{prop}})(X))^2}( \ip{\theta_\star, X} - \tau_0(X))X}\\
    =& -\Ex{\prob}{\p{  (g^{\text{out}} - g_0^{\text{out}})(X) - (g^{\text{prop}} - g_0^{\text{prop}})(X)\tau_0(X)}\p{g^{\text{prop}} - g_0^{\text{prop}}}(X)X}\\
    &+\Ex{\prob}{((g^{\text{prop}} -g_0^{\text{prop}})(X))^2( \ip{\theta_\star, X} - \tau_0(X))X},
 \end{align*}
 which implies
 \begin{align*}
     \norm{\score(\theta_\star, g)}_2^2 \leq 3C_X^2(1 + C_\tau^2 + (C_X\norm{\theta_\star}_2 + C_\tau)^2)r^4.
 \end{align*}
On the other hand, 
\begin{align*}
    H_{\theta\theta}(\theta_\star, g; Z) - H_{\theta\theta}(\theta_\star, g) &\preccurlyeq  \p{W - g^{\text{prop}}(X)}^2XX^\top +  (1+r^2)C_X^2\mathbf{I}\\
    &\preccurlyeq C_X^2(3 + r^2+ 2((g^{\text{prop}} - g_0^{\text{prop}})(X))^2)\mathbf{I},
\end{align*}
which implies that 
\begin{align*}
   \Ex{\prob}{ \norm{(H_{\theta\theta}(\theta_\star, g; Z) - H_{\theta\theta}(\theta_\star, g))(\theta - \theta_\star)}_2^2}\leq (9C_X^4 + \bigO(r^2))\norm{\theta - \theta_\star}_2^2.
\end{align*}
Thus,
\begin{align*}
    \Ex{\prob}{\norm{\score(\theta, g; Z) - \score(\theta, g)}_2^2} \leq& 3 \Ex{\prob}{\norm{\score(\theta_\star, g; Z)}_2^2} + 3\norm{\score(\theta_\star, g)}_2^2\\
    &+3\Ex{\prob}{ \norm{(H_{\theta\theta}(\theta_\star, g; Z) - H_{\theta\theta}(\theta_\star, g))(\theta - \theta_\star)}_2^2},
\end{align*}
which implies 
\begin{align}\label{eq: ex3 a3(d)}
    \Kconst_1 = 12C_X^2(\sigma^2 + 4(C_X\norm{\theta_\star}_2 + C_\tau)^2) + \bigO(r^2) \text{ and } \kappaconst_1 = 27C_X^4 + \bigO(r^2).
\end{align}
(e) For any $\theta \in \Theta$ and $g, \bar g \in \Gr$, we have
\begin{align*}
    \D_g \D_\theta &L(\theta_\star, \bar g)[\theta - \theta_\star, g - g_0] \\
    &= -\Ex{\prob}{\p{(\bar g^{\text{out}} - g_0^{\text{out}})(g^{\text{prop}} - g_0^{\text{prop}})(X) + (g^{\text{out}} - g_0^{\text{out}})(\bar g^{\text{prop}} - g_0^{\text{prop}})(X)}\ip{X, \theta - \theta_\star}}\\
    &\quad +2\Ex{\prob}{\tau_0(X)(g^{\text{prop}} - g_0^{\text{prop}})(\bar g^{\text{prop}} - g_0^{\text{prop}})(X)\ip{X, \theta - \theta_\star}}\\
    &\quad +2\Ex{\prob}{(g^{\text{prop}} -g_0^{\text{prop}})(\bar g^{\text{prop}} -g_0^{\text{prop}})(X)( \ip{\theta_\star, X} - \tau_0(X))\ip{X, \theta - \theta_\star}}.
\end{align*}
Thus,
\begin{align*}
     |\D_g \D_\theta &L(\theta_\star, \bar g)[\theta - \theta_\star, g - g_0]|\\
     &\leq C_X\norm{\theta - \theta_\star}_2\Ex{\prob}{\abs{(\bar g^{\text{out}} - g_0^{\text{out}})(g^{\text{prop}} - g_0^{\text{prop}})(X)} + \abs{(g^{\text{out}} - g_0^{\text{out}})(\bar g^{\text{prop}} - g_0^{\text{prop}})(X)}}\\
     &\quad + 2C_X (2C_\tau + C_X\norm{\theta_\star}_2) \norm{\theta - \theta_\star}_2\Ex{\prob}{\abs{(g^{\text{prop}} - g_0^{\text{prop}})(\bar g^{\text{prop}} - g_0^{\text{prop}})(X)}}\\
     &\leq C_X(1 + 4(C_X\norm{\theta_\star}_2 + C_\tau))r\norm{\theta - \theta_\star}_2\gnorm{g - g_0},
\end{align*}
which implies 
\begin{align}\label{eq: ex3 a3(e)}
    \secsmooth_1 = C_X(1 + 4(C_X\norm{\theta_\star}_2 + C_\tau))r.
\end{align}

In addition, $L(\theta, g)$ is Neyman orthogonal at $(\theta_\star, g_0)$ since
\begin{align*}
    \D_g \D_\theta L(\theta_\star, g_0)[\theta - \theta_\star, g - g_0] = 0.
\end{align*}
Since for any $\theta \in \Theta$ and $g, \bar g \in \Gr$,
\begin{align*}
    \D_g^2 \D_\theta &L(\theta_\star, \bar g)[\theta - \theta_\star, g - g_0, g - g_0] \\
    &= -\Ex{\prob}{\p{( g^{\text{out}} - g_0^{\text{out}})(g^{\text{prop}} - g_0^{\text{prop}})(X) + (g^{\text{out}} - g_0^{\text{out}})( g^{\text{prop}} - g_0^{\text{prop}})(X)}\ip{X, \theta - \theta_\star}}\\
    &\quad +2\Ex{\prob}{\tau_0(X)((g^{\text{prop}} - g_0^{\text{prop}})(X))^2\ip{X, \theta - \theta_\star}}\\
    &\quad +2\Ex{\prob}{((g^{\text{prop}} -g_0^{\text{prop}})(X))^2( \ip{\theta_\star, X} - \tau_0(X))\ip{X, \theta - \theta_\star}}.
\end{align*}
Similarly, we can show that 
\begin{align*}
    |\D_g^2 \D_\theta &L(\theta_\star, \bar g)[\theta - \theta_\star, g - g_0, g - g_0]|\\
    &\leq C_X(1 + 4(C_X\norm{\theta_\star}_2 + C_\tau))\norm{\theta - \theta_\star}_2\gnorm{g - g_0}^2,
\end{align*}
which implies that
\begin{align}\label{eq: ex3 a4 highsmooth}
    \highsmooth_1 = C_X(1 + 4(C_X\norm{\theta_\star}_2 + C_\tau)).
\end{align}
\end{proof}

\subsubsection{Proof of Lemma \ref{lemma: example res CATE}}\label{sec: proof ex4}
\begin{proof}
Let 
\begin{align*}
    \phi(g; z) = g\spow{1}(x) - g\spow{0}(x) + \frac{w-g^{\text{prop}}(x)}{g^{\text{prop}}(x)(1-g^{\text{prop}}(x))}(y - g\spow{w}(x)).
\end{align*}
We consider the following loss:
\begin{align*}
    \loss(\theta, g; z) = \frac{1}{2}\p{\phi(g; z) - \ip{\theta, x}}^2,
\end{align*}
with the corresponding risk function defined as
\begin{align*}
    \ploss(\theta, g) = \frac{1}{2}\Ex{\prob}{\p{\phi(g; z) - \ip{\theta, x}}^2}.
\end{align*}
We define the residual $\epsilon$ as
\begin{align*}
    \epsilon = \frac{W-g_0^{\text{prop}}(X)}{g_0^{\text{prop}}(X)(1-g_0^{\text{prop}}(X))}(Y - g_0\spow{W}(X)).
\end{align*}
Under \Cref{asm causal}, we have $\Ex{\prob}{\epsilon \mid W,X} = 0$. Since $\tau_0(x) = g_0\spow{1}(x) - g_0\spow{0}(x)$, we have
\begin{align*}
    \ploss(\theta, g_0) &= \frac{1}{2}\Ex{\prob}{\p{\epsilon + \tau_0(X)-\ip{\theta, X}}^2}\\
    &= \frac{1}{2}\Ex{\prob}{(\tau_0(X)-\ip{\theta, X})^2} + \frac{1}{2}\Ex{\prob}{\epsilon^2}.
\end{align*}
Thus, the target is the minimizer of the following squared loss:
\begin{align}\label{eq: ex4 target}
    \theta_\star = \argmin_{\theta \in \R^d}\Ex{\prob}{(\tau_0(X)-\ip{\theta, X})^2}.
\end{align}

    Differentiating $\loss(\theta, g; z)$ with respect to $\theta$, we obtain the gradient and Hessian \wrt $\theta$ as
\begin{align*}
    \score(\theta, g; z) &= - \p{\phi(g; z) - \ip{\theta, x}}x, \\
    H_{\theta\theta}(\theta, g; z) &= xx^\top.
\end{align*}
The expected gradient and expected Hessian are then obtained as
\begin{align*}
    \score(\theta, g) &= -\Ex{\prob}{\p{\phi(g; Z) - \ip{\theta, X}}X}, \\
    H_{\theta\theta}(\theta, g) &= \Ex{\prob}{XX^\top}.
\end{align*}
We consider the nuisance neighborhood such that for $g \in \Gr$,
\begin{align*}
   \gnorm{g - g_0} := \max\br{\Ex{\prob}{\p{\frac{(g\pow{w} - g_0\pow{w})(X)}{g\pow{w}(1-g\pow{w})(X)}}^4}^{\frac{1}{4}}, \Ex{\prob}{\p{\frac{(g^{\text{prop}} - g_0^{\text{prop}})(X)}{g^{\text{prop}}(1-g^{\text{prop}})(X)}}^4}^{\frac{1}{4}}} \leq r.
\end{align*}

We now verify that the loss function $\loss$ satisfies \Cref{assumption of population risk}.

(a) We assume that $g\pow{w}: \R^d \mapsto \R, w = 0,1,$ and $g^{\text{prop}}: \R^d \mapsto (0,1)$ are continuous, thus \Cref{assumption of population risk}(a) is satisfied.

(b) Since $\theta_\star$ is a global minimizer of \eqref{eq: ex4 target}, we have
\begin{align}\label{eq: ex4 a3(b)}
    \score(\theta_\star, g_0) = 0.
\end{align}

(c) When $\lambda_{\min}(\Ex{\prob}{XX^\top}) \geq \lambda_0 > 0$ and $\norm{X}_2 \leq C_X$ \as, we have 
\begin{align}\label{eq: ex4 a3(c)}
    \lambda_0\mathbf{I} \preccurlyeq H_{\theta\theta}(\theta, g) \preccurlyeq C_X^2\mathbf{I} \implies \sconv = \lambda_0 \text{ and } \smooth = C_X^2.
\end{align}

(d) Consider the Taylor expansion around $\theta_\star$, we have
\begin{align*}
    \score(\theta, g; Z) - \score(\theta, g) = \score(\theta_\star, g; Z) - \score(\theta_\star, g) + (H_{\theta\theta}(\theta_\star, g; Z) - H_{\theta\theta}(\theta_\star, g))(\theta - \theta_\star).
\end{align*}
Let $\tau = g\pow{1} - g\pow{0}$ and 
\begin{align*}
    \psi(g; x) &= \frac{1}{g^{\text{prop}}(1-g^{\text{prop}})(x)} - \frac{1}{g_0^{\text{prop}}(1-g_0^{\text{prop}})(x)}\\
    &=\frac{(g^{\text{prop}} - g_0^{\text{prop}})(x)}{g^{\text{prop}}(1-g^{\text{prop}})(x)}\cdot\frac{(g^{\text{prop}}+g_0^{\text{prop}})(x) - 1}{g_0^{\text{prop}}(1-g_0^{\text{prop}})(x)}.
\end{align*}
Under \Cref{asm causal}, we have
\begin{align*}
    \abs{\psi(g; X)} \leq 2c_0^{-2}\abs{\frac{(g^{\text{prop}} - g_0^{\text{prop}})(X)}{g^{\text{prop}}(1-g^{\text{prop}})(X)}}.
\end{align*}
We can decompose $\score(\theta_\star, g; Z)$ as
\begin{align*}
    \score(\theta_\star, g; Z)&=I_1 + I_2 + I_3 + I_4 + I_5 + I_6 + I_7 + I_8 + I_9,
\end{align*}
where
\begin{align*}
    I_1 &= -\p{(\tau-\tau_0)(X) + \tau_0(X) - \ip{\theta_\star, X}}X,\\
    I_2 &= - \psi(g; X)(W - g_0^{\text{prop}}(X))(Y - g_0\pow{W}(X))X,\\
     I_3 &= \psi(g; X)(W - g_0^{\text{prop}}(X))(g\pow{W} - g_0\pow{W})(X)X,\\
      I_4 &= \psi(g; X)(g^{\text{prop}}- g_0^{\text{prop}})(X)(Y - g_0\pow{W}(X))X,\\
       I_5 &= - \psi(g; X)(g^{\text{prop}}- g_0^{\text{prop}})(g\pow{W} - g_0\pow{W})(X)X,\\
        I_6 &= - \frac{(W - g_0^{\text{prop}}(X))(Y - g_0\pow{W}(X))}{g_0^{\text{prop}}(1-g_0^{\text{prop}})(X)}X = -\epsilon X,\\
         I_7 &= \frac{(W - g_0^{\text{prop}}(X))(g\pow{W} - g_0\pow{W})(X)}{g_0^{\text{prop}}(1-g_0^{\text{prop}})(X)}X,\\
          I_8 &= \frac{(g^{\text{prop}}- g_0^{\text{prop}})(X)(Y - g_0\pow{W}(X))}{g_0^{\text{prop}}(1-g_0^{\text{prop}})(X)}X,\\
           I_9 &= - \frac{(g^{\text{prop}}- g_0^{\text{prop}})(g\pow{W} - g_0\pow{W})(X)}{g_0^{\text{prop}}(1-g_0^{\text{prop}})(X)}X.
\end{align*}
For $I_1$, when $\norm{X}_2 \leq C_X$ \as, we have $\abs{\tau_0(X) - \ip{\theta_\star, X}} \leq C$ for some $C>0$ and 
\begin{align*}
    \norm{I_1}_2 \leq C_X(\abs{(\tau - \tau_0)(X)} + C),
\end{align*}
which implies 
\begin{align*}
    \Ex{\prob}{\norm{I_1}_2^2} \leq 2C_X^2(2r^2 + C^2).
\end{align*}

For $I_2$, when $\abs{Y - g_0\pow{W}(X)} \leq C_Y$ \as, we have
\begin{align*}
   \norm{I_2}_2 \leq  4c_0^{-2}C_XC_Y\abs{\frac{(g^{\text{prop}} - g_0^{\text{prop}})(X)}{g^{\text{prop}}(1-g^{\text{prop}})(X)}},
\end{align*}
which implies 
\begin{align*}
    \Ex{\prob}{\norm{I_2}_2^2} \leq 16c_0^{-4}(C_XC_Y)^2r^2.
\end{align*}

For $I_3$, 
\begin{align*}
    \norm{I_3}_2 \leq  4c_0^{-2}C_X\abs{\frac{(g^{\text{prop}} - g_0^{\text{prop}})(X)}{g^{\text{prop}}(1-g^{\text{prop}})(X)}}\abs{(g\pow{W} - g_0\pow{W})(X)},
\end{align*}
which implies 
\begin{align*}
    \Ex{\prob}{\norm{I_3}_2^2} \leq 16c_0^{-4}C_X^2r^4.
\end{align*}

For $I_4$, since $\abs{g^{\text{prop}} - g_0^{\text{prop}}}\leq 2$,
\begin{align*}
    \norm{I_4}_2 \leq  4c_0^{-2}C_XC_Y\abs{\frac{((g^{\text{prop}} - g_0^{\text{prop}})(X))}{g^{\text{prop}}(1-g^{\text{prop}})(X)}},
\end{align*}
which implies 
\begin{align*}
    \Ex{\prob}{\norm{I_4}_2^2} \leq 16c_0^{-4}(C_XC_Y)^2r^2.
\end{align*}

For $I_5$, 
\begin{align*}
    \norm{I_5}_2 \leq  4c_0^{-2}C_X\abs{\frac{(g^{\text{prop}} - g_0^{\text{prop}})(X)}{g^{\text{prop}}(1-g^{\text{prop}})(X)}}\abs{(g\pow{W} - g_0\pow{W})(X)},
\end{align*}
which implies 
\begin{align*}
    \Ex{\prob}{\norm{I_5}_2^2} \leq 16c_0^{-4}(C_XC_Y)^2r^4.
\end{align*}

For $I_6$, 
\begin{align*}
    \norm{I_6}_2 \leq C_X\abs{\epsilon}.
\end{align*}
When $\Ex{\prob}{\epsilon^2}\leq \sigma^2$, we have
\begin{align*}
    \Ex{\prob}{\norm{I_6}_2^2} \leq C_X^2\sigma^2.
\end{align*}

For $I_7$, 
\begin{align*}
     \norm{I_7}_2 \leq 2c_0^{-2}C_X\abs{(g\pow{W} - g_0\pow{W})(X)},
\end{align*}
which implies 
\begin{align*}
    \Ex{\prob}{\norm{I_7}_2^2} \leq 4c_0^{-4}C_X^2r^2.
\end{align*}

For $I_8$, 
\begin{align*}
     \norm{I_8}_2 \leq c_0^{-2}C_XC_Y\abs{(g^{\text{prop}} - g_0^{\text{prop}})(X)},
\end{align*}
which implies 
\begin{align*}
    \Ex{\prob}{\norm{I_8}_2^2} \leq c_0^{-4}(C_XC_Y)^2r^2.
\end{align*}

For $I_9$, 
\begin{align*}
     \norm{I_9}_2 \leq c_0^{-2}C_X\abs{(g^{\text{prop}} - g_0^{\text{prop}})(X)}\abs{(g\pow{W} - g_0\pow{W})(X)},
\end{align*}
which implies 
\begin{align*}
    \Ex{\prob}{\norm{I_9}_2^2} \leq c_0^{-4}C_X^2r^4.
\end{align*}
By Cauchy-Schwarz inequality, it follows that
\begin{align*}
   \Ex{\prob}{\norm{\score(\theta_\star, g; Z)}_2^2 } \leq 9C_X^2(2C^2 + \sigma^2) + \bigO(r^2).
\end{align*}
Similarly, we have
\begin{align*}
    \norm{\score(\theta_\star, g)}_2^2 \leq 9C_X^2(2C^2 + \sigma^2) + \bigO(r^2).
\end{align*}
Since $H_{\theta\theta}(\theta_\star, g; Z) - H_{\theta\theta}(\theta_\star, g) \preccurlyeq C_X^2\mathbf{I}$, we have
\begin{align*}
    \norm{(H_{\theta\theta}(\theta_\star, g; Z) - H_{\theta\theta}(\theta_\star, g))(\theta - \theta_\star)}_2^2 \leq C_X^4\norm{\theta - \theta_\star}_2^2.
\end{align*}
Thus,
\begin{align*}
    \Ex{\prob}{\norm{\score(\theta, g; Z) - \score(\theta, g)}_2^2} \leq& 3 \Ex{\prob}{\norm{\score(\theta_\star, g; Z)}_2^2} + 3\norm{\score(\theta_\star, g)}_2^2\\
    &+3\Ex{\prob}{ \norm{(H_{\theta\theta}(\theta_\star, g; Z) - H_{\theta\theta}(\theta_\star, g))(\theta - \theta_\star)}_2^2}\\
    =&27C_X^2(2C^2 + \sigma^2) + \bigO(r^2) + 3C_X^4\norm{\theta - \theta_\star}_2^2,
\end{align*}
which implies 
\begin{align}\label{eq: ex4 a3(d)}
    \Kconst_1 = 27C_X^2(2C^2 + \sigma^2) + \bigO(r^2) \text{ and } \kappaconst_1 = 3C_X^4.
\end{align}
(e) Since $Y = W(Y(1)-Y(0)) + Y(0)$ and $g\spow{W}(X) = W(g\spow{1}(X) - g\spow{0}(X)) + g\spow{0}(X)$, $\phi(g; Z)$ can be written as
\begin{align*}
    &\phi(g; Z) = \tau_0(X) + \p{1 - \frac{W(W-g^{\text{prop}}(X))}{g^{\text{prop}}(X)(1-g^{\text{prop}}(X))}}\p{g\spow{1}(X) - g\spow{0}(X) - \tau_0(X)}\\
    &\quad + \frac{W-g^{\text{prop}}(X)}{g^{\text{prop}}(1-g^{\text{prop}})(X)}(Y(0) - g\spow{0}(X))  + \frac{W(W-g^{\text{prop}}(X))}{g^{\text{prop}}(1-g^{\text{prop}})(X)}\p{Y(1) - Y(0) - \tau_0(X)}\\
    &=\tau_0(X) + \frac{g^{\text{prop}}(X) -W}{g^{\text{prop}}(X)}\p{g\spow{1} - g_0\spow{1}}(X)  +\frac{g^{\text{prop}}(X)-W}{1-g^{\text{prop}}(X)}(g\spow{0} - g_0\spow{0})(X)\\
    &\quad + \frac{W-g^{\text{prop}}(X)}{g^{\text{prop}}(X)(1-g^{\text{prop}}(X))}(Y(0) - g_0\spow{0}(X))  + \frac{W}{g^{\text{prop}}(X)}\p{Y(1) - Y(0) - \tau_0(X)}.
\end{align*}
Under \Cref{asm causal}, we have
\begin{align*}
    \Ex{\prob}{\phi(g; Z) \mid X} &= \tau_0(X) + \frac{g^{\text{prop}} - g_0^{\text{prop}}}{g^{\text{prop}}}(g\spow{1} - g_0\spow{1})(X)  + \frac{g^{\text{prop}}-g_0^{\text{prop}}}{1-g^{\text{prop}}}(g\spow{0} - g_0\spow{0})(X).
\end{align*}
Thus, for $\tau = g\spow{1} - g\spow{0}$ and for any $\theta \in \Theta$ and $g, \bar g \in \Gr$ such that $\bar g = t g + (1-t)g_0$ for some $t \in (0,1)$, we have
\begin{align*}
    &\D_g \D_\theta L(\theta_\star, \bar g)[\theta - \theta_\star, g - g_0] \\
    &= -\D_g\Ex{\prob}{(\Ex{\prob}{\phi(g; Z) \mid X} + \ip{\theta_\star, X})\ip{X, \theta - \theta_\star}}[g - g_0]\\
    &= -\Ex{\prob}{\ip{X, \theta - \theta_\star}\D_g\Ex{\prob}{\phi(\bar g; Z) \mid X}[g - g_0]}\\
    &=- \Ex{\prob}{\ip{X, \theta - \theta_\star}\p{\frac{g_0^{\text{prop}}}{(\bar g^{\text{prop}})^2}(\bar g\spow{1} - g_0\spow{1})(g^{\text{prop}} - g_0^{\text{prop}})(X) - \frac{\bar g^{\text{prop}} - g_0^{\text{prop}}}{\bar g^{\text{prop}}}(g\spow{1} - g_0\spow{1})(X)}}\\
    &\quad - \Ex{\prob}{\ip{X, \theta - \theta_\star}\p{\frac{(1-g_0^{\text{prop}})(g^{\text{prop}}-g_0^{\text{prop}})(\bar g\spow{0} - g_0\spow{0})}{(1-\bar g^{\text{prop}})^2} - \frac{(\bar g^{\text{prop}}-g_0^{\text{prop}})(g\spow{0} - g_0\spow{0})}{1-\bar g^{\text{prop}}}}(X)}.
\end{align*}
Since $(a+b)^4 \leq 8a^4 + 8b^4$ for $a, b \in \R$, we have
\begin{align}
    \E\sbr{\frac{g_0^{\text{prop}}(X)^4}{g^{\text{prop}}(X)^4}} &= \E\sbr{\frac{\p{g_0^{\text{prop}}(X) - g^{\text{prop}}(X) + g^{\text{prop}}(X)}^4}{g^{\text{prop}}(X)^4}} \notag\\
    &\leq 8\E\sbr{\frac{\p{g_0^{\text{prop}}(X) - g^{\text{prop}}(X)}^4 + g^{\text{prop}}(X)^4}{g^{\text{prop}}(X)^4}} \leq 8r^4 + 8.\label{e*/e ratio is bounded}
\end{align}  
Similarly, we have
\begin{align*}
    \E\sbr{\frac{(1-g_0^{\text{prop}}(X))^4}{(1-g^{\text{prop}}(X))^4}} \leq 8r^4 + 8.
\end{align*}
It is easy to show that
\begin{align*}
    &\abs{\Ex{\prob}{\ip{X, \theta - \theta_\star}\frac{g_0^{\text{prop}}}{(\bar g^{\text{prop}})^2}(\bar g\spow{1} - g_0\spow{1})(g^{\text{prop}} - g_0^{\text{prop}})(X)}} \\
    &\leq C_X(8r^4+8)^{\frac{1}{4}}r\norm{\theta - \theta_\star}_2\gnorm{g - g_0},
\end{align*}
and
\begin{align*}
    \abs{\Ex{\prob}{\ip{X, \theta - \theta_\star}\frac{\bar g^{\text{prop}} - g_0^{\text{prop}}}{\bar g^{\text{prop}}}(g\spow{1} - g_0\spow{1})(X)}}\leq C_Xr\norm{\theta - \theta_\star}_2\gnorm{g - g_0}.
\end{align*}
Thus,
\begin{align*}
    \abs{\D_g \D_\theta L(\theta_\star, \bar g)[\theta - \theta_\star, g - g_0]} &\leq 2((8r^4+8)^{\frac{1}{4}} + 1)C_Xr\norm{\theta - \theta_\star}_2\gnorm{g - g_0}\\
    &\leq 2(2(r+1) + 1)C_Xr\norm{\theta - \theta_\star}_2\gnorm{g - g_0},
\end{align*}
which implies 
\begin{align}\label{eq: ex4 a3(e)}
    \secsmooth_1 = 2C_X(2r + 3)r.
\end{align}

In addition, $L(\theta, g)$ is Neyman orthogonal at $(\theta_\star, g_0)$ since
\begin{align*}
    \D_g \D_\theta L(\theta_\star, g_0)[\theta - \theta_\star, g - g_0] = 0.
\end{align*}
We have the higher-order derivative such that for any $\theta \in \Theta$ and $g, \bar g \in \Gr$,
\begin{align*}
    \D_g^2 &\D_\theta L(\theta_\star, \bar g)[\theta - \theta_\star, g - g_0, g - g_0] \\
    &= 2\E_{\prob}\sbr{\p{\frac{g_0^{\text{prop}}(g^{\text{prop}} - g_0^{\text{prop}})^2}{(\bar g^{\text{prop}})^3}\p{\bar g\spow{1}\p{X} - g_0\spow{1}\p{X}}}\ip{X, \theta - \theta_\star}}\\
    &\quad -2\E_{\prob}\sbr{\p{\frac{g_0^{\text{prop}}(g^{\text{prop}} - g_0^{\text{prop}})}{(\bar g^{\text{prop}})^2}\p{g\spow{1}\p{X} - g_0\spow{1}\p{X}}}\ip{X, \theta - \theta_\star}}\\
    &\quad -2\E_{\prob}\sbr{\p{\frac{(1-g_0^{\text{prop}})(g^{\text{prop}} - g_0^{\text{prop}})^2}{(1-\bar g^{\text{prop}})^3}\p{\bar g\spow{0}\p{X} - g_0\spow{0}\p{X}}}\ip{X, \theta - \theta_\star}}\\
     &\quad -2\E_{\prob}\sbr{\p{\frac{(1-g_0^{\text{prop}})(g^{\text{prop}} - g_0^{\text{prop}})}{(1-\bar g^{\text{prop}})^2}\p{g\spow{0}\p{X} - g_0\spow{0}\p{X}}}\ip{X, \theta - \theta_\star}}.
\end{align*}
Note that $\bar g^{\text{prop}} = t g^{\text{prop}} + (1-t)g^{\text{prop}}_0$ for some $t \in (0,1)$ by Taylor's theorem. Then
\begin{align}\label{ineq ex4}
    \frac{g^{\text{prop}}(X)}{\bar g^{\text{prop}}(X)} &= \frac{g^{\text{prop}}(X)}{t g^{\text{prop}}(X) + (1-t)g^{\text{prop}}_0(X)}\notag \\
    &= \frac{1}{t + (1-t)(g^{\text{prop}}_0/g^{\text{prop}})(X)} \leq \frac{1}{t + (1-t)g^{\text{prop}}_0(X)} \leq c_0^{-1}.
\end{align}
Thus, 
\begin{align*}
&\abs{\E_{\prob}\sbr{\p{\frac{g_0^{\text{prop}}(g^{\text{prop}} - g_0^{\text{prop}})^2}{(\bar g^{\text{prop}})^3}\p{\bar g\spow{1}\p{X} - g_0\spow{1}\p{X}}}\ip{X, \theta - \theta_\star}}}\\
    &\leq \E_{\prob}\sbr{\frac{(g^{\text{prop}})^4}{(\bar g^{\text{prop}})^4}\frac{(g^{\text{prop}}(X) - g_0^{\text{prop}}(X))^4}{(g^{\text{prop}})^4}}^{1/2}\Ex{\prob}{\p{\frac{\bar g\spow{1}\p{X} - g_0\spow{1}\p{X}}{\bar g^{\text{prop}}}}^2}^{1/2}C_X\norm{\theta - \theta_\star}_2\\
    &\leq c_0^{-2}C_Xr\norm{\theta - \theta_\star}_2\gnorm{g - g_0}^2.
\end{align*}
Similarly, we have
\begin{align*}
&\abs{\E_{\prob}\sbr{\p{\frac{g_0^{\text{prop}}(g^{\text{prop}} - g_0^{\text{prop}})}{(\bar g^{\text{prop}})^2}\p{g\spow{1}\p{X} - g_0\spow{1}\p{X}}}\ip{X, \theta - \theta_\star}}}\\
&\leq \E_{\prob}\sbr{\abs{\frac{(g^{\text{prop}})^2}{(\bar g^{\text{prop}})^2}\frac{(g^{\text{prop}} - g_0^{\text{prop}})}{g^{\text{prop}}}\frac{(g\spow{1} - g_0\spow{1})}{g^{\text{prop}}}}(X)}C_X\norm{\theta - \theta_\star}_2\leq c_0^{-2}C_X \norm{\theta - \theta_\star}_2\gnorm{g - g_0}^2.
\end{align*}
Then we can show that 
\begin{align*}
    \abs{\D_g^2 \D_\theta L(\theta_\star, \bar g)[\theta - \theta_\star, g - g_0, g - g_0]} \leq 4c_0^{-2}C_X(1 +r)\norm{\theta - \theta_\star}_2\gnorm{g - g_0}^2.
\end{align*}
which implies that
\begin{align}\label{eq: ex4 a4 highsmooth}
    \highsmooth_1 = 4c_0^{-2}C_X(1 +r).
\end{align}
\end{proof}

\subsubsection{Proof of Lemma \ref{lemma: example CRR}}\label{sec: proof ex5}

\begin{proof}
Define
    \begin{align*}
        \mu_g\spow{s}(z) = g\spow{s}(x) + \frac{\ind(w = s)}{sg^{\text{prop}}(x) + (1-s) (1-g^{\text{prop}}(x))}(y - g\spow{s}(x)),
    \end{align*}
    where $\ind(\cdot)$ denotes the indicator function, and the log-linear predictor $ p_\theta(x) = e^{\ip{\theta, x}}/(1 + e^{\ip{\theta, x}})$. Under \Cref{asm causal}, 
    \begin{align}\label{eq ex5}
        \E_{\prob}[&\mu_g\spow{s}(Z) \mid X] = \Ex{\prob}{g\spow{s}(X) + \frac{\ind(W = s)}{sg^{\text{prop}}(X) + (1-s) (1-g^{\text{prop}}(X))}(Y(s) - g\spow{s}(X)) \mid X}\notag\\
        &=g\spow{s}(X) + \frac{sg_0^{\text{prop}}(X) + (1-s) (1-g_0^{\text{prop}}(X))}{sg^{\text{prop}}(X) + (1-s) (1-g^{\text{prop}}(X))}(g_0\spow{s}(X) - g\spow{s}(X)) \notag\\
        &=g_0\spow{s}(X) + \frac{(2s-1) (g^{\text{prop}}(X)-g_0^{\text{prop}}(X))}{sg^{\text{prop}}(X) + (1-s) (1-g^{\text{prop}}(X))}(g\spow{s}(X) - g_0\spow{s}(X))=: f\spow{s}(g; X).
    \end{align}
We consider the following loss:
    \begin{align*}
        \loss(\theta,  g; z) = - \sbr{\mu_g\spow{1}(z) \log p_\theta(x) + \mu_g\spow{0}(z) \log(1 - p_\theta(x))}.
    \end{align*}
with the corresponding risk function defined as
\begin{align*}
    \ploss(\theta, g) = -\Ex{\prob}{\mu_g\spow{1}(Z) \log p_\theta(x) + \mu_g\spow{0}(Z) \log(1 - p_\theta(X))}.
\end{align*}
By \eqref{eq ex5}, we have $\E_{\prob}[\mu_{g_0}\spow{s}(Z) \mid X] = g_0\spow{s}(X)$. Thus, the target is the minimizer of the following squared loss:
\begin{align}\label{eq: ex5 target}
    \theta_\star = \argmin_{\theta \in \R^d}-\Ex{\prob}{g_0\spow{1}(X) \log p_\theta(x) +g_0\spow{0}(X) \log(1 - p_\theta(X))}.
\end{align}
Since $\nabla_\theta p_\theta(x) = p_\theta(x)(1-p_\theta(x))x$, we have
\begin{align*}
    \nabla_\theta \log p_\theta(x) &= \frac{\nabla_\theta p_\theta(x)}{p_\theta(x)}  = (1-p_\theta(x))x,\\
    \nabla_\theta \log (1-p_\theta(x)) &= -\frac{\nabla_\theta p_\theta(x)}{1-p_\theta(x)} = -p_\theta(x) x.
\end{align*}
Differentiating $\loss(\theta, g; z)$ with respect to $\theta$, we obtain the gradient and Hessian \wrt $\theta$ as
\begin{align*}
    \score(\theta, g; z) &= - \sbr{\mu_g\spow{1}(z) (1-p_\theta(x))x - \mu_g\spow{0}(z) p_\theta(x) x}, \\
    H_{\theta\theta}(\theta, g; z) &= (\mu_g\spow{1}(z) + \mu_g\spow{0}(z))p_\theta(x)(1-p_\theta(x)) xx^\top.
\end{align*}
The expected gradient and expected Hessian are then obtained as
\begin{align*}
    \score(\theta, g) &= -\Ex{\prob}{f\spow{1}(g; X)(1-p_\theta(X))X - f\spow{0}(g; X) p_\theta(X) X}, \\
    H_{\theta\theta}(\theta, g) &= \Ex{\prob}{(f\spow{1}(g; X) + f\spow{0}(g; X))p_\theta(X)(1-p_\theta(X)) XX^\top}.
\end{align*}
We consider the nuisance neighborhood such that for $g \in \Gr$,
\begin{align*}
   \gnorm{g - g_0} := \max\br{\Ex{\prob}{\p{\frac{(g\pow{w} - g_0\pow{w})(X)}{g\pow{w}(1-g\pow{w})(X)}}^4}^{\frac{1}{4}}, \Ex{\prob}{\p{\frac{(g^{\text{prop}} - g_0^{\text{prop}})(X)}{g^{\text{prop}}(1-g^{\text{prop}})(X)}}^4}^{\frac{1}{4}}} \leq r.
\end{align*}
We assume that $\delta \leq g_0\spow{0}(X) + g_0\spow{1}(X) \leq \delta^{-1}$ for $\delta >0$. In addition, we assume that for $g \in \Gr$, $f\spow{1}(g; X) + f\spow{0}(g; X) \geq \delta$ \as. Note that
\begin{align*}
    \Ex{\prob}{{\abs{f\spow{1}(g; X) + f\spow{0}(g; X)}}} &\leq \delta + \sum_{s=\br{0,1}}\Ex{\prob}{\abs{\frac{g^{\text{prop}}(X)-g_0^{\text{prop}}(X)}{g^{\text{prop}}(1-g^{\text{prop}}(X))}(g\spow{s}(X) - g_0\spow{s}(X))}}\\
    &\leq \delta + 2r^2.
\end{align*}
We now verify that the loss function $\loss$ satisfies \Cref{assumption of population risk}.

(a) We assume that $g\pow{w}: \R^d \mapsto \R, w = 0,1,$ and $g^{\text{prop}}: \R^d \mapsto (0,1)$ are continuous, thus \Cref{assumption of population risk}(a) is satisfied.

(b) Since $\theta_\star$ is a global minimizer of \eqref{eq: ex5 target}, we have
\begin{align}\label{eq: ex5 a3(b)}
    \score(\theta_\star, g_0) = 0.
\end{align}

(c) We assume that $\Theta$ is bounded such that $C \leq p_\theta(X) \leq 1-C$ \as for some $C>0$. When $\lambda_{\min}(\Ex{\prob}{XX^\top}) \geq \lambda_0 > 0$ and $\norm{X}_2 \leq C_X$ \as, we have
\begin{align}\label{eq: ex5 a3(c)}
    C^2\delta\lambda_0\mathbf{I} \preccurlyeq H_{\theta\theta}(\theta, g) \preccurlyeq C_X^2(1 + 2\delta^{-1}r^2)\mathbf{I} \implies \sconv = C^2\delta\lambda_0 \text{ and } \smooth = C_X^2(1 + 2\delta^{-1}r^2).
\end{align}

(d) Consider the Taylor expansion around $\theta_\star$, we have
\begin{align*}
    \score(\theta, g; Z) - \score(\theta, g) = \score(\theta_\star, g; Z) - \score(\theta_\star, g) + (H_{\theta\theta}(\theta_\star, g; Z) - H_{\theta\theta}(\theta_\star, g))(\theta - \theta_\star).
\end{align*}
Note that
\begin{align*}
        \E_{\prob}\sbr{\norm{\score(\theta_\star, g; Z)-S_4(\theta_\star, g_4)}_2^2} &\leq \E_{\prob}\sbr{\norm{\score(\theta_\star, g; Z)}_2^2}\\
        &\leq C_X^2\E\sbr{\p{\abs{\mu_g\spow{1}(Z)}+ \abs{\mu_g\spow{0}(Z)}}^2}\\
        &\leq 2C_X^2\E\sbr{(\mu_g\spow{1}(Z))^2 + (\mu_g\spow{1}(Z))^2}.
\end{align*}
For $s = 1$, when $Y(1) - g_0\spow{1}(X) \leq C_Y$ \as for $C_Y > 0$, we have
\begin{align*}
    \E_{\prob}[&\mu_g\spow{1}(Z)^2] = \E_{\prob}\sbr{\p{g_0\spow{1}(X) + \frac{g^{\text{prop}}(X) -W}{g^{\text{prop}}(X)}(g\spow{1} - g_0\spow{1})(X) +  \frac{W}{g^{\text{prop}}(X)}\p{Y - g_0\spow{1}(X)}}^2}\\
    &\leq 3 \delta^{-2} + 3\E_{\prob}\sbr{\frac{(g^{\text{prop}}(X) - W)^2}{(g^{\text{prop}}(X))^2}\p{(g\spow{1} - g_0\spow{1})(X)}^2}+ 3C_Y^2\E_{\prob}\sbr{\frac{g_0^{\text{prop}}(X)}{(g^{\text{prop}}(X))^2}}\\
    &\leq 3 \delta^{-2} + 12\E_{\prob}\sbr{\p{\frac{(g\spow{1} - g_0\spow{1})(X)}{g^{\text{prop}}(X)}}^2}  + 3c_0^{-1}C_Y^2\E_{\prob}\sbr{\frac{(g_0^{\text{prop}}(X))^2}{(g^{\text{prop}}(X))^2}}\\
    &\leq 3\delta^{-2} + 12r^2 + 3c_0^{-1}C_Y^2(8r^4 + 8)^{1/2} \leq 3(\delta^{-2} + 4c_0^{-1}C_Y^2) + \bigO(r^2).
\end{align*}
Thus,
\begin{align*}
    \E_{\prob}\sbr{\norm{\score(\theta_\star, g; Z)-S_4(\theta_\star, g_4)}_2^2} \leq 12C_X^2(\delta^{-2} + 4c_0^{-1}C_Y^2) + \bigO(r^2).
\end{align*}
Since 
\begin{align*}
    H_{\theta\theta}(\theta, g; Z) - H_{\theta\theta}(\theta, g) \preccurlyeq C_X^2(\mu_g\spow{1}(Z) + \mu_g\spow{0}(Z) + 1 + 2\delta^{-1}r^2)\mathbf{I},
\end{align*}
we have
\begin{align*}
    \norm{(H_{\theta\theta}(\theta_\star, g; Z) - H_{\theta\theta}(\theta_\star, g))}_2^2\leq 3C_X^4((\mu_g\spow{1}(Z))^2 + (\mu_g\spow{0}(Z))^2 + (1 + 2\delta^{-1}r^2)^2).
\end{align*}
Similarly, we can show that
\begin{align*}
    \Ex{\prob}{ \norm{(H_{\theta\theta}(\theta_\star, g; Z) - H_{\theta\theta}(\theta_\star, g))}_2^2} \leq 3C_X^4(1+6(\delta^{-2} + 4c_0^{-1}C_Y^2)) + \bigO(r^2).
\end{align*}
It follows that
\begin{align}
    \Kconst_1 = 24C_X^2(\delta^{-2} + 4c_0^{-1}C_Y^2) + \bigO(r^2) \text{ and } \kappaconst_1 = 6C_X^4(1+6(\delta^{-2} + 4c_0^{-1}C_Y^2)) + \bigO(r^2).
\end{align}

(e) For $s = 1$, we have
\begin{align*}
    \D_g f\spow{1}(\bar g; X)[g - g_0] =& \frac{g_0^{\text{prop}}(X)}{(\bar g^{\text{prop}}(X))^2}(\bar g\spow{1}(X) - g_0\spow{1}(X))(g^{\text{prop}}(X)-g_0^{\text{prop}}(X))\\
    &+ \frac{\bar g^{\text{prop}}(X) - g_0^{\text{prop}}(X)}{\bar g^{\text{prop}}(X)}(g\spow{1}(X) - g_0\spow{1}(X)).
\end{align*}
Similarly, for $s = 0$,
\begin{align*}
    \D_g f\spow{0}(\bar g; X)[g - g_0] =& \frac{1-g_0^{\text{prop}}(X)}{(1-\bar g^{\text{prop}}(X))^2}(\bar g\spow{0}(X) - g_0\spow{0}(X))(g^{\text{prop}}(X)-g_0^{\text{prop}}(X))\\
    &- \frac{\bar g^{\text{prop}}(X) - g_0^{\text{prop}}(X)}{1-\bar g^{\text{prop}}(X)}(g\spow{0}(X) - g_0\spow{0}(X)).
\end{align*}
For any $\bar g, g \in \G_{r}(g_0)$ such that $\bar g = t g + (1-t)g_0$ for some $t \in (0,1)$, we have
\begin{align*}
    &\D_g \D_\theta L(\theta_\star, \bar g)[\theta - \theta_\star, g - g_0] \\
    &=  -\E_{\prob}\sbr{\D_g(f\spow{1}(\bar g; X)(1-p_\theta(X)) - f\spow{0}(\bar g; X) p_\theta(X))[g - g_0]\ip{X, \theta - \theta_\star}}\\
    &= -\E_{\prob}\sbr{\frac{g_0^{\text{prop}}(1-p_\theta)}{(\bar g^{\text{prop}})^2}(\bar g\spow{1} - g_0\spow{1})(g^{\text{prop}}-g_0^{\text{prop}})(X)\ip{X,\theta - \theta_\star}}\\
    &\quad -\E_{\prob}\sbr{\frac{(1-g_0^{\text{prop}})p_\theta}{(1-\bar g^{\text{prop}})^2}(\bar g\spow{0} - g_0\spow{0})(g^{\text{prop}}-g_0^{\text{prop}})(X)\ip{X,\theta - \theta_\star}}\\
    &\quad - \E_{\prob}\sbr{\frac{(\bar g^{\text{prop}} - g_0^{\text{prop}})(1-p_\theta)}{\bar g^{\text{prop}}}(g\spow{1} - g_0\spow{1})(X)\ip{X,\theta - \theta_\star}}\\
    &\quad + \E_{\prob}\sbr{\frac{(\bar g^{\text{prop}} - g_0^{\text{prop}})p_\theta}{1-\bar g^{\text{prop}}}(g\spow{0} - g_0\spow{0})(X)\ip{X,\theta - \theta_\star}}.
\end{align*}
Note that by \eqref{ineq ex4},
\begin{align*}
&\abs{\E_{\prob}\sbr{\frac{g_0^{\text{prop}}(1-p_\theta)}{(\bar g^{\text{prop}})^2}(\bar g\spow{1} - g_0\spow{1})(g^{\text{prop}}-g_0^{\text{prop}})(X)\ip{X,\theta - \theta_\star}}}\\
    &\leq \E_{\prob}\sbr{\abs{\frac{g^{\text{prop}}(X)}{\bar g^{\text{prop}}(X)}\frac{(\bar g\spow{1} - g_0\spow{1})(X)}{g^{\text{prop}}(X)}\frac{(g^{\text{prop}}-g_0^{\text{prop}})(X)}{g^{\text{prop}}(X)}}}C_X\norm{\theta - \theta_\star}_2\\
    &\leq c_0^{-1}C_Xr\norm{\theta - \theta_\star}_2\gnorm{g - g_0}.
\end{align*}
In addition,
\begin{align*}
    &\abs{\E_{\prob}\sbr{\frac{(\bar g^{\text{prop}} - g_0^{\text{prop}})(1-p_\theta)}{\bar g^{\text{prop}}}(g\spow{1} - g_0\spow{1})(X)\ip{X,\theta - \theta_\star}}}\\
    &\leq \E_{\prob}\sbr{\abs{\frac{(\bar g^{\text{prop}} - g_0^{\text{prop}})}{\bar g^{\text{prop}}}(g\spow{1} - g_0\spow{1})(X)}}C_X\norm{\theta - \theta_\star}_2 \leq C_Xr\norm{\theta - \theta_\star}_2\gnorm{g - g_0}.
\end{align*}
Thus, it is easy to show that 
\begin{align*}
    \abs{\D_g \D_\theta L(\theta_\star, \bar g)[\theta - \theta_\star, g - g_0]} \leq 2(c_0^{-1}+ 1)C_Xr\norm{\theta - \theta_\star}_2\gnorm{g - g_0},
\end{align*}
which implies
\begin{align}
    \secsmooth_1 = 2(c_0^{-1}+ 1)C_Xr.
\end{align}
In addition, $L(\theta, g)$ is Neyman orthogonal at $(\theta_\star, g_0)$ since
\begin{align*}
    \D_g \D_\theta L(\theta_\star, g_0)[\theta - \theta_\star, g - g_0] = 0.
\end{align*}

Now we compute the the higher-order derivative. For $s = 1$, we have
\begin{align*}
    \D_g^2 f\spow{1}(\bar g; X)[g - g_0, g - g_0] =& -\frac{2g_0^{\text{prop}}(X)}{(\bar g^{\text{prop}}(X))^3}(\bar g\spow{1}(X) - g_0\spow{1}(X))(g^{\text{prop}}(X)-g_0^{\text{prop}}(X))^2\\
    &+\frac{2g_0^{\text{prop}}(X)}{(\bar g^{\text{prop}}(X))^2}( g\spow{1}(X) - g_0\spow{1}(X))(g^{\text{prop}}(X)-g_0^{\text{prop}}(X)).
\end{align*}
Similarly, for $s = 0$,
\begin{align*}
    \D_g^2 f\spow{0}(\bar g; X)[g - g_0, g - g_0] =& 2\frac{1-g_0^{\text{prop}}(X)}{(1-\bar g^{\text{prop}}(X))^3}(\bar g\spow{0}(X) - g_0\spow{0}(X))(g^{\text{prop}}(X)-g_0^{\text{prop}}(X))^2\\
    &+ 2\frac{1-g_0^{\text{prop}}(X)}{(1-\bar g^{\text{prop}}(X))^2}( g\spow{0}(X) - g_0\spow{0}(X))(g^{\text{prop}}(X)-g_0^{\text{prop}}(X)).
\end{align*}
Then for any $\theta \in \Theta$ and $g, \bar g \in \Gr$ such that $\bar g = tg + (1-t)g_0$ for some $t \in (0,1)$,
\begin{align*}
    \D_g^2 &\D_\theta L(\theta_\star, \bar g)[\theta - \theta_\star, g - g_0, g - g_0] \\
    &= -\E_{\prob}\sbr{\D^2_g(f\spow{1}(\bar g; X)(1-p_\theta(X)) - f\spow{0}(\bar g; X) p_\theta(X))[g - g_0, g - g_0]\ip{X, \theta - \theta_\star}}\\
    &= 2\E_{\prob}\sbr{\frac{g_0^{\text{prop}}(X)}{(\bar g^{\text{prop}}(X))^3}(\bar g\spow{1}(X) - g_0\spow{1}(X))(g^{\text{prop}}(X)-g_0^{\text{prop}}(X))^2\ip{X,\theta - \theta_\star}}\\
    &\quad - 2\Ex{\prob}{\frac{1-g_0^{\text{prop}}(X)}{(1-\bar g^{\text{prop}}(X))^3}(\bar g\spow{0}(X) - g_0\spow{0}(X))(g^{\text{prop}}(X)-g_0^{\text{prop}}(X))^2\ip{X,\theta - \theta_\star}}\\
    &\quad - 2\E_{\prob}\sbr{\frac{g_0^{\text{prop}}(X)}{(\bar g^{\text{prop}}(X))^2}( g\spow{1}(X) - g_0\spow{1}(X))(g^{\text{prop}}(X)-g_0^{\text{prop}}(X))\ip{X,\theta - \theta_\star}}\\
    &\quad - 2\Ex{\prob}{\frac{1-g_0^{\text{prop}}(X)}{(1-\bar g^{\text{prop}}(X))^2}( g\spow{0}(X) - g_0\spow{0}(X))(g^{\text{prop}}(X)-g_0^{\text{prop}}(X))\ip{X,\theta - \theta_\star}}.
\end{align*}
By \eqref{ineq ex4}, we have
\begin{align*}
&\abs{\E_{\prob}\sbr{\frac{g_0^{\text{prop}}(X)}{(\bar g^{\text{prop}}(X))^3}(\bar g\spow{1}(X) - g_0\spow{1}(X))(g^{\text{prop}}(X)-g_0^{\text{prop}}(X))^2\ip{X,\theta - \theta_\star}}}\\
    &\leq \E_{\prob}\sbr{\abs{\frac{(g^{\text{prop}}(X))^2}{(\bar g^{\text{prop}}(X))^2}\frac{\bar g\spow{1}(X) - g_0\spow{1}(X)}{\bar g^{\text{prop}}(X)}\p{\frac{g^{\text{prop}}(X)-g_0^{\text{prop}}(X)}{g^{\text{prop}}(X)}}^2}}C_X\norm{\theta - \theta_\star}_2\\
    &\leq c_0^{-2}C_Xr\norm{\theta - \theta_\star}_2\gnorm{g - g_0}^2.
\end{align*}
In addition, 
\begin{align*}
&\abs{\E_{\prob}\sbr{\frac{g_0^{\text{prop}}(X)}{(\bar g^{\text{prop}}(X))^2}( g\spow{1}(X) - g_0\spow{1}(X))(g^{\text{prop}}(X)-g_0^{\text{prop}}(X))\ip{X,\theta - \theta_\star}}}\\
    &\leq \E_{\prob}\sbr{\frac{(g^{\text{prop}}(X))^2}{(\bar g^{\text{prop}}(X))^2}\frac{g\spow{1}(X) - g_0\spow{1}(X)}{g^{\text{prop}}(X)}\frac{g^{\text{prop}}(X)-g_0^{\text{prop}}(X)}{g^{\text{prop}}(X)}} C_X\norm{\theta - \theta_\star}_2\\
    &\leq c_0^{-2}C_X\norm{\theta - \theta_\star}_2\gnorm{g - g_0}^2.
\end{align*}
Together we have
\begin{align*}
    \abs{\D_g^2 \D_\theta L(\theta_\star, \bar g)[\theta - \theta_\star, g - g_0, g - g_0]}\leq 4c_0^{-2}C_X(1 + r)\norm{\theta - \theta_\star}_2\gnorm{g - g_0}^2,
\end{align*}
which implies 
\begin{align}
    \highsmooth_1 = 4c_0^{-2}C_X(1 + r).
\end{align}
\end{proof}

\clearpage

\section{Convergence Proofs for Stochastic Gradient}\label{sec:a:sgd}

This section is dedicated to demonstrate the SGD convergence in \Cref{theorem convergence rate baseline SGD} from \Cref{sec:optimization} of the main text using \Cref{assumption of population risk} and \Cref{assumption orthogonality main}. We first give an overview of the problem settings and the expected results with the proof outline in \Cref{sec sgd overview}. We then provide all the technical lemmas needed for \Cref{theorem convergence rate baseline SGD} in \Cref{sec: sgd tech lem}, and finally prove our first main result in \Cref{sec: sgd thm proof}.

\subsection{Overview}\label{sec sgd overview}
In this section, we demonstrate the convergence of SGD for a risk minimization problem with nuisance:
\begin{align*}
    \theta_\star = \argmin_{\theta \in \Theta} L(\theta, g_0),
\end{align*}
where $g_0 \in \G$ is the true nuisance, $L(\theta, g) = \Ex{Z\sim \prob}{\ell(\theta,g; Z)}$, and $\ell$ is a prespecified loss function. We consider the stochastic gradient method for learning $\theta_\star$ when $g_0$ is unknown but an estimate $\hat g$ is accessible. Define $\mc{D}_n = (Z_1, \ldots, Z_n)$, sampled from the product measure $\prob^n$. Recall the SGD $\sgd^{(n)}$ defined as 
\begin{align}
    \sgd^{(n)} = \sgd^{(n-1)} - \eta \score(\sgd^{(n-1)}, \hat g; Z_n).
\end{align}
Throughout the section, we take the following notations for simplicity:
\begin{align}
    \delta^{(n)} &= \sgd^{(n)} - \theta_\star,\\
    S^{(n)} &= \score(\sgd^{(n-1)}, \hat g; Z_n),\\
    v^{(n)} &= S^{(n)} - \score(\sgd^{(n-1)}, \hat g),
\end{align}
where $\score(\theta,g;z) = \nabla_\theta \ell(\theta,g;z)$ is the gradient and $\score(\theta,g) = \Ex{Z \sim \prob}{\score(\theta,g;Z)}$ is the population gradient. 
We are interested in the mean squared error using an estimated nuisance $\hat g$, and our results show that for non-Neyman orthogonal loss $L$, the error $\delta^{(n)}$ satisfies
        \begin{align}\label{sgd overview non-ortho}
             \E_{\mc{D}_{n} \sim \prob^n}[\Vert\delta^{(n)}\Vert_2^2]  \lesssim \p{1-\frac{\sconv \eta}{2}}^n\Vert\delta^{(0)}\Vert_2^2+  \Vert\hat g - g_0\Vert_\G^2 + \eta,
        \end{align}
where the nuisance estimator $\hat g$ would lead to a bias of order $\bigO(\Vert\hat g - g_0\Vert_\G^2)$ for the SGD convergence. If $L$ is Neyman orthogonal, this bias introduced by the nuisance estimation would be further removed, resulting in the following convergence
        \begin{align}\label{sgd overview ortho}
             \E_{\mc{D}_{n} \sim\prob^n}[\Vert\delta^{(n)}\Vert_2^2]   \lesssim& \p{1-\frac{\sconv\eta}{2}}^n\Vert\delta^{(0)}\Vert_2^2 +  \Vert\hat g - g_0\Vert_\G^4 + \eta.
        \end{align}

\myparagraph{Proof Outline}
The proofs for both results \eqref{sgd overview non-ortho} and \eqref{sgd overview ortho} proceed through the following four steps:
\begin{enumerate}
    \item Upper bound the excess risk  $L(\sgd^{(n)}, \hat g) - L(\theta_\star, \hat g)$ in terms of the SGD improvement.
    \item Lower bound $L(\sgd^{(n)}, \hat g) - L(\theta_\star, \hat g)$ using strong convexity and Neyman orthogonality.
    \item Derive a recursive formula of $\E_{\calD_n\sim\prob^n}[\norm{\delta^{(n)}}_2^2]$ from these bounds.
    \item Perform the recursion and obtain the final result.
\end{enumerate}
Follow these steps above, we provide technical lemma in \Cref{sec: sgd tech lem}, and then prove our first main result \Cref{theorem convergence rate baseline SGD} in \Cref{sec: sgd thm proof}.

\subsection{Technical Lemma}\label{sec: sgd tech lem}
\begin{lemma}[One-step improvement for SGD]\label{lemma one-step improvement}
    Suppose that \Cref{assumption of population risk} holds. 
  If $\eta < 1/\smooth$,  $\sgd^{(n)} \in \Theta$, and $\hat g \in \G_{r}(g_0)$, it holds that
    \begin{align*}
    2\eta(L(\sgd^{(n)}, \hat g) - L(\theta_\star, \hat g)) \leq& \p{ 1 - \sconv\eta} \norm{\delta^{(n-1)}}_2^2 - \norm{\delta^{(n)}}_2^2 - 2\eta\ip{v^{(n)}, \delta^{(n-1)}}+ \frac{\eta^2}{1 - M\eta} \norm{v^{(n)}}_2^2.
    \end{align*}
\end{lemma}
\begin{proof}
We first define the $\eta^{-1}$-strongly convex function $f_n$ as
\begin{align}\label{eta strongly conv function}
    f_n\p{u} = \ip{S^{(n)}, u - \sgd^{(n-1)}} + \frac{1}{2\eta}\norm{u - \sgd^{(n-1)}}_2^2.
\end{align}
Note that 
\begin{align*}
    \argmin_{u \in \R^d} f_n\p{u} = \argmin_{u \in \R^d} \norm{u - (\sgd^{(n-1)} - \eta  S^{(n)})}_2^2,
\end{align*}
which implies that $\sgd^{(n)} = \sgd^{(n-1)} - \eta  S^{(n)}$ is the global minimizer of \eqref{eta strongly conv function} and $\nabla_{\theta} f_n(\sgd^{(n)}) = 0$.  Then
\begin{align}
    f_n\p{\theta_\star} &\geq f_n(\sgd^{(n)}) + \ip{\nabla_{\theta} f_n(\sgd^{(n)}), -\delta^{(n)}} + \frac{1}{2\eta} \Vert \delta^{(n)} \Vert_2^2 \notag\\
    &=f_n(\sgd^{(n)}) + \frac{1}{2\eta} \norm{\delta^{(n)}}_2^2.\label{ineq for ft}
\end{align}
Since $L\p{\cdot,\hat g}$ is $\sconv$-strongly convex and $f_n\p{\theta_\star} = \ip{S^{(n)}, -\delta^{(n-1)}} + (2\eta)^{-1}\norm{\delta^{(n-1)}}_2^2$, we have that
\begin{align*}
    L(\theta_\star, \hat g) &\geq L(\sgd^{(n-1)}, \hat g) + \ip{\score(\sgd^{(n-1)},\hat g), -\delta^{(n-1)}} + \frac{\sconv}{2} \norm{\delta^{(n-1)}}_2^2\\
    &=L(\sgd^{(n-1)}, \hat g) + \ip{S^{(n)}, -\delta^{(n-1)}} + \ip{v^{(n)}, \delta^{(n-1)}} + \frac{\sconv}{2} \norm{\delta^{(n-1)}}_2^2\\
    &=L(\sgd^{(n-1)}, \hat g) + f_n\p{\theta_\star} + \ip{v^{(n)}, \delta^{(n-1)}} + \p{\frac{\sconv}{2} - \frac{1}{2\eta}} \norm{\delta^{(n-1)}}_2^2.
\end{align*}
Together with \eqref{ineq for ft}, it follows that
\begin{align}\label{ineq for L + ft}
    L(\sgd^{(n-1)}, \hat g) + f_n(\sgd^{(n)}) \leq& L(\theta_\star, \hat g) - \ip{v^{(n)}, \delta^{(n-1)}} \notag \\
    &+ \p{\frac{1}{2\eta} - \frac{\sconv}{2}} \norm{\delta^{(n-1)}}_2^2 - \frac{1}{2\eta} \norm{\delta^{(n)}}_2^2 . 
\end{align}

Since $L\p{\cdot, \hat g}$ is $\smooth$-smooth  and $f_n\p{\sgd^{(n)}} = \ip{S^{(n)}, \sgd^{(n)} - \sgd^{(n-1)}} + (2\eta)^{-1}\norm{\sgd^{(n)} - \sgd^{(n-1)}}_2^2$, we have that  
    \begin{align*}
        L(\sgd^{(n)}, \hat g) &\leq L(\sgd^{(n-1)}, \hat g) + \ip{S(\sgd^{(n-1)}, \hat g), \sgd^{(n)} - \sgd^{(n-1)}} + \frac{\smooth}{2} \Vert \sgd^{(n)} - \sgd^{(n-1)}\Vert_2^2\notag \\
        &=L(\sgd^{(n-1)}, \hat g) + \ip{S^{(n)}, \sgd^{(n)} - \sgd^{(n-1)}}  + \frac{\smooth}{2} \Vert \sgd^{(n)} - \sgd^{(n-1)}\Vert_2^2 
        - \ip{v^{(n)}, \sgd^{(n)} - \sgd^{(n-1)}}\\
        &=L(\sgd^{(n-1)}, \hat g) + f_n(\sgd^{(n)}) + \p{\frac{\smooth}{2} - \frac{1}{2\eta}}\norm{\sgd^{(n)} - \sgd^{(n-1)}}_2^2 
        - \ip{v^{(n)}, \sgd^{(n)} - \sgd^{(n-1)}}.
    \end{align*}
By \eqref{ineq for L + ft}, it follows that
\begin{align*}
    L(\sgd^{(n)}, \hat g) &\leq L(\theta_\star, \hat g) - \ip{v^{(n)}, \delta^{(n-1)}}  + \p{\frac{1}{2\eta} - \frac{\sconv}{2}} \norm{\delta^{(n-1)}}_2^2- \frac{1}{2\eta} \norm{\delta^{(n)}}_2^2\\
        &\quad + \p{\frac{\smooth}{2} - \frac{1}{2\eta}}\norm{\sgd^{(n)} - \sgd^{(n-1)}}_2^2 
        - \ip{v^{(n)}, \sgd^{(n)} - \sgd^{(n-1)}},
\end{align*}
which implies that 
\begin{align}\label{ineq L thetahat - L theta}
    L(\sgd^{(n)}, \hat g) - L(\theta_\star, \hat g) &\leq  \p{ \frac{1}{2\eta} - \frac{\sconv}{2}} \norm{\delta^{(n-1)}}_2^2 - \frac{1}{2\eta} \norm{\delta^{(n)}}_2^2 - \ip{v^{(n)}, \delta^{(n-1)}}\notag\\
    & \quad + \p{\frac{\smooth}{2} - \frac{1}{2\eta}}\norm{\sgd^{(n)} - \sgd^{(n-1)}}_2^2 
        - \ip{v^{(n)}, \sgd^{(n)} - \sgd^{(n-1)}}.
\end{align}
For any $\omega > 0$, by Cauchy-Schwarz inequality and Young's inequality, we have
\begin{align*}
    -\ip{v^{(n)}, \sgd^{(n)} - \sgd^{(n-1)}} \leq \frac{\omega}{2} \norm{v^{(n)}}_2^2 + \frac{1}{2\omega} \norm{\sgd^{(n)} - \sgd^{(n-1)}}_2^2.
\end{align*}
Take this into \eqref{ineq L thetahat - L theta} and we have
\begin{align*}
    L(\sgd^{(n)}, \hat g) - L(\theta_\star, \hat g) &\leq  \p{ \frac{1}{2\eta} - \frac{\sconv}{2}} \norm{\delta^{(n-1)}}_2^2 - \frac{1}{2\eta} \norm{\delta^{(n)}}_2^2 - \ip{v^{(n)}, \delta^{(n-1)}} \notag\\
    & \quad + \p{\frac{\smooth}{2} - \frac{1}{2\eta} + \frac{1}{2\omega}}\norm{\sgd^{(n)} - \sgd^{(n-1)}}_2^2 
        + \frac{\omega}{2} \norm{v^{(n)}}_2^2.
\end{align*}
When $\eta < 1/M$, set $\frac{\smooth}{2} - \frac{1}{2\eta} + \frac{1}{2\omega} = 0$, i.e., set $\omega = 1/(\eta^{-1} - M)$. It follows that 
\begin{align*}
    L(\sgd^{(n)}, \hat g) - L(\theta_\star, \hat g)  \leq& \p{ \frac{1}{2\eta} - \frac{\sconv}{2}} \norm{\delta^{(n-1)}}_2^2 - \frac{1}{2\eta} \norm{\delta^{(n)}}_2^2 - \ip{v^{(n)}, \delta^{(n-1)}} \\
    &+ \frac{\eta}{2(1 - M\eta)} \norm{v^{(n)}}_2^2.
\end{align*}
We complete the proof by multiplying both sides of the inequality by $2\eta$.
\end{proof}

\begin{lemma}\label{lemma: sgd recursive formula}
Suppose that \Cref{assumption of population risk} holds. If $\eta < 1/\smooth$, $\sgd^{(n)} \in \Theta$, and $\hat g \in \G_{r}(g_0)$, it holds that
\begin{align*}
        \norm{\delta^{(n)}}_2^2 \leq& \p{1-\sconv\eta}\norm{\delta^{(n-1)}}_2^2+\frac{\secsmooth_1^2 \eta}{\sconv}\gnorm{\hat g - g_0}^2 
    + 2\eta\ip{v^{(n)}, \delta^{(n-1)}}
    + \frac{\eta^2\norm{v^{(n)}}_2^2}{1 - \smooth\eta}. 
\end{align*}      
\end{lemma}

\begin{proof}
Under \Cref{assumption of population risk},
\begin{align}\label{eq: lower bound under secsmooth}
    D_\theta L(\theta_\star, \hat g)[\delta^{(n)}] &=  D_\theta L(\theta_\star, g_0)[\delta^{(n)}] + D_g D_\theta L(\theta_\star, \bar g)[\delta^{(n)}, \hat g - g_0]\notag\\
    &= D_g D_\theta L(\theta_\star, \bar g)[\delta^{(n)}, \hat g - g_0] \geq -\secsmooth_1 \norm{\delta^{(n)}}_2\gnorm{\hat g - g_0}.  
\end{align}
    Since $L(\cdot, \hat g)$ is $\sconv$-strongly convex,
\begin{align*}
    L(\sgd^{(n)}, \hat g) - L(\theta_\star, \hat g) &\geq \ip{S(\theta_\star, \hat g), \delta^{(n)}} + \frac{\sconv}{2}\norm{\delta^{(n)}}_2^2 \notag\\
    &=D_\theta L(\theta_\star, \hat g)[\delta^{(n)}]+ \frac{\sconv}{2}\norm{\delta^{(n)}}_2^2. 
\end{align*}
By \eqref{eq: lower bound under secsmooth}, it follows that
\begin{align*}
    L(\sgd^{(n)}, \hat g) - L(\theta_\star, \hat g) \geq -\secsmooth_1 \norm{\delta^{(n)}}_2\gnorm{\hat g - g_0} + \frac{\sconv}{2}\norm{\delta^{(n)}}_2^2.
\end{align*}
Together with Lemma \ref{lemma one-step improvement}, we have 
\begin{align*}
    2\eta\p{-\secsmooth_1 \norm{\delta^{(n)}}_2\gnorm{\hat g - g_0} + \frac{\sconv}{2}\norm{\delta^{(n)}}_2^2} \leq&  (1-\sconv\eta)\norm{\delta^{(n-1)}}_2^2 \\
    &- \norm{\delta^{(n)}}_2^2 - 2\eta \langle v^{(n)}, \delta^{(n-1)}\rangle + \frac{\eta^2\norm{v^{(n)}}_2^2}{1 - \smooth\eta}.
\end{align*}
Rearranging it, we have
\begin{align}
    (1 + \mu \eta )\norm{\delta^{(n)}}_2^2 \leq& (1-\sconv\eta)\norm{\delta^{(n-1)}}_2^2\notag \\
    & + 2\eta\secsmooth_1 \norm{\delta^{(n)}}_2\gnorm{\hat g - g_0}- 2\eta \langle v^{(n)}, \delta^{(n-1)}\rangle + \frac{\eta^2\norm{v^{(n)}}_2^2}{1 - \smooth\eta}.\label{eq lemma recursive 1}
\end{align}
By Young's inequality, 
\begin{align*}
    2\eta\secsmooth_1 \norm{\delta^{(n)}}_2\gnorm{\hat g - g_0} \leq \eta\secsmooth_1\p{\frac{\sconv}{\secsmooth_1}\norm{\delta^{(n)}}_2^2 + \frac{\secsmooth_1}{\sconv}\gnorm{\hat g - g_0}^2}.
\end{align*}
Take this into \eqref{eq lemma recursive 1} and we have
\begin{align*}
    \norm{\delta^{(n)}}_2^2 \leq (1-\sconv\eta)\norm{\delta^{(n-1)}}_2^2 + \eta\secsmooth_1^2\sconv^{-1}\gnorm{\hat g - g_0}^2 - 2\eta \langle v^{(n)}, \delta^{(n-1)}\rangle + \frac{\eta^2\norm{v^{(n)}}_2^2}{1 - \smooth\eta}.
\end{align*}
\end{proof}

\begin{corollary}\label{cor: expectation recursion}
    Suppose that \Cref{assumption of population risk} holds. 
  If $\eta < 1/\smooth$, $\sgd^{(n)} \in \Theta$, and $\hat g \in \G_{r}(g_0)$,   it holds that
    \begin{align*}
     \E_{Z_{n}\sim \prob}\sbr{\norm{\delta^{(n)}}_2^2} \leq& \p{1-\sconv\eta+\frac{\kappaconst_1\eta^2}{1 - \smooth\eta}}\norm{\delta^{(n-1)}}_2^2 +\frac{\secsmooth_1^2 \eta}{\sconv}\gnorm{\hat g - g_0}^2 + \frac{\Kconst_1\eta^2}{1 - \smooth\eta}.
    \end{align*}   
\end{corollary}

\begin{proof}
Note that $\E_{Z_{n}\sim \prob}\sbr{\ip{v^{(n)}, \delta^{(n-1)}}} = 0$. By \Cref{lemma: sgd recursive formula}, we have
\begin{align*}
        \E_{Z_{n}\sim \prob}\sbr{\norm{\delta^{(n)}}_2^2} \leq& \p{1-\sconv\eta}\norm{\delta^{(n-1)}}_2^2+\frac{\secsmooth_1^2 \eta}{\sconv}\gnorm{\hat g - g_0}^2 
    + \frac{\eta^2}{1 - \smooth\eta}\E_{Z_{n}\sim \prob}\sbr{\norm{v^{(n)}}_2^2}. 
\end{align*}   
Under \Cref{assumption of population risk}, $\E_{Z_{n}\sim \prob}\sbr{\norm{v^{(n)}}_2^2} \leq \Kconst_1 + \kappaconst_1\norm{\delta^{(n)}}_2^2$, and it follows that
\begin{align*}
     \E_{Z_{n}\sim \prob}\sbr{\norm{\delta^{(n)}}_2^2} \leq& \p{1-\sconv\eta+\frac{\kappaconst_1\eta^2}{1 - \smooth\eta}}\norm{\delta^{(n-1)}}_2^2 +\frac{\secsmooth_1^2 \eta}{\sconv}\gnorm{\hat g - g_0}^2 + \frac{\Kconst_1\eta^2}{1 - \smooth\eta}.
\end{align*}
\end{proof}

\begin{lemma}\label{lemma: sgd recursive formula orthogonal}
    Suppose that \Cref{assumption of population risk} and \Cref{assumption orthogonality main} hold. If $\eta < 1/\smooth$, $\sgd^{(n)} \in \Theta$, and $\hat g \in \G_{r}(g_0)$, it holds that
\begin{align*}
        \norm{\delta^{(n)}}_2^2 \leq& \p{1-\sconv\eta}\norm{\delta^{(n-1)}}_2^2+\frac{\highsmooth_1^2\eta}{4\sconv}\gnorm{\hat g - g_0}^4 - 2\eta\ip{v^{(n)}, \delta^{(n-1)}}
    + \frac{\eta^2\norm{v^{(n)}}_2^2}{1 - \smooth\eta}. 
\end{align*} 
\end{lemma}

\begin{proof}
   Under \Cref{assumption of population risk} and \Cref{assumption orthogonality main},
\begin{align*}
    D_\theta L(\theta_\star, \hat g)[\delta^{(n)}] =&  D_\theta L(\theta_\star, g_0)[\delta^{(n)}] + D_g D_\theta L(\theta_\star, g_0)[\delta^{(n)}, \hat g - g_0] \\
    &+ \frac{1}{2}D_g^2 D_\theta L(\theta_\star, \bar g)[\delta^{(n)}, \hat g - g_0, \hat g - g_0]\notag\\
    =& \frac{1}{2}D_g^2 D_\theta L(\theta_\star, \bar g)[\delta^{(n)}, \hat g - g_0, \hat g - g_0] \geq -\frac{\highsmooth_1}{2} \norm{\delta^{(n)}}_2\gnorm{\hat g - g_0}^2.  
\end{align*}
The rest of the proof is similar to \Cref{lemma: sgd recursive formula}.
\end{proof}

\begin{corollary}\label{cor: expectation recursion orthogonal}
Suppose that \Cref{assumption of population risk} and \Cref{assumption orthogonality main} hold. 
  If $\eta < 1/\smooth$, $\sgd^{(n)} \in \Theta$, and $\hat g \in \G_{r}(g_0)$,   it holds that
    \begin{align*}
    \E_{Z_{n}\sim \prob}\sbr{\norm{\delta^{(n)}}_2^2} \leq& \p{1-\sconv\eta+\frac{\kappaconst_1\eta^2}{1 - \smooth\eta}}\norm{\delta^{(n-1)}}_2^2+\frac{\highsmooth_1^2\eta}{4\sconv}\gnorm{\hat g - g_0}^4 + \frac{\Kconst_1\eta^2}{1 - \smooth\eta}.
    \end{align*} 
\end{corollary}
\begin{proof}
The proof is similar to \Cref{cor: expectation recursion} using \Cref{lemma: sgd recursive formula orthogonal}.
\end{proof}

\subsection{Proof of Theorem \ref{theorem convergence rate baseline SGD}}\label{sec: sgd thm proof}

\begin{proof}

Let $c(\eta) = \sconv - \kappaconst_1\eta/(1 - \smooth\eta)$. When $\eta < \sconv/( \smooth\sconv + \kappaconst_1)$, we have
\begin{align*}
    1-\sconv\eta+\frac{\kappaconst_1\eta^2}{1 - \smooth\eta} = 1 - \p{\sconv - \frac{\kappaconst_1\eta}{1 - \smooth\eta}}\eta = 1 - c(\eta)\eta < 1.
\end{align*}

Under \Cref{assumption of population risk} and by \Cref{cor: expectation recursion}, we have that
\begin{align*}
    \E_{\mc{D}_{n} \sim \prob^n}\sbr{\norm{\delta^{(n)}}_2^2}=&\E_{\mc{D}_{n-1} \sim \prob^{n-1}}\sbr{\E_{Z_{n}\sim \prob}\sbr{\norm{\delta^{(n)}}_2^2}}\\
    \leq& \p{1-c(\eta)\eta}\E_{\mc{D}_{n-1} \sim \prob^{n-1}}\sbr{\norm{\delta^{(n-1)}}_2^2}+ \frac{\secsmooth_1^2 \eta}{\sconv} \gnorm{\hat g - g_0}^2 + \frac{\Kconst_1 \eta^2}{1 - \smooth\eta}\\
    \leq & \p{1-c(\eta)\eta}^2\E_{\mc{D}_{n-2}\sim \prob^{n-2}}\sbr{\norm{\delta^{(n-1)}}_2^2}\\
    &+ \br{1 + \p{1-c(\eta)\eta}}\frac{\secsmooth_1^2 \eta}{\sconv} \gnorm{\hat g - g_0}^2 + \br{1 + \p{1-c(\eta)\eta}}\frac{\Kconst_1 \eta^2}{1 - \smooth\eta}.
\end{align*}
By recursion, it follows that
\begin{align*}
    \E_{\mc{D}_{n} \sim \prob^n}\sbr{\norm{\delta^{(n)}}_2^2}\leq& \p{1-c(\eta)\eta}^n\norm{\delta^{(0)}}_2^2+ \frac{\secsmooth_1^2\eta}{\sconv} \gnorm{\hat g - g_0}^2\sum_{i=0}^{n-1}\p{1-c(\eta)\eta}^i\\
    &+ \frac{\Kconst_1 \eta^2}{1 - \smooth\eta}\sum_{i=0}^{n-1}\p{1-c(\eta)\eta}^i\\
    \leq& \p{1-c(\eta)\eta}^n\norm{\delta^{(0)}}_2^2 + \frac{\secsmooth_1^2}{\sconv c(\eta)} \gnorm{\hat g - g_0}^2 + \frac{\Kconst_1 \eta}{c(\eta)(1 - \smooth\eta)}.
\end{align*}
If $\eta \leq \sconv/2( \smooth\sconv + \kappaconst_1)$, we have $\sconv/2 \leq c(\eta) \leq \sconv$. Thus,
\begin{align*}
     \E_{\mc{D}_{n} \sim \prob^n}\sbr{\norm{\delta^{(n)}}_2^2}   \leq \p{1-\frac{\sconv\eta}{2}}^n\norm{\delta^{(0)}}_2^2 + \frac{2\secsmooth_1^2}{\sconv^2} \gnorm{\hat g - g_0}^2 + \frac{4\Kconst_1 \eta}{\sconv}.
\end{align*}

In addition, if \Cref{assumption orthogonality main} holds, then by \Cref{cor: expectation recursion orthogonal} and using identical proof as above, it follows that for a Neyman orthogonal risk $L$,
\begin{align*}
     \E_{\mc{D}_{n} \sim \prob^n}\sbr{\norm{\delta^{(n)}}_2^2}   \leq \p{1-\frac{\sconv\eta}{2}}^n\norm{\delta^{(0)}}_2^2 + \frac{\highsmooth_1^2}{2\sconv^2} \gnorm{\hat g - g_0}^4 + \frac{4\Kconst_1 \eta}{\sconv}.
\end{align*}
\end{proof}

\clearpage

\section{Orthogonalization with respect to Nuisance}\label{sec:orthogonalization}

In this section, we establish our orthogonalization method for the possibly infinite-dimensional nuisance introduced in \Cref{sec:optimization} of the main text. We demonstrate how we construct the orthogonalizing operator in \Cref{sec: ortho method}, and provide all the technical lemmas in \Cref{sec: ortho lemma}.

\subsection{Orthogonalization via Riesz Representation}\label{sec: ortho method}

We consider $\G \equiv (\G, \ip{\cdot, \cdot}_\G)$ as a possibly infinite-dimensional Hilbert space. Recall the \textit{derivative operator} $\D_g$ defined as for any $h \in \G$,
\begin{align*}
    \D_g \loss(\theta, g; z)[h] =  \frac{\d}{\d t} \loss(\theta, g + th; z)\mid_{t = 0}.
\end{align*}
This derivative operator is also known as the Gateaux derivative. We posit the usual assumption as \citet{jordan2022empirical} that the derivative operator $\D_g \loss(\theta, g; z)$ is linear and continuous in $\G$ for any $(\theta, g, z) \in \Theta \times \Gr \times \calZ$. We also assume regularity conditions such that $\D_g\D_\theta \loss(\theta, g; z)$ is continuous and $\D_g\D_\theta \loss(\theta, g; z) = \D_\theta\D_g\loss(\theta, g; z)$ at any $(\theta, g, z)$. 

Since $\D_g \loss(\theta, g; z)$ is linear and continuous, by the Riesz representation theorem, there uniquely exists some $\nabla_g \ell(\theta, g; z) \in \G$ such that for any $g \in \G$,  
\begin{align}\label{eq Riesz Dg}
    \D_g \loss(\theta, g; z)[g] = \ip{\nabla_g \ell(\theta, g; z), g}_{\G}.
\end{align}
\Cref{lemma linear Dg score} shows that the operator $\D_g\score(\theta_\star, g_0)$ is linear and continuous. By Riesz representation theorem, we can define $H_{\theta g} = (H_{\theta g}\pow{1}, \dots, H_{\theta g}\pow{d}) \in \G^d$ such that for all $g \in \G$,
\begin{align}\label{def H theta g}
    \D_g S_\theta\pow{j}(\theta_\star, g_0)[g] = \ip{H_{\theta g}\pow{j}, g}_{\G} \text{ for any $j=1,\dots, n$.}
\end{align}
The Hessian operator $H_{gg}: \G \mapsto \G$ is defined as for any $g_1, g_2 \in \G$,
\begin{align}\label{def H gg}
    \D_g^2 \ploss(\theta_\star, g_0)[g_1,g_2] = \ip{H_{gg}g_1, g_2}_{\G}.
\end{align}
We will show in \Cref{lemma hessian operator} that $H_{gg}$ uniquely exists and is an self-adjoint and bounded linear operator when $\D_g^2 \ploss(\theta_\star, g_0)$ is bounded and symmetric bilinear. Assuming that $H_{gg}$ is invertible, we define the orthogonalizing operator as
\begin{align}\label{appx gamma}
    \pdir_0: \G \mapsto \R^d, \quad [\pdir_0 g]_j = \ip{H_{\theta g}\pow{j}, H_{gg}^{-1} g}_\G, \forall g \in \G.
\end{align}
We now construct the Neyman orthogonalized (NO) gradient oracle 
\begin{align*}
    \noscore(\theta, g;z) = \score(\theta, g; z)- \pdir_0 \nabla_g\loss(\theta, g;z).
\end{align*}
In addition, $\pdir_0 \nabla_g\loss(\theta, g;z)$ can be written as the derivative in the sense that for each $j = 1,\dots,d$,
\begin{align*}
    [\pdir_0 \nabla_g\loss(\theta, g;z)]_j &= \ip{H_{\theta g}\pow{j}, H_{gg}^{-1} \nabla_g\loss(\theta, g;z)}_\G\\
    &=\ip{H_{gg}^{-1}H_{\theta g}\pow{j},  \nabla_g\loss(\theta, g;z)}_\G = \D_g\loss(\theta, g;z)[H_{gg}^{-1}H_{\theta g}\pow{j}].
\end{align*}
That is, the NO gradient oracle can be easily obtain by
\begin{align}\label{def no score appx}
    \noscore(\theta, g;z) = \score(\theta, g; z)- \D_g\loss(\theta, g;z)[H_{gg}^{-1}H_{\theta g}].
\end{align}
The following Lemma shows that $\noscore(\theta, g;z)$ is Neyman orthogonal at $(\theta_\star, g_0)$. 
\begin{lemma}\label{lemma construct neyman orthogonal score}
    $\noscore(\theta, g;z)$ is a Neyman orthogonal score at $(\theta_0, g_0)$.
\end{lemma}
\begin{proof}
Since $\noscore(\theta, g;z) = \score(\theta, g; z)- \D_g\loss(\theta, g;z)[H_{gg}^{-1}H_{\theta g}]$, for any $h \in \G$,
    \begin{align*}
        \E_\prob\sbr{\D_g S_{\text{no}}(\theta_0, g_0;Z)[h]} &= \D_g \score(\theta_0, g_0)[h] - \D_g^2\ploss(\theta_0, g_0)[H_{gg}^{-1}H_{\theta g}, h]\\
        &= \ip{H_{\theta g}, h}_\G - \ip{H_{gg} H_{gg}^{-1}H_{\theta g}, h}_\G\\
        &=\ip{H_{\theta g}, h}_\G - \ip{H_{\theta g}, h}_\G = 0,
    \end{align*}
    which implies that $\noscore(\theta, g;z)$ is Neyman orthogonal at $(\theta_0, g_0)$.
\end{proof}

\subsection{Technical Lemma}\label{sec: ortho lemma}

\begin{lemma}\label{lemma linear Dg score}
    $\D_g \score(\theta, g; z): \G \mapsto \R^d$ and $\D_g \score(\theta, g): \G \mapsto \R^d$ are linear and continuous in $\G$.
\end{lemma}
\begin{proof}
    The continuity of $\D_g \score(\theta, g; z)$ and $\D_g \score(\theta, g)$ follows from the continuity of $\D_g\D_\theta \loss(\theta, g; z)$. It suffices to prove that $\D_g \score(\theta, g; z)$ is linear. For all $u \in \R^d$, $h_1, h_2 \in \G$,
\begin{align*}
    \ip{u, \D_g \score(\theta, g; z)[h_1 + h_2]} &=\D_g\ip{u, \score(\theta, g; z)}[h_1 + h_2]\\
    &=\D_g\D_\theta\loss(\theta, g; z)[u, h_1 + h_2]\\
    &=\D_\theta\D_g\loss(\theta, g; z)[ h_1 + h_2,u]\\
    &= \ip{\nabla_\theta\p{\D_g \loss(\theta, g; z)[h_1 + h_2]}, u}\\
    &=\ip{\nabla_\theta\p{\D_g \loss(\theta, g; z)[h_1 ] + \D_g \loss(\theta, g; z)[h_2]}, u}\\
    &=\D_\theta \D_g \loss(\theta, g; z)[h_1, u] + \D_\theta \D_g \loss(\theta, g; z)[h_2, u]\\
    &=  \D_g \D_\theta\loss(\theta, g; z)[u,h_1] + \D_g \D_\theta \loss(\theta, g; z)[u,h_2]\\
    &= \ip{u, \D_g \score(\theta, g; z)[h_1]+\D_g \score(\theta, g; z)[h_2]},
\end{align*}
which implies that 
\begin{align*}
    \D_g \score(\theta, g; z)[h_1 + h_2] = \D_g \score(\theta, g; z)[h_1]+\D_g \score(\theta, g; z)[h_2].
\end{align*}

\end{proof}

\begin{lemma}\label{lemma hessian operator}
    Suppose that $\D_g^2 \ploss(\theta_\star, g_0)[\cdot,\cdot]: \G \times \G \mapsto \R$ is a bounded and symmetric bilinear form. Then $H_{gg}: \G \mapsto \G$ uniquely exists and is self-adjoint, bounded, and linear.
\end{lemma}
\begin{proof}
    Given $g_1, g_2 \in \G$, since $\D_g^2 \ploss(\theta_\star, g_0)[g_1,\cdot]$ is a bounded linear map from $\G$ to $\R$, by Riesz representation theorem, for any $g_2 \in \G$, there uniquely exists some $Th_1 \in \G$ such that 
    \begin{align*}
        \D_g^2 \ploss(\theta_\star, g_0)[g_1,g_2] = \ip{Tg_1, g_2}_\G.
    \end{align*}
    Thus, we define the operator $T: \G \mapsto \G$. Note that $\D_g^2 \ploss(\theta_\star, g_0)[\cdot,\cdot]$ is bilinear. For any $a, a' \in \R$, and any $g_1, g_1', g_2 \in \G$, we have
    \begin{align*}
        \ip{T(ag_1 + a'g_1'), g_2}_\G&=\D_g^2 \ploss(\theta_\star, g_0)[ag_1 + a'g_1',g_2] \\
        &= a\D_g^2 \ploss(\theta_\star, g_0)[g_1,g_2] + a'\D_g^2 \ploss(\theta_\star, g_0)[g_1',g_2]\\
        &=a\ip{Tg_1, g_2}_\G + a'\ip{Tg_1', g_2}_\G\\
        &=\ip{aTg_1 + a'Tg_1', g_2}_\G,
    \end{align*}
    which implies $T$ is a linear operator. To show $T$ is bounded, suppose that the norm of the bilinear form $\D_g^2 \ploss(\theta_\star, g_0)$ is bounded by $B$. Thus, for $Tg_1 \neq 0$,
    \begin{align*}
        \gnorm{Tg_1} = \ip{Tg_1, \frac{Tg_1}{\gnorm{Tg_1}}}_\G \leq \sup_{\gnorm{g_2}=1} \ip{Tg_1, g_2}_\G \leq \sup_{\gnorm{g_2}=1} \abs{\D_g^2 \ploss(\theta_\star, g_0)[g_1,g_2]}\leq B\gnorm{g_1},
    \end{align*}
    which implies $T$ is bounded. Note that $\D_g^2 \ploss(\theta_\star, g_0)[\cdot,\cdot]$ is symmetric, we have $T$ being self-adjoint since
    \begin{align*}
        \ip{Tg_1, g_2}_\G = \D_g^2 \ploss(\theta_\star, g_0)[g_1,g_2] = \D_g^2 \ploss(\theta_\star, g_0)[g_2,g_1] = \ip{Tg_2, g_1}_\G = \ip{g_1, Tg_2}_\G.
    \end{align*}
    Finally, we show that $T$ is unique. If there exists some $T': \G \mapsto \G$ such that for any $g_1, g_2 \in \G$,
    \begin{align*}
        \ip{Tg_1, g_2}_\G = \D_g^2 \ploss(\theta_\star, g_0)[g_1,g_2] = \ip{T'g_1, g_2}_\G,
    \end{align*}
    which implies
    \begin{align*}
        \ip{(T-T')g_1, g_2}_\G = 0 \quad \forall g_1, g_2 \in \G.
    \end{align*}
    That is, $T = T'$. We finish the proof by letting $H_{gg} = T$.
\end{proof}

\clearpage

\section{Convergence Proofs for Orthogonalized Stochastic Gradient}\label{sec:a:debiased}
This section is dedicated to demonstrate the OSGD convergence in \Cref{theorem debiased SGD convergence rate}  of the main text. In \Cref{sec osgd overview}, we give an overview of the OSGD settings, the additional assumptions, and the expected results with the proof outline. We then provide all the technical lemmas in \Cref{appdx osgd lemma}, and finally prove \Cref{theorem debiased SGD convergence rate} in \Cref{appdx osgd thm proof}.

\subsection{Overview}\label{sec osgd overview}
Following the same problem settings in \Cref{sec:a:sgd}, we consider the orthogonalized SGD (OSGD) using the estimated NO score $\noscorehat$ oracle defined as
\begin{align}
    \noscorehat(\theta, g; z) = \score(\theta, g; z) - \pdirhat \nabla_g \loss(\theta, g; z),
\end{align}
where $\pdirhat$ is an estimator for the orthogonalizing operator defined in \eqref{appx gamma}. Specifically, we consider all continuous linear $\pdirhat: \G \mapsto \R^d$ for estimating the orthogonalizing operator $\pdir_0$. By Riesz representation theorem, there exists some $\hat \gamma^{(j)} \in \G$ for $j=1,\dots,d,$ such that 
\begin{align*}
    [\pdirhat g]_j = \ip{\hat \gamma^{(j)}, g}_\G \text{ for all } g \in \G.
\end{align*}
For the orthogonalizing operator $\pdir_0$, we define $\gamma_0^{(j)} = H_{gg}^{-1}H_{\theta g}^{(j)}, j = 1,\dots, d,$ such that
\begin{align*}
    [\pdir_0 g]_j = \ip{\gamma_0^{(j)}, g}_\G \text{ for all } g \in \G.
\end{align*}
We focus on the OSGD defined below:
\begin{align}\label{estimated OSGD appendix}
    \dsgd^{(n)} = \dsgd^{(n-1)} - \eta \noscorehat(\dsgd^{(n-1)}, \hat g;Z_n), \quad \dsgd^{(0)} \in \Theta.
\end{align}
Throughout the section, we take the following notations for simplicity:
\begin{align}
    \dno^{(n)} &= \dsgd^{(n)} - \theta_\star,\\
    \noscorehat^{(n)} &= \noscorehat(\dsgd^{(n-1)}, \hat g; Z_n),\\
    \vno^{(n)} &= \noscorehat^{(n)} - \noscorehat(\dsgd^{(n-1)}, \hat g).
\end{align}

Let $\grad_g\ploss(\theta, g) = \E_{Z \sim \prob}[\grad_g\loss(\theta, g; Z)]$ and $\noscore(\theta, g) = \E_{Z \sim \prob}[\noscore(\theta, g; Z)]$. We need the following assumptions to establish the convergence result of the OSGD.

\begin{assumption}\label{asm:noscore}
 The following conditions hold:
\begin{enumerate}[label=(\alph*)]
    \item \emph{First-order optimality:} $\noscore(\theta_\star, g_0) = 0$ and $(\pdirhat - \pdir_0)\nabla_g L(\theta_\star, g_0) = 0$.
    \item \emph{Smoothness and strong convexity:} There exists some $\hconst, \hessianconst >0$ such that for all $\theta \in \Theta$ and $g \in \G_{r}\left(g_0\right)$, $\norm{\nabla_\theta \noscore(\theta, g)}_2 \leq \hconst$ and
    \begin{align*}
        \lambda_{\min}(\nabla_\theta \noscore(\theta, g) + \nabla_\theta \noscore(\theta, g)^\top) \geq 2\hessianconst.
    \end{align*}
    \item \emph{Second-moment growth:} There exist constants $\Kconst_2, \kappaconst_2 > 0$ such that 
    \begin{align*}
        \E_{Z \sim \prob}[\p{\D_g \ploss(\theta, \bar g;Z)[g] -  \D_g\ploss(\theta, \bar g)[g]}^2] \leq (\Kconst_2 + \kappaconst_2\norm{\theta - \theta_\star}_2^2)\gnorm{g}^2.
    \end{align*}
    for all $\theta \in \Theta$, $\bar g\in \G_{r}\left(g_0\right)$, and $g \in \G$.
    
    \item \emph{Second-order smoothness:} There exists a constant $\secsmooth_2>0$ such that
    \begin{align*}
        &\abs{\D_g^2\ploss(\theta, \bar g)[g_1, g_2]} \leq \secsmooth_2\gnorm{g_1}\gnorm{g_2} \quad \forall \theta \in \Theta, \bar g \in \G_{r}\left(g_0\right), g_1, g_2 \in \G,\\
        &\abs{\D_\theta\D_g\ploss(\bar \theta, g_0)[g, \theta - \theta_\star]} \leq \secsmooth_1\norm{\theta - \theta_\star}_2\gnorm{g} \quad \forall \theta, \bar \theta \in \Theta, g \in \G.
    \end{align*}
    \item \emph{Higher-order smoothness:} There exists a constants $\highsmooth_2>0$ such that
    \begin{align*}
        \norm{\D_g^2 \noscore(\theta_\star, \bar g)[g - g_0, g - g_0]}_2 \leq \highsmooth_2 \gnorm{g - g_0}^2 \quad \forall g, \bar g \in \G_{r}\left(g_0\right).
    \end{align*}
    \end{enumerate}
\end{assumption}
\Cref{asm:noscore}(a) is necessary for the convergence of the OSGD to $\theta_\star$. When $\score$ is Neyman orthogonal at $(\theta_\star, g_0)$, $\Gamma_0 = 0$ is accessible and thus, $\noscore = \score$. When $\score$ is non-orthogonal, \Cref{asm:noscore}(a) can be satisfied whenever $\nabla_g L(\theta_\star, g_0) = 0$, implying that $(\theta_\star, g_0)$ is a \emph{local} minimizer of $\ploss(\theta, g)$. \Cref{asm:noscore}(b) is related to the Schur complement of the population Hessian. Thus, the hypothetical objective relating to $\noscore$ inherits its strong convexity from that of the population risk $\ploss$ \wrt $(\theta, g) \in \Theta \times \G_r(g_0)$ when $\G$ is finite-dimensional; see \citet{boyd2004convex}. \Cref{asm:noscore}(c) and (d) are exactly analogous to \Cref{assumption of population risk}(d) and (e), while \Cref{asm:noscore}(e) is analogous to \Cref{assumption orthogonality main}.

With \Cref{asm:noscore}, we aim to show that the error $\dno\pow{n}$ satisfies
\begin{align}\label{appdx osgd rate}
        \E_{\mathcal{D}_n \sim \prob^n}[\Vert\dno\pow{n}\Vert_2^2] &\lesssim \p{1 - \frac{\hessianconst \eta}{2}}^n\Vert \dno\pow{0}\Vert_2^2+ \Vert\hat g - g_0\Vert_\G^4 + \Vert\hat g - g_0\Vert_\G^2 \cdot \norm{\pdirhat - \pdir_0}_{\Fro}^2 + \eta.
    \end{align}

\myparagraph{Proof Outline}
The proof the \eqref{appdx osgd rate} follows the following four steps:
\begin{enumerate}
    \item Upper bound $\norm{I - \eta \nabla_\theta \noscorehat(\theta, g)}_2$ \wrt the operator estimation error $\Vert \pdirhat - \pdir_0\Vert_\Fro$.
    \item Upper bound $\norm{\noscorehat(\theta_\star, \hat g)}_2$ using Neyman orthogonality and the first order optimality.
    \item Derive a recursive formula of $\E_{\calD_n\sim\prob^n}[\norm{\dno^{(n)}}_2^2]$ from these bounds.
    \item Perform the recursion and obtain the final result.
\end{enumerate}
Follow these steps above, we provide technical lemma in \Cref{appdx osgd lemma}, and then prove our second main result \Cref{theorem debiased SGD convergence rate} in \Cref{appdx osgd thm proof}.

Alternatively, the intuition of step 1 also suggests that we should focus on $\pdirhat$ that lies in the neighborhood of $\pdir_0$ such that $\Vert \pdirhat - \pdir_0\Vert_\Fro \leq R$ for a small $R>0$. Then, instead of assuming \Cref{asm:noscore}(b), we can directly assume that for all $\theta \in \Theta$ and $g \in \G_{r}\left(g_0\right)$, $\norm{\nabla_\theta \noscorehat(\theta, g)}_2 \leq \hconst$ and
    \begin{align*}
        \lambda_{\min}(\nabla_\theta \noscorehat(\theta, g) + \nabla_\theta \noscorehat(\theta, g)^\top) \geq 2\hessianconst.
    \end{align*}
With this assumption, one can still show the same OSGD convergence rate by the identical proof while the constraint of the learning rate $\eta$ will be simplified.

\subsection{Technical Lemma}\label{appdx osgd lemma}

\begin{lemma}\label{lemma: osgd youngs}
Given $\eta >0$. For any $\omega >0$ and $u, v \in \R^d$,
    \begin{align*}
        \norm{u + \eta v}_2^2 \leq (1 + \eta\omega)\norm{u}_2^2 + (\eta^2 + \eta\omega^{-1})\norm{v}_2^2.
    \end{align*} 
\end{lemma}

\begin{proof}
    By definition,
    \begin{align*}
        \norm{u + \eta v}_2^2 = \norm{u}_2^2 + \eta^2\norm{v}_2^2 + 2\eta\ip{u, v} \leq \norm{u}_2^2 + \eta^2\norm{v}_2^2 + 2\eta\norm{u}_2\norm{v}_2.
    \end{align*}
    By Young's inequality, for any $\omega>0$,
    \begin{align*}
        2\norm{u}_2\norm{v}_2 \leq \omega\norm{u}_2^2 + \omega^{-1}\norm{v}_2^2.
    \end{align*}
    Thus,
    \begin{align*}
        \norm{u + \eta v}_2^2 \leq (1 + \eta\omega)\norm{u}_2^2 + (\eta^2 + \eta\omega^{-1})\norm{v}_2^2.
    \end{align*}
\end{proof}

\begin{lemma}\label{lemma: osgd smallest eigenvalue}
    Suppose that \Cref{asm:noscore} holds. For all $\theta \in \Theta$ and $g \in \Gr$, 
    \begin{align*}
        \norm{I - \eta\nabla_\theta \noscorehat(\theta, g)}_2^2 \leq 1 - 2\eta(\hessianconst  - \secsmooth_1 \norm{\pdirhat - \pdir_0}_\Fro) + 2\eta^2(\hconst^2 + 2\secsmooth_1^2\norm{\pdirhat - \pdir_0}_\Fro^2).
    \end{align*}
\end{lemma}

\begin{proof}
Note that
    \begin{align*}
        (I - &\eta\nabla_\theta \noscorehat(\theta, g))(I - \eta\nabla_\theta \noscorehat(\theta, g))^\top\\
        &= I - \eta\p{\nabla_\theta \noscorehat(\theta, g) + \nabla_\theta \noscorehat(\theta, g)^\top} + \eta^2 \nabla_\theta \noscorehat(\theta, g)\nabla_\theta \noscorehat(\theta, g)^\top.
    \end{align*}
Since $\nabla_\theta \noscorehat(\theta, g) = \nabla_\theta \noscore(\theta, g) - \nabla_\theta((\pdirhat - \pdir_0)\nabla_g L(\theta,g))$, we have
\begin{align*}
    \nabla_\theta \noscorehat(&\theta, g) + \nabla_\theta \noscorehat(\theta, g)^\top \\
    &= \nabla_\theta \noscore(\theta, g) + \nabla_\theta \noscore(\theta, g)^\top - \nabla_\theta((\pdirhat - \pdir_0)\nabla_g L(\theta,g)) - \nabla_\theta((\pdirhat - \pdir_0)\nabla_g L(\theta,g))^\top\\
    &\succcurlyeq 2(\hessianconst  - \norm{\nabla_\theta((\pdirhat - \pdir_0)\nabla_g L(\theta,g))}_2)\mathbf{I}.
\end{align*}
We now bound $\norm{\nabla_\theta (\pdirhat - \pdir_0)\nabla_g L(\theta, g)}_2$. For each $j = 1,\dots, d$, for any $\theta \in \R^d$,
\begin{align*}
    \abs{\nabla_\theta ((\pdirhat - \pdir_0)^{(j)}\nabla_g L(\theta, g)) (\theta - \theta_\star)} &= \abs{\D_\theta\ip{\hat \gamma^{(j)} - \gamma_0^{(j)},\nabla_g L(\theta, g)}_\G[\theta - \theta_\star]}\\
    &=\abs{\D_\theta\D_g L(\theta, g)[\hat \gamma^{(j)} - \gamma_0^{(j)},\theta - \theta_\star]}\\
    &\leq \secsmooth_1 \gnorm{\hat \gamma^{(j)} - \gamma_0^{(j)}}\norm{\theta - \theta_\star}_2.
\end{align*}
Thus,
\begin{align*}
    \norm{\nabla_\theta ((\pdirhat - \pdir_0)\nabla_g L(\theta, g)) (\theta - \theta_\star)}_2 \leq \secsmooth_1 \norm{\pdirhat - \pdir_0}_\Fro\norm{\theta - \theta_\star}_2,
\end{align*}
which implies that $\norm{\nabla_\theta (\pdirhat - \pdir_0)\nabla_g L(\theta, g)}_2 \leq \secsmooth_1 \norm{\pdirhat - \pdir_0}_\Fro$ and
\begin{align*}
    \nabla_\theta \noscorehat(\theta, g) + \nabla_\theta \noscorehat(\theta, g)^\top \geq 2(\hessianconst  - \secsmooth_1 \norm{\pdirhat - \pdir_0}_\Fro)\mathbf{I}.
\end{align*}
Additionally, we have
\begin{align*}
    \norm{\nabla_\theta \noscorehat(\theta, g)\nabla_\theta \noscorehat(\theta, g)^\top}_2 &\leq \norm{\nabla_\theta \noscorehat(\theta, g)}_2^2\\
    &\leq 2\norm{\nabla_\theta \noscore(\theta, g)}_2^2 + 2\norm{\nabla_\theta((\pdirhat - \pdir_0)\nabla_g L(\theta,g))}_2^2\\
    &\leq 2\hconst^2 + 2\secsmooth_1^2\norm{\pdirhat - \pdir_0}_\Fro^2.
\end{align*}
In conclusion, we have
\begin{align*}
    (I - &\eta\nabla_\theta \noscorehat(\theta, g))(I - \eta\nabla_\theta \noscorehat(\theta, g))^\top \\
    &\preccurlyeq \p{1 - 2\eta(\hessianconst  - \secsmooth_1 \norm{\pdirhat - \pdir_0}_\Fro) + 2\eta^2(\hconst^2 + 2\secsmooth_1^2\norm{\pdirhat - \pdir_0}_\Fro^2)}\mathbf{I}.
\end{align*}
\end{proof}

\begin{lemma}\label{lemma: osgd norm bound noscorehat}
Suppose that \Cref{asm:noscore} holds. When $\hat g \in \Gr$, 
    \begin{align*}
        \norm{\noscorehat(\theta_\star, \hat g)}_2^2 \leq \frac{\highsmooth_2^2}{2} \gnorm{\hat g - g_0}^4 + 2\secsmooth_2^2\norm{\pdirhat - \pdir_0}_\Fro^2\cdot\gnorm{\hat g - g_0}^2.
    \end{align*}

    \begin{proof}
        Note that 
        \begin{align*}
            \norm{\noscorehat(\theta_\star, \hat g)}_2^2 &= \norm{\noscore(\theta_\star, \hat g) - (\pdirhat - \pdir_0)\nabla_g \ploss(\theta_\star, \hat g)}_2^2\\
            &\leq 2\norm{\noscore(\theta_\star, \hat g)}_2^2 + 2\norm{(\pdirhat - \pdir_0)\nabla_g \ploss(\theta_\star, \hat g)}_2^2\\
            &=2\norm{\noscore(\theta_\star, \hat g)}_2^2 + 2\sum_{j=1}^d\ip{\nabla_g \ploss(\theta_\star, \hat g), \hat \gamma^{(j)} - \gamma_0^{(j)}}_\G^2\\
            &=2\norm{\noscore(\theta_\star, \hat g)}_2^2 + 2\sum_{j=1}^d (\D_g L(\theta_\star, \hat g)[\hat \gamma^{(j)} - \gamma_0^{(j)}])^2.
        \end{align*}
Since $\noscore(\theta_\star, g_0) = 0$ and $\noscore$ is Neyman orthogonal at $(\theta_\star, g_0)$, we have for some $\bar g \in \Gr$,
\begin{align*}
    \noscore(\theta_\star, \hat g) &= \noscore(\theta_\star, g_0) + \D_g\noscore(\theta_\star, g_0)[\hat g - g_0] + \frac{1}{2}\D_g^2\noscore(\theta_\star, \bar g)[\hat g - g_0]\\
    &=\frac{1}{2}\D_g^2\noscore(\theta_\star, \bar g)[\hat g - g_0],
\end{align*}
which implies
\begin{align*}
    \norm{\noscore(\theta_\star, \hat g)}_2 \leq \frac{\highsmooth_2}{2} \gnorm{\hat g - g_0}^2.
\end{align*}
Similarly, since $(\pdirhat^{(j)} - \pdir_0)\nabla_g L(\theta_\star, g_0) = 0$, we have for some $\bar g' \in \Gr$,
\begin{align*}
    \D_g L(\theta_\star, \hat g)[\hat \gamma^{(j)} - \gamma_0^{(j)}] &= \D_g L(\theta_\star, g_0)[\hat \gamma^{(j)} - \gamma_0^{(j)}] + \D_g^2 L(\theta_\star, \bar g')[\hat \gamma^{(j)} - \gamma_0^{(j)}]\\
    &= (\pdirhat^{(j)} - \pdir_0)\nabla_g L(\theta_\star, g_0) + \D_g^2 L(\theta_\star, \bar g')[\hat \gamma^{(j)} - \gamma_0^{(j)}]\\
    &=\D_g^2 L(\theta_\star, \bar g')[\hat \gamma^{(j)} - \gamma_0^{(j)}],
\end{align*}
which implies
\begin{align*}
    \abs{\D_g L(\theta_\star, \hat g)[\hat \gamma^{(j)} - \gamma_0^{(j)}, \hat g - g_0]} \leq \secsmooth_2 \gnorm{\hat \gamma^{(j)} - \gamma_0^{(j)}}\gnorm{\hat g - g_0}.
\end{align*}
In conclusion,
\begin{align*}
    \norm{\noscorehat(\theta_\star, \hat g)}_2^2 &\leq \frac{\highsmooth_2^2}{2} \gnorm{\hat g - g_0}^4 + 2\secsmooth_2^2\norm{\pdirhat - \pdir_0}_\Fro^2\cdot\gnorm{\hat g - g_0}^2.
\end{align*}
    \end{proof}
\end{lemma}

\begin{lemma}\label{lemma: osgd bound vno}
    Suppose that \Cref{assumption of population risk} and \Cref{asm:noscore} holds. Then
    \begin{align*}
        \Ex{Z_n \sim \prob}{\norm{\vno^{(n)}}_2^2} 
    \leq 2(\Kconst_1 + \Kconst_2\norm{\hat \Gamma}_{\Fro}^2) + 2(\kappaconst_1 + \kappaconst_2\norm{\hat \Gamma}_{\Fro}^2)\norm{\dno^{n-1}}_2^2.
    \end{align*}
\end{lemma}
\begin{proof}
Note that $\vno^{(n)} = \noscorehat(\dsgd^{(n-1)}, \hat g;Z_n) - \noscorehat(\dsgd^{(n-1)}, \hat g)$.
    \begin{align*}
    \Ex{Z_n \sim \prob}{\norm{\vno^{(n)}}_2^2} &\leq 2\Ex{Z_n \sim \prob}{\norm{\score(\dsgd^{(n-1)}, \hat g;Z_n) - \score(\dsgd^{(n-1)}, \hat g)}_2^2}\\
    &\quad + 2\Ex{Z_n \sim \prob}{\norm{\pdirhat\grad_g\loss(\dsgd^{(n-1)}, \hat g;Z_n) - \pdirhat\grad_g\ploss(\dsgd^{(n-1)}, \hat g)}_2^2}.
\end{align*}
By \Cref{assumption of population risk},
\begin{align*}
    \Ex{Z_n \sim \prob}{\norm{\score(\dsgd^{(n-1)}, \hat g;Z_n) - \score(\dsgd^{(n-1)}, \hat g)}_2^2} \leq \Kconst_1 + \kappaconst_1\norm{\dno^{n-1}}_2^2.
\end{align*}
Since 
\begin{align*}
    &\Ex{Z_n \sim \prob}{\norm{\pdirhat\grad_g\loss(\dsgd^{(n-1)}, \hat g;Z_n) - \pdirhat\grad_g\ploss(\dsgd^{(n-1)}, \hat g)}_2^2}\\
    &= \sum_{j=1}^d \Ex{Z_n \sim \prob}{(\D_g\loss(\dsgd^{(n-1)}, \hat g;Z_n)[\hat \gamma^{(j)}] - \D_g\ploss(\dsgd^{(n-1)}, \hat g)[\hat \gamma^{(j)}])^2}\\
    &\leq \sum_{j=1}^d(\Kconst_2 + \kappaconst_2\norm{\theta - \theta_\star}_2^2)\gnorm{\hat \gamma^{(j)}}^2 = (\Kconst_2 + \kappaconst_2\norm{\theta - \theta_\star}_2^2)\norm{\pdirhat}_\Fro^2.
\end{align*}
In conclusion,
\begin{align*}
    \Ex{Z_n \sim \prob}{\norm{\vno^{(n)}}_2^2} 
    &\leq 2(\Kconst_1 + \Kconst_2\norm{\hat \Gamma}_{\Fro}^2) + 2(\kappaconst_1 + \kappaconst_2\norm{\hat \Gamma}_{\Fro}^2)\norm{\dno^{n-1}}_2^2.
\end{align*}
\end{proof}

\subsection{Proof of Theorem \ref{theorem debiased SGD convergence rate}}\label{appdx osgd thm proof}
\begin{proof}
    Since $\dsgd^{(n)} = \dsgd^{(n-1)} - \eta \noscorehat^{(n)}$, by Taylor's theorem we have that for some $\dsgdbar^{(n-1)}$,
    \begin{align*}
        \dno^{(n)} &= \dno^{(n-1)}-\eta(\noscorehat(\dsgd^{(n-1)}, \hat g) - \noscorehat(\theta_\star, \hat g)) - \eta\noscorehat(\theta_\star, \hat g)  -  \eta \vno^{(n)}\\
        &= (I - \eta \nabla_\theta \noscorehat(\dsgdbar^{(n-1)}, \hat g))\dno^{(n-1)}- \eta\noscorehat(\theta_\star, \hat g)  -  \eta \vno^{(n)}.
    \end{align*}
Note that $\Ex{Z_n \sim \prob}{\vno^{(n)}} = 0$. Take the expectation of the squared norm of both sides \wrt $Z_n$ and we have
\begin{align*}
    \Ex{Z_n\sim \prob}{\norm{\dno^{(n)}}_2^2} = \norm{(I - \eta \nabla_\theta \noscorehat(\dsgdbar^{(n-1)}, \hat g))\dno^{(n-1)}- \eta\noscorehat(\theta_\star, \hat g)}_2^2 + \eta^2\Ex{Z_n \sim \prob}{\norm{ \vno^{(n)}}_2^2}.
\end{align*}
By \Cref{lemma: osgd youngs}, \Cref{lemma: osgd smallest eigenvalue}, and \Cref{lemma: osgd norm bound noscorehat}, for any $\omega>0$,
\begin{align*}
    \Vert (I - &\eta \nabla_\theta \noscorehat(\dsgdbar^{(n-1)}, \hat g))\dno^{(n-1)}- \eta\noscorehat(\theta_\star, \hat g) \Vert_2^2 \\
    &\leq (1 + \eta \omega)\norm{(I - \eta \nabla_\theta \noscorehat(\dsgdbar^{(n-1)}, \hat g))\dno^{(n-1)}}_2^2 + (\eta^2 + \eta\omega^{-1})\norm{\noscorehat(\theta_\star, \hat g)}_2^2\\
    &\leq (1 + \eta \omega)\p{1 - 2\eta(\hessianconst  - \secsmooth_1 \norm{\pdirhat - \pdir_0}_\Fro) + 2\eta^2(\hconst^2 + 2\secsmooth_1^2\norm{\pdirhat - \pdir_0}_\Fro^2)}\norm{\dno^{(n-1)}}_2^2\\
    &\quad + (\eta^2 + \eta\omega^{-1})\p{\frac{\highsmooth_2^2}{2} \gnorm{\hat g - g_0}^4 + 2\secsmooth_2^2\norm{\pdirhat - \pdir_0}_\Fro^2\cdot\gnorm{\hat g - g_0}^2}.
\end{align*}
Set $\omega = \hessianconst$. For $\eta \leq 2/\hessianconst$, we have
\begin{align*}
    &(1 + \eta \omega)\p{1 - 2\eta(\hessianconst  - \secsmooth_1 \norm{\pdirhat - \pdir_0}_\Fro) + 2\eta^2(\hconst^2 + 2\secsmooth_1^2\norm{\pdirhat - \pdir_0}_\Fro^2)}\\
    &=1 - (\hessianconst - 2\secsmooth_1 \norm{\pdirhat - \pdir_0}_\Fro)\eta - 2(\hessianconst^2 - \hessianconst\secsmooth_1 \norm{\pdirhat - \pdir_0}_\Fro)\eta^2 \\
    &\quad + 2(1 + \eta \hessianconst)(\hconst^2 + 2\secsmooth_1^2\norm{\pdirhat - \pdir_0}_\Fro^2)\eta^2\\
    &\leq 1 - (\hessianconst - 2\secsmooth_1 \norm{\pdirhat - \pdir_0}_\Fro)\eta \\
    &\quad+ \p{6\hconst^2-2\hessianconst^2 +2\hessianconst\secsmooth_1 \norm{\pdirhat - \pdir_0}_\Fro  + 12\secsmooth_1^2\norm{\pdirhat - \pdir_0}_\Fro^2}\eta^2\\
    &=: 1 - b(\eta)\eta,
\end{align*}
where
\begin{align*}
    b(\eta) = \hessianconst - 2\secsmooth_1 \norm{\pdirhat - \pdir_0}_\Fro - (6\hconst^2-2\hessianconst^2 +2\hessianconst\secsmooth_1 \norm{\pdirhat - \pdir_0}_\Fro  + 12\secsmooth_1^2\norm{\pdirhat - \pdir_0}_\Fro^2)\eta.
\end{align*}
By \Cref{lemma: osgd bound vno},
\begin{align*}
    \Ex{Z_n \sim \prob}{\norm{\vno^{(n)}}_2^2} 
    \leq 2(\Kconst_1 + \Kconst_2\norm{\hat \Gamma}_{\Fro}^2) + 2(\kappaconst_1 + \kappaconst_2\norm{\hat \Gamma}_{\Fro}^2)\norm{\dno^{n-1}}_2^2.
\end{align*}
It follows that 
\begin{align*}
    \E_{Z_n\sim \prob}[&\norm{\dno^{(n)}}_2^2] \leq (1 - b(\eta) \eta + 2(\kappaconst_1 + \kappaconst_2\norm{\hat \Gamma}_{\Fro}^2)\eta^2)\norm{\dno^{(n-1)}}_2^2\\
    &+\frac{3\eta}{\hessianconst}\p{\frac{\highsmooth_2^2}{2} \gnorm{\hat g - g_0}^4 + 2\secsmooth_2^2\norm{\pdirhat - \pdir_0}_\Fro^2\cdot\gnorm{\hat g - g_0}^2} + 2(\Kconst_1 + \Kconst_2\norm{\hat \Gamma}_{\Fro}^2)\eta^2.
\end{align*}
Thus, it is clear that when $\norm{\pdirhat - \pdir_0}_\Fro \leq \hessianconst/ (4\secsmooth_1) $ and the learning rate satisfies
\begin{align}\label{eq: learning rate}
    \eta &\leq \min\br{\frac{2}{\hessianconst},\frac{\hessianconst - 4\secsmooth_1 \norm{\pdirhat - \pdir_0}_\Fro}{ 12\hconst^2-3\hessianconst^2/2 + 4(\kappaconst_1 + \kappaconst_2\norm{\hat \Gamma}_{\Fro}^2)}} \notag\\
    &= \frac{\hessianconst - 4\secsmooth_1 \norm{\pdirhat - \pdir_0}_\Fro}{12\hconst^2-3\hessianconst^2/2 + 4(\kappaconst_1 + \kappaconst_2\norm{\hat \Gamma}_{\Fro}^2)},
\end{align}
we have $1 - \hessianconst\eta/2 \geq 0$ and
\begin{align*}
    1 - b(\eta) \eta + 2(\kappaconst_1 + \kappaconst_2\norm{\hat \Gamma}_{\Fro}^2)\eta^2 \leq 1-\frac{\hessianconst\eta}{2}.
\end{align*}
When $\eta$ satisfies \eqref{eq: learning rate}, it follows that
\begin{align*}
    \E_{Z_n\sim \prob}[&\norm{\dno^{(n)}}_2^2]\leq \p{1 - \frac{\hessianconst\eta}{2}}\norm{\dno^{(n-1)}}_2^2\\
    &+\frac{3\eta}{\hessianconst}\p{\frac{\highsmooth_2^2}{2} \gnorm{\hat g - g_0}^4 + 2\secsmooth_2^2\norm{\pdirhat - \pdir_0}_\Fro^2\cdot\gnorm{\hat g - g_0}^2}+ 2(\Kconst_1 + \Kconst_2\norm{\hat \Gamma}_{\Fro}^2)\eta^2.
\end{align*}
Finally, perform the same recursion in \Cref{sec: sgd thm proof} and we have
\begin{align*}
    \E_{\calD_n \sim \prob^n}[&\norm{\dno^{(n)}}_2^2]\leq \p{1 - \frac{\hessianconst\eta}{2}}^n\norm{\dno^{(0)}}_2^2 + \frac{4(\Kconst_1 + \Kconst_2\norm{\hat \Gamma}_{\Fro}^2)\eta}{\hessianconst}\\
    &+\frac{3}{\hessianconst^2}\p{\highsmooth_2^2\gnorm{\hat g - g_0}^4 + 4\secsmooth_2^2\norm{\pdirhat - \pdir_0}_\Fro^2\cdot\gnorm{\hat g - g_0}^2}.
\end{align*}

\end{proof}

\clearpage

\section{Detailed Discussion}\label{appx:discussion}

This section provides details on comparisons and remarks following the statements of the main results in the main text, and details on the discussions summarized in \Cref{sec:discussion} from the main text. In \Cref{appx discuss sec 1}, we compare our results to generic state-of-the-art results on biased SGD, that is SGD with errors in the stochastic gradients. In \Cref{appx discuss sec 2}, we discuss different orthogonalization method in orthogonal statistical learning. In \Cref{appx discuss sec 3}, we discuss how to interleave the target and nuisance estimation. In \Cref{appx discuss sec 4}, we describe the connection between our orthogonalized gradient to variance reduction method in the Monte Carlo estimation literature. In \Cref{appx discuss sec 55}, we discuss the double robustness of SGD for dose-response estimation.

\subsection{Comparison to Biased SGD}\label{appx discuss sec 1}

There are several ways to think about the bias induced by using an imperfect estimate $\hg$ as opposed to the true nuisance $g_0 \in \G$. For the sake of discussion, we will define $L(\cdot, g_0)$ and $L(\cdot, \hg)$ as the ``original objective'' and ``shifted objective'', respectively. Accordingly, we will call $\theta_\star$ the ``original minimizer'' and denote by
\begin{align}
    \hat \theta_\star = \argmin_{\theta \in \Theta} \ploss(\theta, \hat g).
\end{align}
the ``shifted minimizer''. The bias can be measured in terms of (i) the error $\norm{\hat \theta_\star - \theta_\star}_2^2$ between the original and shifted minimizers, (ii) the uniform error $\sup_{\theta}\abs{L(\theta, g_0) - L(\theta, \hg)}$ between the original and shifted objectives, and (iii) some summary of the gradient bias 
\begin{align}
    \Ex{Z_t \sim \prob}{S(\theta\pow{t-1}, \hg; Z_t)} - \grad_\theta L(\theta\pow{t}, g_0),\label{eq:a:discussion:gradient_bias}
\end{align}
of the oracle $S$ (a vector-valued quantity) for step $t = 1, \ldots, n$ of the algorithm. Whether one appeals to (i) or (ii) depends on whether the convergence guarantees are stated in terms of iterate convergence or function value convergence; because we analyze convergence of iterates, our discussion will cover (i) and (iii).

On (i), one applies the decomposition
\begin{align*}
    \norm{\theta\pow{n} - \theta_\star}_2^2 \leq 2\norm{\theta\pow{n} - \hat\theta_\star}_2^2 + 2\norm{\hat \theta_\star - \theta_\star}_2^2,
\end{align*}
and plugs an analysis of unbiased SGD from the current literature for the $\norm{\theta\pow{n} - \hat\theta_\star}_2^2$ term. The purpose of this substitution is to check how our theoretical results align with the known results on \emph{unbiased SGD}. 

\citet[Thm.~1]{moulines2011non} show that for constant learning rate $\eta = \bigO(\mu/\smooth^{2})$, the iterate $\sgd\pow{n}$ satisfies
\begin{align}\label{eq moulines}
    \E_{\mc{D}_{n} \sim \prob^n}[\Vert\sgd\pow{n} - \theta_\star\Vert_2^2]   \lesssim& \exp\p{-\frac{\sconv\eta n}{2}}\Vert\theta\pow{0} - \theta_\star\Vert_2^2 + \frac{\Kconst_1 \eta}{\sconv} + \norm{\hat \theta_\star - \theta_\star}_2^2.
\end{align}
\citet[Thm.~3]{cutler2023stochastic} demonstrate that with the learning rate $\eta = \bigO(1/\smooth)$, the iterates would satisfy the following bound:
\begin{align}\label{eq cutler}
    \E_{\mc{D}_{n} \sim \prob^n}[\Vert\sgd\pow{n} - \theta_\star\Vert_2^2]   \lesssim& \p{1 - \frac{\sconv\eta}{2}}^n\Vert\theta\pow{0} - \theta_\star\Vert_2^2 + \frac{\Kconst_1 \eta}{\sconv} + \norm{\hat \theta_\star - \theta_\star}_2^2
\end{align}
In addition, \citet[Thm.~6]{cutler2023stochastic} provide the high probability bound of $\sgd\pow{n}$ that for $\eta = \bigO(1/\smooth)$, with probability at least $1-\delta$, 
\begin{align}\label{eq cutler high prob}
    \Vert\sgd\pow{n} - \theta_\star\Vert_2^2   \lesssim& \p{1 - \frac{\sconv\eta}{2}}^n\Vert\theta\pow{0} - \theta_\star\Vert_2^2 + \frac{\Kconst_1 \eta}{\sconv} \log\p{\frac{e}{\delta}} + \norm{\hat \theta_\star - \theta_\star}_2^2
\end{align}
All of these bounds essentially agree, as we may apply $(1 - \sconv\eta/2) \leq \exp(-\sconv\eta/2)$. In comparison to our \Cref{theorem convergence rate baseline SGD}, our bias term is stated directly in terms of the nuisance error $\norm{\hg - g_0}_\calG^2$. This can be viewed as a refinement of the less transparent bias measurement $\norm{\hat \theta_\star - \theta_\star}_2^2$.
Moreover, although~\eqref{eq moulines}--\eqref{eq cutler high prob} are of the same order as our results in \Cref{theorem convergence rate baseline SGD} when the true nuisance $g_0$ is available, all of the three bounds above require $\kappaconst_1 = 0$ in \Cref{assumption of population risk}(d). In this case, provide~\eqref{eq cutler} and~\eqref{eq cutler high prob} use a learning rate of the order $\bigO(1/M)$ (whereas the learning rate of~\eqref{eq moulines} encounters an additional condition number $\smooth/\sconv$). Our learning rate recovers $\bigO(1/M)$ when $\kappaconst_1 = 0$, and adapts via the setting $\eta = \bigO(\sconv/(M\sconv + \kappaconst_1))$ when $\kappaconst_1 > 0$. Finally, the high probability bound~\eqref{eq cutler high prob} requires a stronger assumption in the sense that $\score(\theta, g; Z) - \score(\theta, g)$ is sub-Gaussian with uniform parameter $\Kconst_1/2$ for all $\theta \in \Theta$ and $g \in \Gr$.

Returning to the bias in the stochastic gradient oracle~\eqref{eq:a:discussion:gradient_bias}, this case is handled quite generally in \citet{demidovich2023aguide}. Their ``ABC assumption'' considers constants $A, B, C, b, c \geq 0$ such that the inequalities
\begin{align}
    \ip{\grad L(\theta, g_0), \Ex{Z \sim \prob}{S_\theta(\theta, \hg; Z)}} &\geq b \norm{\grad L(\theta, g_0)}_2^2 + c \label{eq:a:abc1}\\
    \E_{Z \sim \prob}\norm{S_\theta(\theta, \hg; Z)}_2^2 &\leq 2A \p{L(\theta, g_0) - L(\theta_\star, g_0)} +  B \norm{\grad L(\theta, g_0)}_2^2 + C \label{eq:a:abc2}\\
    A + \smooth\p{B + 1 - 2b} &< \frac{\sconv}{2} \label{eq:a:abc3}
\end{align}
hold for all $\theta \in \R^d$ (where the expectations are conditional on any randomness in $\hg$).\footnote{The third inequality is actually $A + \smooth\p{B + 1 - 2b} < \sconv$, but the constant $(1/2)$ to make the resulting bound more comparable, in that their bound can only improve over ours for the stronger inequality.} The bias is really captured in the first of the three inequalities, whereas the third inequality places conditions on the parameters of the problem that are not in the hands of the algorithm user. By strong convexity, our \Cref{assumption of population risk}(d) satisfies~\eqref{eq:a:abc2} with $A = \kappaconst_1/\mu$, $B = 0$, $C = \Kconst_1$.
The resulting convergence guarantee \citep[Thm.~5]{demidovich2023aguide} gives
\begin{align*}
    \E_{\mc{D}_{n} \sim \prob^n}[\Vert\sgd\pow{n} - \theta_\star\Vert_2^2]   &\lesssim \p{1 - \frac{\sconv\eta}{2}}^n\Vert\theta\pow{0} - \theta_\star\Vert_2^2 + \frac{\eta \Kconst_1}{\sconv} + \frac{\Kconst_1 + c}{\sconv^2}.
\end{align*}
for a learning rate set as
\begin{align*}
    \eta \leq \min\br{\frac{2}{\sconv}, \frac{\sconv - 2(\kappaconst_1/\sconv + \smooth(1-2b))}{\kappaconst_1}}.
\end{align*}
We would like to check how our theoretical results compare with the above application of generic results on \emph{biased SGD}. 
In our setting, the condition~\eqref{eq:a:abc3} reads $\kappaconst_1/\sconv + \smooth(1 - 2b) < \sconv/2$. We assume neither this nor~\eqref{eq:a:abc1}, and in our \Cref{theorem convergence rate baseline SGD}, replace the $\frac{\Kconst_1 + c}{\sconv^2}$ term with a bias term that depends directly on the nuisance error $\norm{\hg - g_0}_\calG^2$, either in the nuisance sensitive regime, or in the nuisance insensitive regime.

\subsection{Discussion of Full-sample Orthogonal Statistical Learning and Related Methods}\label{appx discuss sec 2}

\myparagraph{Comparison of Orthogonalizing Operators}
Constructing orthogonal losses or scores has been widely studied in semiparametric inference, hypothesis testing, and machine learning. In semiparametric statistics, such constructions often rely on the efficient influence function, which characterizes the asymptotic efficient estimation bound; see \citet[Ch. 3]{bickel1993efficient}, \citet[Ch.~4]{tsiatis2006semiparametric}, \citet[Ch.~25]{, van2000asymptotic}, \citet{luedtke2024one}. In hypothesis testing, orthogonal scores were used by \citet{neyman1959probability, neyman1979c} and \citet{ferguson2014mathematical} to guarantee the local unbiasedness of specific tests based on the likelihood with finite-dimensional nuisance. In machine learning, the construction of orthogonal scores was latter extended to non-likelihood losses in \citet{wooldridge1991specification} and \citet{liu2022orthogonal}, which aligns with our construction limited to the finite-dimensional nuisance case. Recent work of orthogonalization in machine learning with infinite-dimensional nuisance relies on the approach named \textit{concentrating-out} \citep{newey1994asymptotic,chernozhukov2018double}. However, although all these constructions produce Neyman orthogonal losses or scores, none of them consider the stochastic design. Our work is complementary to these, providing non-asymptotic guarantees for stochastic optimization.

Although these constructions might lead to different orthogonal scores, they can be the same at both the target and the true nuisance. Specifically, when $\ell$ is the negative log-likelihood and $\G = \R^k$, the concentrating-out approach and our NO gradient oracle $\noscore$ both produce the efficient score in the semiparametric theory literature; see \citet[Page 1359]{newey1994asymptotic}, \citet[Ch. 25.4]{van2000asymptotic}, and \citet[Def.~8]{tsiatis2006semiparametric}. This identity can happen for infinite-dimensional nuisances as well. As an example, consider the partial linear model from \Cref{subsubsec: non-ortho plm}, where the non-orthogonal loss is defined as
\begin{align*}
    \ell(\theta, g; z) = \frac{1}{2}(y - g(w) - \ip{\theta, x})^2.
\end{align*}
\citet[Sec.~2.2.2]{chernozhukov2018double} showed that the concentrating-out approach would produce an orthogonal score under the concentrated-out nuisance $\varphi_0(\theta) = Z \mapsto \E_{\prob}[Y - \ip{\theta, X} \mid W]$ as
\begin{align*}
    S(\theta_\star, \varphi_0(\theta_\star); Z) 
    &= -(X - \Ex{\prob}{X\mid W})(Y - \Ex{\prob}{Y\mid W} - \ip{\theta_\star, X - \Ex{\prob}{X\mid W}}).
\end{align*}
On the other hand, it is easy to verify that $H_{gg} = \mathbf{I}$, $H_{\theta g} = \E_{\prob}[X \mid W]$, and $\pdir_0 \nabla_g \ell(\theta, g; z) = \D_g \ell(\theta,g;z)[H_{gg}^{-1}H_{\theta g}]$, which implies that our orthogonal gradient oracle $\noscore$ in \eqref{no score of infinite dim} has the same form under the target $\theta_\star$ and true nuisance $g_0(W) = \E_{\prob}[Y - \ip{\theta_\star, X} \mid W]$:
\begin{align*}
    \noscore(\theta_\star, g_0; Z) 
    &= -(X - \Ex{\prob}{X\mid W})(Y - \Ex{\prob}{Y\mid W} - \ip{\theta_\star, X - \Ex{\prob}{X\mid W}}).
\end{align*}

\myparagraph{Comparison with Debiased Machine Learning} 
In machine learning, debiasing typically refers to reducing the impact of model selection error on the parameter or quantity of interest. In particular, mitigating the bias introduced by nuisance estimation is known as \textit{debiased machine learning} (DML), which has been recently studied by \citet{van2011targeted,shi2019adapting, chernozhukov2021automatic, vanderlaan2025automatic}. Some of the calculations used by DML estimators have been shown to be amenable to computerization, simplifying their construction \citep{carone2019toward,ichimura2022influence,luedtke2024simplifying}. Statistical learning methods that use debiasing are also called \textit{orthogonal statistical learning} (OSL) and have been studied in \cite{foster2023orthogonal, liu2022orthogonal, zadik2018orthogonal}. While the earlier studies focus on the empirical risk minimization, our paper provide a stochastic approximation method in DML/OSL and establish the convergence rate of the debiased estimation.

To strengthen the debiasing effect, one possible approach is to consider the higher-order Neyman orthogonality. If the loss function satisfies the $k$-th order orthogonality at $(\theta_\star, g_0)$, \citet[Cor. 4]{zadik2018orthogonal} show that we only need the nuisance estimator to converge at rate $\bigO_p(n^{-\frac{1}{2(k+1)}})$ to have the nuisance effect in the order of $\bigO_p(n^{-1})$, which aligns with the nuisance insensitive rate in \Cref{theorem convergence rate baseline SGD}, where $k=1$ and the nuisance effect $\gnorm{\hat g - g_0}^4 = \bigO_p(n^{-1})$ when $\gnorm{\hat g - g_0} = \bigO_p(n^{-1/4})$. Similar improvements in sensitivity to nuisance estimation rates have been developed previously using higher-order influence functions \citep{pfanzagl1985asymptotic,robins2008higher}.

For a range of problems, debiasing methods often lead to cross-product estimations consisting of two nuisance estimators \citep{rotnitzky2021characterization}. Such remainders frequently result from orthogonalization procedures used in missing data problems and causal inference problems \citep{robins1994estimation,robins1995semiparametric,laan2003unified}. \citet{chernozhukov2021automatic} consider cases where $Z = (W,Y)$, $g_0(W)= \Ex{\prob}{Y \mid W}$, and the target can be written as the averaged moment of the form
\begin{align*}
    \theta_\star = \E_{\prob}[m(g_0; W)],
\end{align*}
where $\E[m(g; W)]: \G \times \calW \mapsto \R^d$ is a continuous linear functional of $g: \calW \mapsto \R$. By Riesz representation theorem, there uniquely exists a Riesz representer (RR) $g_{0}^{\text{RR}} \in \G$ such that $\E_{\prob}[m(g; W)] = \E_{\prob}[g_{0}^{\text{RR}}(W) g(W)]$. Then the debiased score for estimating $\theta_\star$ is defined as
\begin{align*}
    S(\theta, g;Z) = m(g;W) + g_{0}^{\text{RR}}(W)(Y - g(W)) - \theta.
\end{align*}
The debiasing effect on the nuisance turns out depending on the cross-product $\gnorm{\hat g^{\text{RR}} - g_{0}^{\text{RR}}}\cdot\gnorm{\hat g - g_0}$. Specifically, \citet[Asm.~4]{chernozhukov2021automatic} requires the cross product to be in the order $\bigO_p(n^{-1/2})$ to construct $\bigO_p(n^{-1})$ consistent target estimator. This aligns with the cross-product $\gnorm{\hat g - g_0}\cdot \norm{\pdirhat - \pdir_0}_\Fro$ in \Cref{theorem debiased SGD convergence rate} where the same requirement needs to be satisfied to obtain a $\bigO_p(n^{-1/2})$ consistent estimator. However, \Cref{theorem debiased SGD convergence rate} also has a second, non-cross-product remainder $\Vert\hat g - g_0\Vert_\G^4$ that will only be small if $\hat g$ approximates $g_0$, making it so that our consistency guarantee is robust to misspecification of $\hat \Gamma$, but not to misspecification of $\hat g$.

\subsection{Discussion of Interleaving Target and Nuisance Estimation}\label{appx discuss sec 3}

To propose the interleaving approach, we consider the case where we learn the nuisance from the $\calW$-valued data $W = (U,V)$ from a probability measure $\Q$. We assume that the true nuisance $g_0$ satisfies $g_0: \mathcal{U} \mapsto \R$ and is the minimizer of the mean squared error over $\G$:
\begin{align*}
    g_0 = \argmin_{g \in \G}\Ex{\Q}{(g(U) - V)^2}.
\end{align*}

Suppose that we observe another data stream $W_1, \ldots, W_m$ sampled i.i.d.~from $\Q$, and that $\calS_m = \{W_1, \ldots, W_m\}$ is independent of the parameter stream $\calD_n$. We define the sigma algebra $\calH_{m} = \sigma(\calS_m), m \geq 1$ as the nuisance filtration and the sigma algebra $\calF_{m,t} = \sigma(\calS_m \cup \calD_{(m-1)n+t}), 0\leq t \leq n$ as the parameter filtration. We assume that there are two stochastic processes $\hat g^{(m)}, m \geq 1$ adapted to $\calH_{m}$ and $\theta^{(m,t)}, 0\leq t \leq n$ adapted to $\calF_{m,t}$, to which we refer as the nuisance estimator and the parameter estimator, respectively. Intuitively, this means that the nuisance estimator $\hat g^{(m)}$ can be updated now instead of being the fixed $\hat g$, and the parameter estimator $\theta^{(m,t)}$ can be updated $n$ times between every two nuisance updates. Specifically, we use SGD as the parameter estimator. We define $\sgd\pow{0, n} = \sgd\pow{0} \in \Theta$ and $\sgd\pow{i, 0} = \sgd\pow{i-1, n}$ for $1 \leq i \leq m$, and produce the sequence $\sgd\pow{i, 1}, \ldots, \sgd\pow{i, n}$ using $n$ steps of the SGD update~\eqref{baseline SGD procedure} initialized at $\sgd\pow{i, 0}$. 

\myparagraph{Under Non-orthogonality}
Consider the case that $\G$ is a reproducing kernel Hilbert space (RKHS) with kernel $k(\cdot,\cdot)$. To obtain a sequence of nuisance estimator ${\hat g^{(m)}}$ on $\calH_{m}$, one possible approach is to adopt the non-parametric stochastic approximation.  With the assumption that the eigenvalues $(\lambda_j)_{j\geq1}$ of covariance operator $\E_{\Q}[k(W,\cdot) \otimes k(W,\cdot)]$ decay polynomially at order $j^{-\alpha}$, \citet[Cor.~3]{dieuleveut2016nonparametric} suggests that the non-parametric stochastic approximation $\hat g^{(m)}$ satisfies for some $C>0$,
\begin{align}\label{eq: dieuleveut}
    \xi_m := \Ex{\calS_m \sim \Q^m}{\gnorm{\hat g^{(m)} - g_0}^2} \leq Cm^{-\frac{2\alpha - 1}{2\alpha}}.
\end{align}

This leads to the following nuisance sensitive rate for non-Neyman orthogonal losses.
\begin{proposition}\label{prop. adaptive}
Suppose that $\hat g^{(m)}$ satisfies \eqref{eq: dieuleveut} and that $\hat g^{(m)}\in \Gr$ and $\theta^{(m,t)} \in \Theta$ almost surely for all $m\geq 1$ and $0\leq t\leq n$. Under the same conditions to \Cref{theorem convergence rate baseline SGD}, it holds that
    \begin{align*}
    \E_{\calD_{mn}\cup\calS_m \sim \prob^{mn}\otimes \Q^m}[\Vert\sgd^{(m,n)} -& \theta_\star\Vert_2^2] \lesssim \p{1-\frac{\sconv\eta}{2}}^{mn}\Vert\sgd^{(0)} - \theta_\star\Vert_2^2 \\
    &+m\exp\p{-\frac{\sconv\eta nm}{4}} + (m^{-\frac{2\alpha-1}{2\alpha}} + \eta)((\eta n)^{-1} + 1).
\end{align*}
In addition, when $(\eta n)^{-1} = \bigO(1)$, it holds that
\begin{align*}
    \E_{\calD_{mn}\cup\calS_m \sim \prob^{mn}\otimes \Q^m}[\Vert\sgd^{(m,n)} - \theta_\star\Vert_2^2] \lesssim \p{1-\frac{\sconv\eta}{2}}^{mn}\Vert\sgd^{(0)} - \theta_\star\Vert_2^2 + m^{-\frac{2\alpha-1}{2\alpha}} + n^{-1} + \eta.
\end{align*}
\end{proposition}
The proof is provided in \Cref{subsec: proof of prop adaptive}. \Cref{prop. adaptive} demonstrates that interleaving the target and nuisance estimation allows $\eta \asymp n^{-1}$ since the nuisance update iterations guarantees the shrinking of the term $\p{1-\sconv\eta/2}^{mn}$ in this case. This is an improvement to \Cref{theorem convergence rate baseline SGD} where $\eta$ should satisfy $(\eta n)^{-1} = o(1)$ to ensure $(1-\sconv\eta/2)^n$ shrinking to zero.

\myparagraph{Under Orthogonalized SGD}
To establish a similar probability bound for OSGD, we assume that the orthogonalizing operator $\pdir_0$ can be written as the minimizer of the following program:
\begin{align*}
    \pdir_0 = \argmin_{\pdir \in \G_*^d} \Ex{\prob}{\norm{S_\theta(\theta_\star, g_0; Z) - \pdir \nabla_g \loss(\theta_\star, g_0; Z)}_2^2},
\end{align*}
where $\G_*$ is the dual space of $\G$. When $d$ is fixed, we assume that $\pdir_0$ can be estimated (coordinate-wisely) from the data stream $\calS_m$ using the stochastic approximation of \citet{dieuleveut2016nonparametric}, which leads to a sequence of operator estimators $\pdirhat^{(m)}, m\geq 1$. For any $s>0$, we define the following events for $i=0,1,\dots, m$,
    \begin{align*}
        \calA_{i}(s) = \br{\gnorm{\hat g\pow{i} - g_0}^2 \leq Cs^{-1}i^{-\frac{2\alpha - 1}{2\alpha}}} \text{ and } \calB_{i}(s) = \br{\norm{\pdirhat^{(i)} - \pdir_0}_\Fro^2 \leq Cs^{-1}i^{-\frac{2\alpha - 1}{2\alpha}}}.
    \end{align*}
We assume that for some constant $c\geq 1$ the nuisance estimator $\hat g\pow{i}$ satisfies 
\begin{align}\label{markov 1}
    \Ex{\calS_i}{\gnorm{\hat g\pow{i} - g_0}^{2c} \mid \calA_{i-1}(s^{1/c}),\dots,\calA_{1}(s^{1/c})} \leq C^ci^{-\frac{(2\alpha - 1)c}{2\alpha}}.
\end{align}
Additionally, we assume that $\pdirhat^{(i)}$ decays in the same rate such that
\begin{align}\label{markov 2}
    \Ex{\calS_i }{\norm{\pdirhat^{(i)} - \pdir_0}_\Fro^{2c}\mid \calB_{i-1}(s^{1/c}),\dots,\calB_{1}(s^{1/c})} \leq  C^ci^{-\frac{(2\alpha - 1)c}{2\alpha}}.
\end{align}
With all the assumptions above, it is possible to provide a convergence bound of $\norm{\sgd^{(m,n)}-\theta_\star}_2^2$ in probability. The following proposition shows that estimations from $\calS_m$ using OSGD contribute to a nuisance insensitive rate of $\bigO(m^{-\frac{2\alpha-1}{\alpha}})$, compared to the nuisance sensitive rate $\bigO(m^{-\frac{2\alpha-1}{2\alpha}})$ in \Cref{prop. adaptive} for non-Neyman orthogonal losses.
\begin{proposition}\label{prop. prob bound}
Suppose that $\{\hat g^{(m)}, m \geq 1\}$ satisfies \eqref{markov 1}, and that $\{\pdirhat^{(m)}, m \geq 1\}$ satisfies \eqref{markov 2}. Assume that $\theta^{(m,t)} \in \Theta$ almost surely for all $m\geq 1$ and $0\leq t\leq n$. For any $s \geq 0$, define $\delta(s)=\bigO(ms)$ as \eqref{eq delta}. Under the same conditions to \Cref{theorem convergence rate baseline SGD}, with probability at least $1-\delta(s)$, it holds that
\begin{align*}
   \Vert\sgd^{(m,n)} -& \theta_\star\Vert_2^2 \lesssim s^{-1}\p{1-\frac{\sconv\eta}{2}}^{mn}\Vert\sgd^{(0)} - \theta_\star\Vert_2^2 \\
    &+s^{-1}\p{m\exp\p{-\frac{\sconv\eta nm}{4}} + (s^{-2/c}m^{-\frac{2\alpha-1}{\alpha}} + \eta)((\eta n)^{-1} + 1)}.
\end{align*}
In addition, when $(\eta n)^{-1} = \bigO(1)$, with probability at least $1-\delta(s)$, it holds that
\begin{align*}
   \Vert\sgd^{(m,n)} -& \theta_\star\Vert_2^2 \lesssim s^{-1}\p{1-\frac{\sconv\eta}{2}}^{mn}\Vert\sgd^{(0)} - \theta_\star\Vert_2^2 +s^{-1}\p{s^{-2/c}m^{-\frac{2\alpha-1}{\alpha}} + n^{-1} + \eta}.
\end{align*}
\end{proposition}
We refer the reader to \Cref{subsec: proof of prop prob bound} for the proof.

\subsection{Interpretation as Control Variate for Variance Reduction}\label{appx discuss sec 4}

The regression equation~\eqref{eq:regression}, which provides an alternate characterization of the orthogonalized stochastic gradient oracle in the case of negative log-likelihood losses, yields an interesting connection to the Monte Carlo estimation literature. Variance reduction techniques (or ``swindles'') are used in problems such as estimating the mean or variance of a statistic via Monte Carlo simulation. Consider a probability space $(\Omega, \msc{F}, \prob)$ with expectation denoted by $\E$ and an unknown vector-valued target $v \in \R^d$. We have $\hat{v}: \Omega \rightarrow \R^d$, where we interpret $\hat{v}$ as a (not necessarily unbiased) sample estimate of $v$. Several variance reduction techniques fall into the category of \emph{control variates} \citep{Graham2013Stochastic}, where a random variable $\hat{u}: \Omega \rightarrow \R^k$ with known expectations $u = \E[\hat{u}]$ and a matrix $\pdir \in \R^{d \times k}$ are used in the variance-reduced estimator
\begin{align*}
    \tilde{v} = \hat{v} - \pdir(\hat{u} - u).
\end{align*}
A mean squared error decomposition yields the identity
\begin{align*}
    \E\norm{\tilde{v} - v}_2^2 &= \E\norm{\hat{v} - v}_2^2 - 2\E\ip{\hat{v} - v, \pdir(\hat{u} - u)} + \E\norm{\pdir(\hat{u} - u)}_2^2\\
    &= \E\norm{\hat{v} - v}_2^2 - 2\E\ip{\hat{v} - v, \pdir(\hat{u} - u)} + o(\norm{\pdir}_{\op}),
\end{align*}
indicating that for sufficiently ``small'' $\pdir$, $\tilde{v}$ provides an improved estimator if $\hat{v} - v$ and $\hat{u} - u$ have high (multiple) correlation. While in the Monte Carlo literature, $\hat{u}$ and $\pdir$ can be chosen optimally provided knowledge of the underlying data-generating mechanism, as $\pdir$ can be interpreted as the regression function of $\hat{v} - v$ on $\hat{u} - u$.\footnote{In the Monte Carlo settings, it often holds that $d = k$ and $\pdir = \alpha I$ for some constant $\alpha \in \R$. Then, $\E\ip{\hat{v} - v, \pdir(\hat{u} - u)}$ can be replaced by $\alpha \Tr(\Cov(\hat{v}, \hat{u}))$ and  $o(\norm{\pdir}_{\op})$ can be replaced by $o(\alpha)$.} Outside of Monte Carlo simulation, this procedure can be applied more widely if the user chooses $\hat{u}$ and $\hat{\pdir}$ based on intuition or limiting arguments.

In the stochastic optimization setting, $v$ represents the true gradient of the objective at a particular parameter, while $\hat{v}$ represents a stochastic gradient estimate from an oracle. Variance reduction techniques have previously been applied in an incremental setting, in which a fixed data set of size $n$ is provided at initialization, and the algorithm may only make multiple passes through this same data set \citep{Gower2020Variance}. Note that this differs from our fully stochastic setting, in which we receive a fresh sample $Z_t$ on each iterate $t = 1, \ldots, n$. For negative log-likelihood losses, our orthogonalized oracle can be viewed in a similar light to control variate-based variance reduction methods (although in an infinite-dimensional setting). To summarize, we have from~\eqref{eq:regression} that
\begin{align*}
    v &= \score(\theta_\star, g_0)\\
    \hat{v} &= S_\theta(\theta_\star, \hat g)\\
    u &= 0 &\text{ (by \Cref{asm:noscore})}\\
    \hat{u} &= \grad_g \ell(\theta_\star, g_0; Z)\\
    \tilde{v} &= S_\theta(\theta_\star, g_0; Z) - \pdir_0 \nabla_g \loss(\theta_\star, g_0; Z),\\
\intertext{using the idealized parameters. Using the approximations for $\theta \neq \theta_\star$, we have}
    v &= \score(\theta, g_0)\\
    \hat{v} &= S_\theta(\theta, \hg)\\
    u &\approx 0 &\text{ (for $\theta \approx \theta_\star$)}\\
    \hat{u} &= \grad_g \ell(\theta, \hg; Z)\\
    \tilde{v} &= S_\theta(\theta, \hg; Z) - \pdirhat \nabla_g \loss(\theta, \hg; Z).
\end{align*}
In fact, using the derivative of the log likelihood in a control variate procedure has been explored in the simulation literature as early as \citet{johnstone1985efficient}, as the correlation between a statistic and the score function has tight connections to the Cram\'{e}r-Rao lower variance bound and exponential families. We emphasize, however, that our method does not require the loss to be of negative log-likelihood form nor any specific distributional knowledge to be applied.

\subsection{Discussion of Double Robustness}\label{appx discuss sec 55}
We now study the double robustness of SGD for dose-response estimation as discussed in \cite{bonvini2022fastconvergenceratesdoseresponse}. Consider estimating the effect of the continuous treatment $A \in \calA \subset \R$ on the outcome $Y \in \calY \subset \R$, which is defined as $\E{Y(a)}$ (known as the dose-response function, DRF) under the potential outcomes framework. Under standard assumptions, the DRF takes the form
\begin{align*}
    \theta_0(t) = \E\sbr{\E[Y \mid A = t, X]} = \int \E[Y \mid A = t, X=x] \d \prob(x),
\end{align*}
where $X \in \calX \subset \R^d$ is the measured confounders. Let $Z = (Y, A, X) \sim \prob$ with density $p$. We take the following notations:
\begin{align*}
    p(u) = \frac{\d}{\d u}\prob(U \leq u), \pi(a \mid x) = \frac{p(a, x)}{p(x)}, \mu(a,x) = \E[Y \mid A = a, X = x], w(a, x) = \frac{p(a)}{\pi(a\mid x)}.
\end{align*}
We can rewrite $\theta_0(t)$ equivalently as
\begin{align*}
    \theta_0(t) = \Ex{}{\mu(t, X)} = \Ex{}{w(t,X)Y\mid A = t}.
\end{align*}
We also take the notations $\prob(g(Z)) = \int g(z) \d \prob(z)$, $\prob_n(g(Z)) = n^{-1}\sum_{i=1}^n g(Z_i)$, $\norm{g}_{L_2(\prob)} = [\prob(g^2(Z))]^{1/2}$ to denote the $L_2(\prob)$ norm, and $\norm{g}_{L_4(\prob)} = [\prob(g^4(Z))]^{1/4}$ to denote the $L_4(\prob)$ norm. We now establish the procedure to estimate $\theta_0(t)$ as Algorithm 1 in \cite{bonvini2022fastconvergenceratesdoseresponse} with slightly modification to apply SGD:
\begin{enumerate}
\item Observe \iid samples $\{Z_i'\}_{i=1}^m$ for the nuisance estimation and \iid samples $\{Z_i\}_{i=1}^n$ for the parameter estimation.
    \item Estimate $\mu$, $w$, and $m(a) = \prob\mu(a, \cdot)$ using $\{Z_i'\}_{i=1}^m$ with $\hat \mu$, $\hat w$, and $\hat m(a) = \prob_n(\hat \mu(a, \cdot))$, respectively.
    \item Construct the pseudo-outcome
    \begin{align*}
        \hat \varphi(Z) = \hat w(A,X){Y - \hat \mu(A,X)} + \hat m(A).
    \end{align*}
    We also define the true nuisance as 
    \begin{align*}
        \varphi_0(Z) = w(A,X){Y - \mu(A,X)} + \int \mu(A, x) \d\prob(x).
    \end{align*}
    \item Define the loss function via a parametric function class $\calF_\Theta = \{f_\theta: \calA \mapsto \R \mid  \theta \in \Theta \subset \R^d\}$ as
    \begin{align}\label{db loss}
        \ell(\theta, \varphi; z) = \frac{1}{2}(f_\theta(a) - \varphi(z))^2.
    \end{align}
    Define the stochastic gradient oracle as
    \begin{align*}
        \score(\theta, \varphi; z) = (f_\theta(a) - \varphi(z))\nabla_\theta f_\theta(a).
    \end{align*}
    \item Solve the optimization problem
    \begin{align*}
        \theta_\star = \argmin_{\theta \in \Theta}\E\sbr{\ell(\theta, \varphi_0; Z)}
    \end{align*}
    using SGD with the stochastic gradient $\score(\theta, \hat \varphi; Z)$ by
    \begin{align}\label{db sgd}
        \theta\pow{n} = \theta\pow{n-1} - \eta \score(\theta\pow{n-1}, \hat \varphi; Z_{n-1}), \quad \theta\pow{0} \in \Theta.
    \end{align}
\end{enumerate}
As demonstrated in \cite{bonvini2022fastconvergenceratesdoseresponse}, this procedure would yield a doubly robust ERM estimator. In the following proposition, we claim that double robustness would be preserved if the SGD estimator is adopted instead.
\begin{proposition}\label{db prop}
    Assume that $\Ex{}{\norm{\nabla_\theta f_{\theta_\star}(A)}_2^2}^{1/2} \leq C_A$. Suppose that \Cref{assumption of population risk} holds and $\sgd\pow{0}, \ldots, \sgd\pow{n} \in \Theta$ almost surely for $\theta\pow{n}$ in \eqref{db sgd}. If $\eta \leq \sconv/2(M\sconv + \kappaconst_1)$, the iterates of~\eqref{db sgd} satisfy
        \begin{align*}
             \E_{\mc{D}_{n} \sim \prob^n}[\Vert\sgd\pow{n} - \theta_\star\Vert_2^2]  \lesssim& \p{1-\frac{\sconv \eta}{2}}^n + \eta\\
             &+ \norm{w - \hat w}_{L_4(\prob)}\norm{\mu - \hat \mu}_{L_4(\prob)} + \max_{a \in \calA} \abs{(\prob_n-\prob)\br{\hat \mu(a, X)}}^2.
        \end{align*}
\end{proposition}
\Cref{db prop} follows directly from the following two lemmas, \Cref{db lem 1} and \Cref{db lem 2}, and the nuisance sensitive rate in \Cref{theorem convergence rate baseline SGD}. Whenever the empirical estimation $\max_{a \in \calA} \abs{(\prob_n-\prob)\br{\hat \mu(a, X)}}^2$ shrinks, \Cref{db prop} suggests that the $\sgd\pow{n}$ would converge to the the target parameter when either $\hat w$ or $\hat\mu$ is correctly specified.
\begin{lemma}\label{db lem 1}
Assume that $\Ex{}{\norm{\nabla_\theta f_{\theta_\star}(A)}_2^2}^{1/2} \leq C_A$. Then for the loss defined in \eqref{db loss}, we have
    \begin{align*}
        \abs{\D_\varphi \D_\theta L(\theta_\star, \bar \varphi)[\theta - \theta_\star, \hat\varphi - \varphi_0]} \leq C_A\Vert \hat \varphi - \varphi_0\Vert_\calG\norm{\theta - \theta_\star}_2,
    \end{align*}
    where $\Vert \hat \varphi - \varphi_0\Vert_\calG = \E[\E[\hat \varphi(Z) - \varphi_0(Z) \mid A]^2]^{1/2} = \E[(\E[\hat \varphi(Z) \mid A] - \theta_0(A))^2]^{1/2}$.
\end{lemma}
\begin{proof}
    Let $\hat r(t) = \Ex{}{\hat \varphi(Z) \mid A = t} - \theta_0(t)$.
Note that 
\begin{align*}
    \D_\varphi \D_\theta L(\theta_\star, \bar \varphi)[\theta - \theta_\star, \hat\varphi - \varphi_0] &= -\Ex{}{(\hat\varphi(Z) - \varphi_0(Z))\ip{\nabla_\theta f_{\theta_\star}(A), \theta - \theta_\star}}\\
    &=-\Ex{}{(\Ex{}{\hat\varphi(Z) \mid A} - \Ex{}{\varphi_0(Z) \mid A})\ip{\nabla_\theta f_{\theta_\star}(A), \theta - \theta_\star}}\\
    &=-\Ex{}{(\Ex{}{\hat\varphi(Z) \mid A} - \theta_0(A))\ip{\nabla_\theta f_{\theta_\star}(A), \theta - \theta_\star}}\\
    &= -\Ex{}{\hat r(A)\ip{\nabla_\theta f_{\theta_\star}(A), \theta - \theta_\star}}.
\end{align*}
Thus, by the assumption that $\Ex{}{\norm{\nabla_\theta f_{\theta_\star}(A)}_2^2}^{1/2} \leq C_A$, 
\begin{align*}
    \abs{\D_\varphi \D_\theta L(\theta_\star, \bar \varphi)[\theta - \theta_\star, \hat\varphi - \varphi_0]} &\leq \Ex{}{\abs{\hat r(A)}^2}^{1/2}\Ex{}{\norm{\nabla_\theta f_{\theta_\star}(A)}_2^2}^{1/2}\norm{\theta - \theta_\star}_2\\
    &\leq C_A\Ex{}{\abs{\hat r(A)}^2}^{1/2}\norm{\theta - \theta_\star}_2.
\end{align*}
\end{proof}

\begin{lemma}\label{db lem 2}
For the norm defined in \Cref{db lem 1}, we have
    \begin{align*}
        \gnorm{\hat\varphi - \varphi_0} \leq \norm{w - \hat w}_{L_4(\prob)}^{1/2}\norm{\mu - \hat \mu}_{L_4(\prob)}^{1/2} + \max_{a \in \calA} \abs{(\prob_n-\prob)\br{\hat \mu(t, X)}}.
    \end{align*}
\end{lemma}

\begin{proof}
Lemma 1 of \cite{bonvini2022fastconvergenceratesdoseresponse} demonstrates that
\begin{align}\label{db eq}
    \abs{\hat r(t)} \leq \norm{w - \hat w}_t \norm{\mu - \hat \mu}_t + \abs{(\prob_n-\prob)\br{\hat \mu(t, X)}},
\end{align}
where $\norm{f}_t^2 = \int f^2(z)\d\prob(z\mid A=t)$. By \eqref{db eq}, we have
\begin{align*}
    \gnorm{\hat\varphi - \varphi_0} &\leq \norm{\norm{w - \hat w}_A \norm{\mu - \hat \mu}_A + \abs{(\prob_n-\prob)_{t=A}\br{\hat \mu(t, X)}}}_{L_2(\prob_A)}\\
    &\leq \norm{w - \hat w}_{L_4(\prob)}^{1/2}\norm{\mu - \hat \mu}_{L_4(\prob)}^{1/2} + \max_{a \in \calA} \abs{(\prob_n-\prob)\br{\hat \mu(t, X)}}.
\end{align*}
\end{proof}

\subsection{Proof of Proposition \ref{prop. adaptive}}\label{subsec: proof of prop adaptive}
\begin{proof}
    For simplicity, we use the notation $\E_{m,n}$ to replace $\E_{\calD_{mn}\cup\calS_m \sim \prob^{mn}\otimes \Q^m}$. Let $q_n = (1-\sconv\eta/2)^{n}$, $\delta\pow{m,n} = \sgd^{(m,n)} - \theta_\star$, and $\delta\pow{0} = \delta\pow{0,n}$. Thus, by \Cref{theorem convergence rate baseline SGD}, 
    \begin{align*}
        \E_{m,n}[\Vert\delta\pow{m,n}\Vert_2^2] &\leq q_n\E_{m-1,n}[\Vert\delta\pow{m,0}\Vert_2^2]+ \frac{2\secsmooth_1^2}{\sconv^2} \xi_m + \frac{4\Kconst_1 \eta}{\sconv}\\
        &=q_n\E_{m-1,n}[\Vert\delta\pow{m-1,n}\Vert_2^2]+ \frac{2\secsmooth_1^2}{\sconv^2} \xi_m + \frac{4\Kconst_1 \eta}{\sconv}.
    \end{align*}
This recursive formula gives a complete bound for $\sgd^{(m,n)}$ as 
\begin{align*}
    \E_{m,n}[\Vert\delta\pow{m,n}\Vert_2^2] &\leq q_n^m\Vert\delta\pow{0}\Vert_2^2+ \frac{2\secsmooth_1^2}{\sconv^2} \sum_{i=1}^m q_n^{m-i}\xi_i + \frac{4\Kconst_1 \eta}{\sconv}\sum_{i=1}^m q_n^{m-i}.
\end{align*}

By \eqref{eq moulines}, we assume that $\xi_m \leq Cm^{-\frac{2\alpha-1}{2\alpha}}$ for some $C>0$. Note that 
\begin{align*}
    q_n = \p{1-\frac{\sconv\eta}{2}}^n \leq \exp\p{-\frac{\sconv\eta n}{2}}.
\end{align*}

For the second term, when $q_n \in (0,1)$ we have
\begin{align*}
    \sum_{i=1}^m q_n^{m-i}\xi_i &= \sum_{i=1}^{\lfloor m/2\rfloor} q_n^{m-i}\xi_i + \sum_{i=\lfloor m/2\rfloor+1}^m q_n^{m-i}\xi_i \\
    &\leq C\sum_{i=1}^{\lfloor m/2\rfloor} q_n^{m-i} + C\p{\frac{m}{2}}^{-\frac{2\alpha-1}{2\alpha}}\sum_{i=\lfloor m/2\rfloor+1}^m q_n^{m-i}\\
    &\leq \frac{Cm}{2}q_n^{m/2} + \frac{C}{1-q_n}\p{\frac{m}{2}}^{-\frac{2\alpha-1}{2\alpha}}\\
    &\leq \frac{Cm}{2}\exp\p{-\frac{\sconv\eta nm}{4}} + \frac{C}{1-q_n}\p{\frac{m}{2}}^{-\frac{2\alpha-1}{2\alpha}}.
\end{align*}

The last term is easy to bound since for $q_n \in (0,1)$,
    \begin{align*}
        \sum_{i=1}^m q_i^{m-i} = \sum_{i=0}^{m-1} q_n^{i} \leq \frac{1}{1 - q_n}.
    \end{align*}

We claim that for some constant $c>0$,
\begin{align}\label{eq 1-q lower bound}
    1 - q_n = 1 - \p{1-\frac{\sconv\eta}{2}}^n\geq c\min\br{\frac{\sconv \eta n}{2}, 1}.
\end{align}
With \eqref{eq 1-q lower bound}, we have 
\begin{align*}
    \frac{1}{1-q_n} \leq c^{-1}\p{\frac{2}{\sconv \eta n} + 1},
\end{align*}
which implies that
\begin{align*}
    \E_{m,n}[\Vert\delta\pow{m,n}\Vert_2^2] &= \bigO\p{q_n^m\Vert\delta\pow{0}\Vert_2^2+  m\exp\p{-\frac{\sconv\eta nm}{4}} + \p{m^{-\frac{2\alpha-1}{2\alpha}} + \eta}\p{\frac{1}{\eta n} + 1}}.
\end{align*}
When $(\eta n)^{-1} = \bigO(1)$, the bound above reduces to
\begin{align*}
    \E_{m,n}[\Vert\delta\pow{m,n}\Vert_2^2] &= \bigO\p{q_n^m\Vert\delta\pow{0}\Vert_2^2+ m^{-\frac{2\alpha-1}{2\alpha}} + n^{-1} + \eta}.
\end{align*}

We will finish the proof by showing \eqref{eq 1-q lower bound}. The key step is to show $1 - e^{-x} \geq c\min(x, 1)$ for all $x > 0$ and some constant $c>0$.

Let $f(x) = 1 - e^{-x} - x/2$, for $x \in (0,1)$ we have
\begin{align*}
    f'(x) = e^{-x} - \frac{1}{2} \begin{cases}
        > 0 \text{ for } x \in (0, \log 2),\\
        = 0 \text{ for } x = \log 2,\\
        < 0 \text{ for } x \in (\log 2, 1).
    \end{cases}
\end{align*}
Thus, $f(x) \geq f(\log 2) = (1 - \log 2)/2 > 0$ for $x \in (0,1)$, which implies that $1 - e^{-x} > x/2$ for $x \in (0,1)$. Note that $1 - e^{-x} \geq 1 - e^{-1}$ for $x \geq 1$. Let $c = \min(2^{-1}, 1-e^{-1})$. Then we have $1 - e^{-x} \geq c\min(x,1)$.

It follows that
\begin{align*}
    1 - q_n = 1 - \exp\p{-n\log \p{\frac{1}{1 - \sconv \eta/2}}} \geq c\min\br{n\log \p{\frac{1}{1 - \sconv \eta/2}}, 1}.
\end{align*}
Since $x-1 \geq \log x$ for all $x >0$, we have
\begin{align*}
    \log \p{1 - \sconv \eta/2} \leq 1 - \sconv \eta/2 - 1 = - \sconv \eta/2,
\end{align*}
which implies that
\begin{align*}
    n\log \p{\frac{1}{1 - \sconv \eta/2}} \geq \frac{\sconv \eta n}{2}.
\end{align*}
Thus, we complete the proof.
\end{proof}

\subsection{Proof of Proposition \ref{prop. prob bound}}\label{subsec: proof of prop prob bound}

\begin{proof}
Given $s>0$, we define $\calA_{i} = \calA_{i}(s^{1/c})$ and $\calB_{i} = \calB_{i}(s^{1/c})$ for $i=0,\dots,m$ for simplicity. First, since $c\geq 1,$ by \eqref{markov 1} and Markov inequality, for $i = 1,\dots, m,$
    \begin{align*}
        \P{}{\calA_{i} \mid \calA_{i-1}, \dots, \calA_{0}} &= 1-\P{}{\gnorm{\hat g\pow{i} - g_0}^2 \geq Cs^{-1/c}i^{-\frac{2\alpha - 1}{2\alpha}} \mid \calA_{i-1}, \dots, \calA_{0}}\\
        &= 1-\P{}{\gnorm{\hat g\pow{i} - g_0}^{2c} \geq C^cs^{-1}i^{-\frac{(2\alpha - 1)c}{2\alpha}} \mid \calA_{i-1}, \dots, \calA_{0}}\\
        &\geq 1-\frac{\Ex{\calS_m}{\gnorm{\hat g\pow{i} - g_0}^{2c} \mid \calA_{i-1}, \dots, \calA_{1}}}{C^cs^{-1}i^{-\frac{(2\alpha - 1)c}{2\alpha}}} \\
        &\geq 1- s.
    \end{align*}
We assume that $\P{}{\calA_0}=\P{}{\calB_0} = 1$, and we have
    \begin{align*}
       \P{}{\calA_{m}, \calA_{i-1}, \dots, \calA_{1}, \calA_0}  &= \P{}{\calA_{m} \mid \calA_{m-1}, \dots, \calA_{1}}\dots  \P{}{\calA_{1} \mid \calA_0}\P{}{\calA_0} \\
       &\geq \prod_{i=1}^m (1- s) = (1-s)^m.
    \end{align*}
Similarly, we have
\begin{align*}
    \P{}{\calB_{m}, \calB_{i-1}, \dots, \calB_{1}}\geq (1-s)^m.
\end{align*}
we consider the conditional mean squared error of $\sgd^{(m,n)}$ given $(\cap_{i=0}^m \calA_i) \cap (\cap_{i=0}^m \calB_i)$. By similar proof to \Cref{prop. adaptive}, we can show that for some constant $C_1>0$,
\begin{align*}
    \E_{\calD_{mn}\cup\calS_m}[\Vert\sgd^{(m,n)} -& \theta_\star\Vert_2^2 \mid (\cap_{i=0}^m \calA_i )\cap (\cap_{i=0}^m \calB_i)] \leq C_1\p{1-\frac{\sconv\eta}{2}}^{mn}\Vert\sgd^{(0)} - \theta_\star\Vert_2^2 \\
    &+C_1\p{m\exp\p{-\frac{\sconv\eta nm}{4}} + (s^{-2/c}m^{-\frac{2\alpha-1}{\alpha}} + \eta)((\eta n)^{-1} + 1)}.
\end{align*}

We define the event of interest as
\begin{align*}
    \calE(s) &= \br{\Vert\sgd^{(m,n)} - \theta_\star\Vert_2^2  \leq C_1s^{-1}f_s(m,n)},
\end{align*}
where $f_s(m,n)$ is defined as
\begin{align*}
    f_s(m,n) = &\p{1-\frac{\sconv\eta}{2}}^{mn}\Vert\sgd^{(0)} - \theta_\star\Vert_2^2 \\
    &+ m\exp\p{-\frac{\sconv\eta nm}{4}} + (s^{-2/c}m^{-\frac{2\alpha-1}{\alpha}} + \eta)((\eta n)^{-1} + 1).
\end{align*}

By Markov inequality, we have
\begin{align*}
    \P{}{\calE_3(s)\mid (\cap_{i=0}^m \calA_i) \cap (\cap_{i=0}^m \calB_i)} &\geq 1-\frac{\E_{\calD_{mn}\cup\calS_m}[\Vert\sgd^{(m,n)} - \theta_\star\Vert_2^2 \mid (\cap_{i=0}^m \calA_i) \cap (\cap_{i=0}^m \calB_i)]}{C_1s^{-1}f(m,n)} \\
    &\geq 1-s.
\end{align*}
Since
\begin{align*}
    \P{}{\calE_3(s)^c} &= \Ex{\prob}{\ind_{\calE_3(s)^c}\ind_{(\cap_{i=0}^m \calA_i) \cap (\cap_{i=0}^m \calB_i)}} + \Ex{\prob}{\ind_{\calE_3(s)^c}\ind_{((\cap_{i=0}^m \calA_i) \cap (\cap_{i=0}^m \calB_i))^c}}\\
    &\leq \Ex{\prob}{\ind_{\calE_3(s)^c\cap(\cap_{i=0}^m \calA_i) \cap (\cap_{i=0}^m \calB_i)}} + \Ex{\prob}{\ind_{((\cap_{i=0}^m \calA_i) \cap (\cap_{i=0}^m \calB_i))^c}}\\
    &=\P{}{\calE_3(s)^c\cap(\cap_{i=0}^m \calA_i) \cap (\cap_{i=0}^m \calB_i)} + \P{}{((\cap_{i=0}^m \calA_i) \cap (\cap_{i=0}^m \calB_i))^c}\\
    &\leq \P{}{\calE_3(s)^c\mid (\cap_{i=0}^m \calA_i) \cap (\cap_{i=0}^m \calB_i)} + \P{}{(\cap_{i=0}^m \calA_i)^c} + \P{}{(\cap_{i=0}^m \calB_i)^c},
\end{align*}
which implies that 
\begin{align*}
  \P{}{\calE_3(s)} \geq  2(1-s)^m - s - 1.
\end{align*}
Define $\delta(s)$ as
\begin{align}\label{eq delta}
    \delta(s) = s + 2(1-(1-s)^m) = \bigO(ms).
\end{align}
Then, with probability at least $1-\delta(s)$, we have
\begin{align*}
   \Vert\sgd^{(m,n)} -& \theta_\star\Vert_2^2 \lesssim s^{-1}\p{1-\frac{\sconv\eta}{2}}^{mn}\Vert\sgd^{(0)} - \theta_\star\Vert_2^2 \\
    &+s^{-1}\p{m\exp\p{-\frac{\sconv\eta nm}{4}} + (s^{-2/c}m^{-\frac{2\alpha-1}{\alpha}} + \eta)((\eta n)^{-1} + 1)}.
\end{align*}
When $(\eta n)^{-1} = \bigO(1)$, it follows that
\begin{align*}
   \Vert\sgd^{(m,n)} -& \theta_\star\Vert_2^2 \lesssim s^{-1}\p{1-\frac{\sconv\eta}{2}}^{mn}\Vert\sgd^{(0)} - \theta_\star\Vert_2^2 +s^{-1}\p{s^{-2/c}m^{-\frac{2\alpha-1}{\alpha}} + n^{-1} + \eta}.
\end{align*}
\end{proof}

\clearpage

\section{Numerical Experiments}\label{appx:experiments}

This section provides numerical experiments of the proposed stochastic methods in this paper. In \Cref{appx discuss sec 5}, we design a numerical experiment to illustrate our orthogonalization method. In \Cref{sec: sim}, we design simulations based on a partially linear model. In \Cref{sec: real data}, we conduct a real data analysis with synthetic outcome to evaluate the performance of our methods. Code for reproduction can be found at \href{https://fachengyu.github.io/}{\url{https://fachengyu.github.io/}}.

\subsection{Numerical Illustration}\label{appx discuss sec 5}

In this section, we design a numerical experiment to illustrate how our orthogonalization method effects the target estimation as shown in \Cref{fig:osl} from the main text.

\myparagraph{Settings} Consider $\Theta\in R$ and $\G =\R$. Let $\ploss(\theta, g)$ be a real-valued risk function defined as
\begin{align}
    \ploss(\theta, g) := \ploss(u) = \frac{1}{2}\ip{u, Au} + \lambda \sin^2(\ip{u, Bu}),
\end{align}
where $u = (\theta, g)^\top \in \R^2$, $\lambda = 0.02$ is the regularization parameter, and 
\begin{align*}
    A = \begin{pmatrix}8 & 3\\ 3 & 2\end{pmatrix} \succ 0 \text{ and } B = \begin{pmatrix}2 & -1\\ -1 & 1.5\end{pmatrix}\succ 0.
\end{align*}
It is easy to see that $(0,0)$ is the global minimizer of $L$ since $L(\theta, g) \geq 0$. Let $q(u) = \ip{u, Bu}$. The gradient w.r.t. $u$ is
\begin{align*}
    \nabla_u L(u) &= Au + 4 \lambda\sin(q(u))\cos(q(u))Bu\\
    &=(A + 2\lambda\sin(2q(u))B)u.
\end{align*}
Since $A + 2\lambda\sin(2q(u))B \succcurlyeq A - 0.04 B \succ 0$, it is clear that $(0,0)$ is the only stationary point, implying that $(0,0)$ is the only minimizer of $L$. Furthermore, we can obtain the Hessian w.r.t. $u$ as
\begin{align}\label{eq colab hessian}
    \nabla_u^2 L(u) = A  + 2\lambda\sin(2q(u))B + 8\lambda \cos(2q(u))Bu(Bu)^\top,
\end{align}
which implies that $L(\cdot,g)$ is not convex in $\R$ given any $g \in \Gr$. However, when $\Theta$ is a small neighborhood around zero, it is still possible to have $L(\cdot,g)$ strongly convex for in $\Theta$ given any $g \in \Gr$.

\myparagraph{Orthogonalization}
To orthogonalize $\ploss$, we first derive the orthogonal gradient oracle using \eqref{no score of infinite dim}, and then integral the oracle \wrt $\theta$ to obtain the orthogonalized loss $\ploss_{\text{no}}$. 

Let $H$ be the Hessian at $(0,0)$. By \eqref{eq colab hessian}, we know that $H = A$, implying $H_{\theta g} = A_{12}$ and $H_{gg} = A_{22}$. Since the gradient w.r.t. $\theta$ satisfies
\begin{align*}
  \nabla_\theta L(\theta, g) &= [1, 0](A + 2\lambda\sin(2q(u))B )u \\
  &= (A_{11} + 2\lambda\sin(2q(\theta, g))B_{11})\theta + (A_{12} + 2\lambda\sin(2q(\theta, g))B_{12})g,
  \end{align*}
and the gradient w.r.t. $g$ satisfies
\begin{align*}
  \nabla_g L(\theta, g) &= [0, 1](A + 2\lambda\sin(2q(u))B )u \\
  &= (A_{21} + 2\lambda\sin(2q(\theta, g))B_{21})\theta + (A_{22} + 2\lambda\sin(2q(u))B_{22})g,
  \end{align*}
follow the construction of \eqref{no score of infinite dim} and we obtain the orthogonal gradient oracle as
  \begin{align*}
  \noscore(\theta ,g) &= \nabla_\theta L(\theta, g) - H_{\theta g} H_{gg}^{-1} \nabla_g L(\theta, g) \\
  &=(a + 2b\lambda\sin(2q(\theta, g)))\theta + 2c\lambda\sin(2q(\theta, g))g,
  \end{align*}
where $a = A_{11} - A_{12} A_{22}^{-1}A_{21}$, $b = B_{11} - A_{12} A_{22}^{-1}B_{21}$, and $c = B_{12} - A_{12} A_{22}^{-1}B_{22}$. Finally, we can integral $\noscore(\theta ,g)$ w.r.t. $\theta$ and recover the orthognalized loss $\ploss_{\text{no}}$ as
\begin{align*}
    \ploss_{\text{no}}(\theta, g) = \int_{0}^\theta \noscore(s ,g) \d s.
\end{align*}

\myparagraph{Numerical Computation}
Usually, $\noscore(s ,g)$ contains a form of integral, which needs to be numerically computed. For the example introduced above, we can simplify $\ploss_{\text{no}}(\theta, g)$ to stabilize the numerical computation. Note that $\nabla_\theta \sin^2(q(\theta, g)) = \sin(2q(\theta, g))(B_{11}\theta + B_{12}g)$. Then
\begin{align*}
  2b\lambda \int_0^\theta \sin(2q(s, g))s ds &= \frac{2b\lambda}{B_{11}}\p{\int_0^\theta \sin(2q(s, g))(B_{11}s + B_{12}g) \d s - B_{12} g\int_0^\theta \sin(2q(s, g)) \d s}\\
  &= \frac{2b\lambda}{B_{11}}\p{\sin^2(q(\theta, g)) -  B_{12} g\int_0^\theta \sin(2q(s, g)) \d s}.
  \end{align*}
It follows that the orthogonalized loss $\ploss_{\text{no}}$ admits the following form 
\begin{align*}
    \ploss_{\text{no}}(\theta, g) &=\frac{a}{2}\theta^2 +  \frac{2b\lambda}{B_{11}}\sin^2(q(\theta, g)) + 2\p{c - \frac{B_{12}}{B_{11}}b}\lambda g\int_0^\theta \sin(2q(s, g)) \d s,
\end{align*}
which implies that only the integral of $\sin(2q(s, g))$ \wrt $s$ needs to be computed.

\subsection{Simulations}\label{sec: sim}
\subsubsection{Data Generating Process}
To demonstrate \Cref{theorem convergence rate baseline SGD} and \Cref{theorem debiased SGD convergence rate}, we revisit the partially linear model and the corresponding orthogonal and non-orthogonal losses in \Cref{appx PLM}. Specifically, $(X,W,Y) \in \R^d \times \R^d \times \R$ satisfies the following partially linear model where the nonlinear function is determined by the distribution of $(W, U) \in \R^d \times \R$: 
\begin{align}
    Y &= \ip{\theta_0, X} + \alpha_0(W) + \epsilon \label{sim eq Y},\\
    U &= \alpha_0(W) + \xi \label{sim eq U},
\end{align}
where $\theta_0 \in \R^d$ is the true parameter, $\alpha_0: \calW \mapsto \R$ is the true nonlinear function, $\epsilon$ and $\xi$ are independent noises that satisfy $\E[\epsilon \mid X, W] = 0$ and $\E[\xi \mid W] = 0$. It is clear that $\alpha_0(W) = \E[U \mid W]$. In our simulations, we choose $d = 2$ and $\theta_0 = [-0.5 ~~1]^\top$. 

To get samples for simulations, we first generate $(X,W)$ under the Gaussian model
\begin{align}
    \begin{bmatrix} X\\ W \end{bmatrix} = \calN\left(\begin{bmatrix} \boldsymbol{\mu}_X \\ \boldsymbol{\mu}_W \end{bmatrix}, \begin{bmatrix} (1 + \delta)\mathbf{I}_2 &\lambda \mathbf{I}_2\\ \lambda \mathbf{I}_2 & (1 + \delta)\mathbf{I}_2 \end{bmatrix}\right),
\end{align}
where $\boldsymbol{\mu}_X = [1 ~~1]^\top$, $\boldsymbol{\mu}_W = [2 ~~2]^\top$, $\mathbf{I}_2 \in \R^{2 \times 2}$ is the identity matrix, $\lambda \in [0,1]$ is used to control the correlation between $X$ and $W$, and  $\delta = 0.05$ is used to prevent the degeneration of the covariance matrix. For simplicity, we define the nonlinear function $\alpha_0$ as
\begin{align}
    \alpha_0(w) = 0.5\times \cos\p{\frac{w_1 + w_2}{2}} + 0.5\times \sin\p{\frac{w_1 + w_2}{2}},
\end{align}
where $w = [w_1~~ w_2]^\top \in \R^2$. We then generate $Y$ and $U$ using independent Gaussian noises $\epsilon \sim \calN(0,1)$ and $\xi \sim \calN(0,1)$ based on \eqref{sim eq Y} and \eqref{sim eq U}, respectively.

\subsubsection{Stochastic Gradient Oracles}\label{subsec: sim SGD oracle}
To estimate the true parameter $\theta_0$ using stochastic gradients, we need to design a correspond loss whose minimizer $\theta_\star$ is equal to $\theta_0$.
Based on \Cref{appx PLM}, there are two types of loss, the orthogonal loss and the non-orthogonal loss, available for this goal. We will derive the stochastic gradient oracle for these two losses and further derive the orthogonalized gradient oracle for the non-orthogonal loss.

\myparagraph{Orthogonal loss}
Recall the orthogonal loss in \Cref{subsubsec: ortho PLM}:
\begin{align}\label{sim ortho loss}
        \loss(\theta, g; z) = \frac{1}{2}[y - g_Y(w) - \ip{\theta, x - g_X(w)}]^2,
\end{align}
where $g = (g_Y, g_X): \calW \rightarrow \R \times \R^d$ and the norm $\gnorm{\cdot}$ is defined in \eqref{eq: ex1 Gr}. The true nuisance for this loss is $g_0 = (g_{0, X}, g_{0, Y})$, where $g_{0, Y}(w) := \Ex{\prob}{Y \mid W = w} \text{ and } g_{0, X}(w) := \Ex{\prob}{X \mid W = w}.$
In fact, the explicit expression for $g_0$ can be easily obtained as
\begin{align}
    g_{0, Y}(w) &= \ip{\theta_0, g_{0, X}(w)} + \alpha_0(w),\label{sim g0,Y}\\
    g_{0, X}(w) &= \mu_X + \frac{\lambda}{1.05}(w - \mu_W). \label{sim g0,X}
\end{align}
From \eqref{sim eq Y} and \eqref{sim g0,Y}, it is clear that 
\begin{align*}
    Y - g_{0, Y}(W) = \ip{\theta_0, X - g_{0, X}(W)} + \epsilon,
\end{align*}
which implies that $\theta_\star = \theta_0$ by \Cref{lemma: example ortho PLM}. The stochastic gradient oracle for the orthogonal loss \eqref{sim ortho loss} is then defined as 
\begin{align}\label{sim ortho score}
    S_\theta(\theta, g; z) = -(y - g_Y(w) - \ip{\theta, x - g_X(w)})(x - g_X(w)).
\end{align}

\myparagraph{Non-orthogonal loss}
We also provide the non-orthogonal loss in \Cref{subsubsec: non-ortho plm} as
\begin{align}\label{sim nonortho loss}
    \loss(\theta, g; z) = \frac{1}{2}[y  - g(w) - \ip{\theta, x}]^2,
\end{align}
where $g: \calW \mapsto \R$ and the norm $\gnorm{\cdot}$ is now defined in \eqref{eq: ex2 Gr}. The true nuisance for this non-orthogonal loss satisfies 
\begin{align}
    g_0(w) = \alpha_0(w) = \Ex{}{U \mid W = w}. \label{sim g0}
\end{align}
By \Cref{lemma: example non-ortho PLM}, we have $\theta_\star = \theta_0$. The stochastic gradient oracle for the orthogonal loss \eqref{sim nonortho loss} is then defined as 
\begin{align}\label{sim nonortho score}
    S_\theta(\theta, g; z) = -(y - g(w) - \ip{\theta, x})x.
\end{align}

\myparagraph{Orthogonalized gradient oracle}
Since we perform orthogonalization on the non-orthogonal loss, we have $\theta_\star = \theta_0$ being the same target parameter. By \eqref{ex2 noscore} in \Cref{subsubsec: non-ortho plm}, the Neyman orthogonalized gradient oracle for this non-orthogonal loss \eqref{sim nonortho loss} is given by
\begin{align}\label{sim noscore}
    \noscore(\theta, g; z) = -(y - g(w) - \ip{\theta, x})(x - \Ex{}{X\mid W=w}).
\end{align}
\subsubsection{Estimation Methods}\label{subsec: sim nuisance estimation}
Throughout the experiments, we estimate the nuisances and the orthogonalizing operator using full-batch data and stream data, respectively.

\myparagraph{Nuisance estimation} Note that the true nuisances for the orthogonal loss and the non-orthogonal loss are conditional expectation given $W$. To conduct nonparametric regression, we use random Fourier feature (RFF) \citep{rahimi2007random} using the kernel $w \mapsto \exp\p{-\gamma \cdot\norm{w}_2^2}$ to generate a randomized feature map for $W$.

The nuisance estimation procedure for obtaining $\hat g\pow{m}$ using full batch data can be summarized as
\begin{enumerate}
    \item Fit RFF sampler with $20$ components using $m$ \iid samples from $P_{W\mid \lambda}$.
    \item Fit Ridge regressions where the regularization parameter is set to be $0.01/m$. Specifically, 
    \begin{itemize}
        \item For the orthogonal loss, fit two Ridge regressions using $m$ \iid samples from the joint distribution $P_{X,W,Y\mid \lambda}$ and the fitted RFF sampler to coordinate-wisely estimate $\E[X \mid W]$. With the same data, fit one Ridge regression using the fitted RFF sampler to estimate $\E[Y \mid W]$.
        \item For the non-orthogonal loss, fit one Ridge regression using $m$ \iid samples from the joint distribution $P_{X,W,Y\mid \lambda}$ and the fitted RFF sampler to estimate $\E[U \mid W]$. 
    \end{itemize}
\end{enumerate}

To estimate nuisances using stream data, instead of fit a Ridge regression each time, we perform SGD for the Ridge regression loss. The procedure can be summarized as
\begin{enumerate}
    \item Initialize RFF sampler with $20$ components using $n_0$ \iid samples $(W_i)_{i=1}^{n_0}$ from $P_{W\mid \lambda}$.
    \item Perform SGD update once observing a mini-batch of \iid samples from the joint distribution $P_{X,W,Y\mid \lambda}$ with size $n_g$. Specifically, 
    \begin{itemize}
        \item For the orthogonal loss, perform two SGD with the Ridge loss for $m$ iterations to estimate $\E[X\mid W]$ coordinate-wisely. With the same data perform another SGD with the Ridge loss for $m$ iterations to estimate $\E[Y\mid W]$.
        \item For the non-orthogonal loss, perform one SGD with the Ridge loss for $m$ iterations to estimate $\E[U\mid W]$.
    \end{itemize}
\end{enumerate}

\myparagraph{Orthogonalizing operator estimation} To approximate the orthogonalizing operator $\Gamma_0$, it suffices to estimate $\Ex{}{X \mid W}$ by \eqref{sim gamma0}. To that end, we use the same method as the nuisance estimation. The orthogonalizing operator estimation procedure for obtaining $\hat \Gamma\pow{k}$ can be summarized as
\begin{enumerate}
    \item Fit RFF sampler with $20$ components using $k$ \iid samples $(W'_i)_{i=1}^{k}$ from $P_{W\mid \lambda}$.
    \item Fit two Ridge regressions with the regularization parameter being $0.01/k$ using the fitted RFF sampler and another $k$ \iid samples $(X'_i, W'_i)_{i=k}^{2k}$ to coordinate-wisely estimate $\E[X \mid W]$.
\end{enumerate}

\myparagraph{Target estimation}
After the estimation of nuisances and orthogonalizing operator, we perform stochastic gradient descent (SGD) to estimate $\theta_\star$ using each of the three stochastic gradient oracles in \eqref{sim ortho score}, \eqref{sim nonortho score}, and \eqref{sim noscore} on $n$ \iid samples drawn from the joint distribution $P_{X,W,U,Y}$. The learning rates of all the three SGDs are fixed during the training.

\subsubsection{Simulation Results}\label{subsec: sim results}
\myparagraph{Setup} 
For each nuisance estimation setting, we study three types of estimation methods for learning $\theta_0$ established in this paper: (1) (orthogonal loss) obtain nuisance estimator $\hat g\pow{m} = (\hat g_{Y}\pow{m}, \hat g_{X}\pow{m})$ of \eqref{sim g0,Y} and \eqref{sim g0,X} and then perform SGD to obtain $\theta\pow{n}$ using the gradient oracle \eqref{sim ortho score} after plugging in $\hat g\pow{m}$; (2) (non-orthogonal loss) obtain the nuisance estimator $\hat g\pow{m} = \hat \alpha\pow{m}$ of \eqref{sim g0} and then perform SGD to obtain $\theta\pow{n}$ using the gradient oracle \eqref{sim nonortho score} after plugging in $\hat g\pow{m}$; (3) (OSGD) obtain the nuisance estimator $\hat g\pow{m}$ of \eqref{sim g0} and the orthogonalizing operator estimator $\hat \Gamma\pow{k}$ of \eqref{sim gamma0}, and then perform SGD to obtain $\theta\pow{n}$ using the gradient oracle \eqref{sim noscore} after plugging in $\hat g\pow{m}$ and $\hat \Gamma\pow{k}$. Each method is independently repeated 20 times. For nuisance estimated using stream data, we allow the procedure repeated by plugging in updated nuisance estimators and an updated operator estimator, where the nuisances get updated for 2000 iterations after every 2000 target SGD iterations.

\myparagraph{Evaluation}
We evaluate the performance of nuisance estimators using the corresponding norms defined in \eqref{eq: ex1 Gr} and \eqref{eq: ex2 Gr}. Specifically, for method (1), we evaluate the nuisance estimation by 
\begin{align}\label{sim gnorm for ortho}
    \gnorm{\hat g\pow{m} - g_0} =  \max\br{\Ex{}{\norm{\hat g_{X}\pow{m}(W) - g_{0,X}(W)}_2^4}^{\frac{1}{4}}, \Ex{}{(\hat g_{Y}\pow{m}(W) - g_{0,Y}(W))^4}^{\frac{1}{4}}}.
\end{align}
For method (2) and (3), we evaluate the nuisance estimation by 
\begin{align}\label{eq: sim gnorm non-ortho}
    \gnorm{\hat g\pow{m} - g_0} =  \Ex{}{\norm{\hat \alpha\pow{m}(W) - \alpha_{0}(W)}_2^2}^{\frac{1}{2}}.
\end{align}
We evaluate $\hat \Gamma\pow{k}: g \mapsto \E[\hg_X\pow{k}(W) g(W)]$ in method (3) using the Frobenius norm $\Vert \hat \Gamma\pow{k} - \Gamma_0\Vert_\Fro$, which is defined as
\begin{align}\label{sim operator norm}
    \Vert \hat \Gamma\pow{k} - \Gamma_0\Vert_\Fro = \Ex{}{\norm{\hg_X\pow{k}(W) - \hg_{0,X}(W)}_2^2}^{\frac{1}{2}}.
\end{align}
Finally, we evaluate the target estimation using two kinds of criterion: (a) the relative error $\frac{\norm{\theta\pow{n} - \theta_0}_2}{\norm{\theta_0}_2}$, and (b) the risk $L(\theta\pow{n}, g_0) - L(\theta_\star, g_0)$ where $L(\theta, g) = \E[\ell(\theta, g; Z)]$. For method (1), $\ell(\theta, g; z)$ is the orthogonal loss defined in \eqref{sim ortho loss} while for method (2) and (3), $\ell(\theta, g; z)$ is the non-orthogonal loss defined in \eqref{sim nonortho loss}.

\myparagraph{Results using nuisances fitted on full-batch data}
We first estimate the target using prefitted nuisances and operator. The estimation errors of nuisances and the operator fitted on full-batch data are shown in \Cref{fig:nuisance estimation error}, where all estimation converges when the sample size $m$ increases and less samples are usually required to obtain the same error level when $\lambda$ increases.

\begin{figure}
    \centering
    \includegraphics[width=\linewidth]{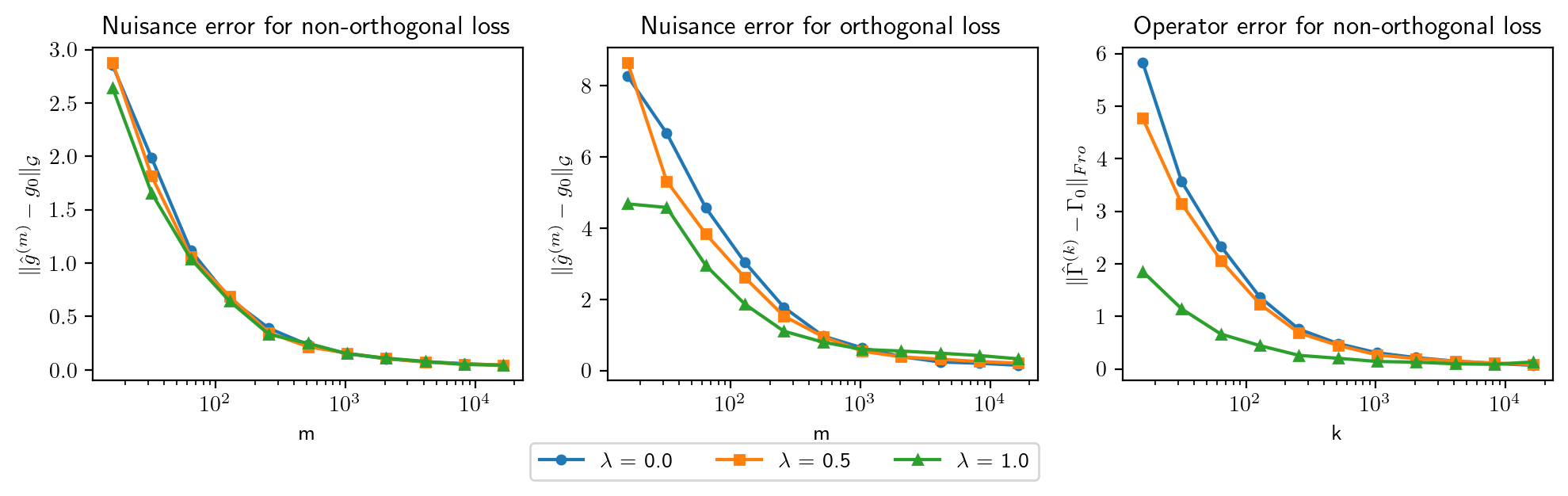}
    \caption{{\bf The Nuisance and Orthogonalizing Operator Fitted on Full-Batch Simulated Data.} The y-axis measures the corresponding error defined in \eqref{sim gnorm for ortho} - \eqref{sim operator norm} and the x-axis displays the sample size of data used to estimate the nuisance and operator.}
    \label{fig:nuisance estimation error}
\end{figure}

The performances of SGDs using prefitted nuisances and stochastic gradient oracles \eqref{sim ortho score}, \eqref{sim nonortho score}, and \eqref{sim noscore} are shown in \Cref{fig: sgd orthogonal}, \Cref{fig: sgd non-orthogonal}, and \Cref{fig: osgd}, respectively. These figures suggest that when $\lambda$ increases, i.e., the correlation between $X$ and $W$ increases, usually more iterations are required to have SGD converged due to the difficulty of separating the effect of $X$ from $W$. In addition, a well prefitted nuisance estimator would largely reduce the SGD estimation error, which aligns with \Cref{theorem convergence rate baseline SGD}. This improvement would be more obvious as $\lambda$ increases. \Cref{fig: osgd} also shows that either using a well estimated nuisance or a well estimated orthogonalizing operator can improve the OSGD performance, and that OSGD using both well prefitted nuisance and operator would perform nearly the same as OSGD using the true nuisance and the true operator.

\begin{figure}
    \centering
    \includegraphics[width=\linewidth]{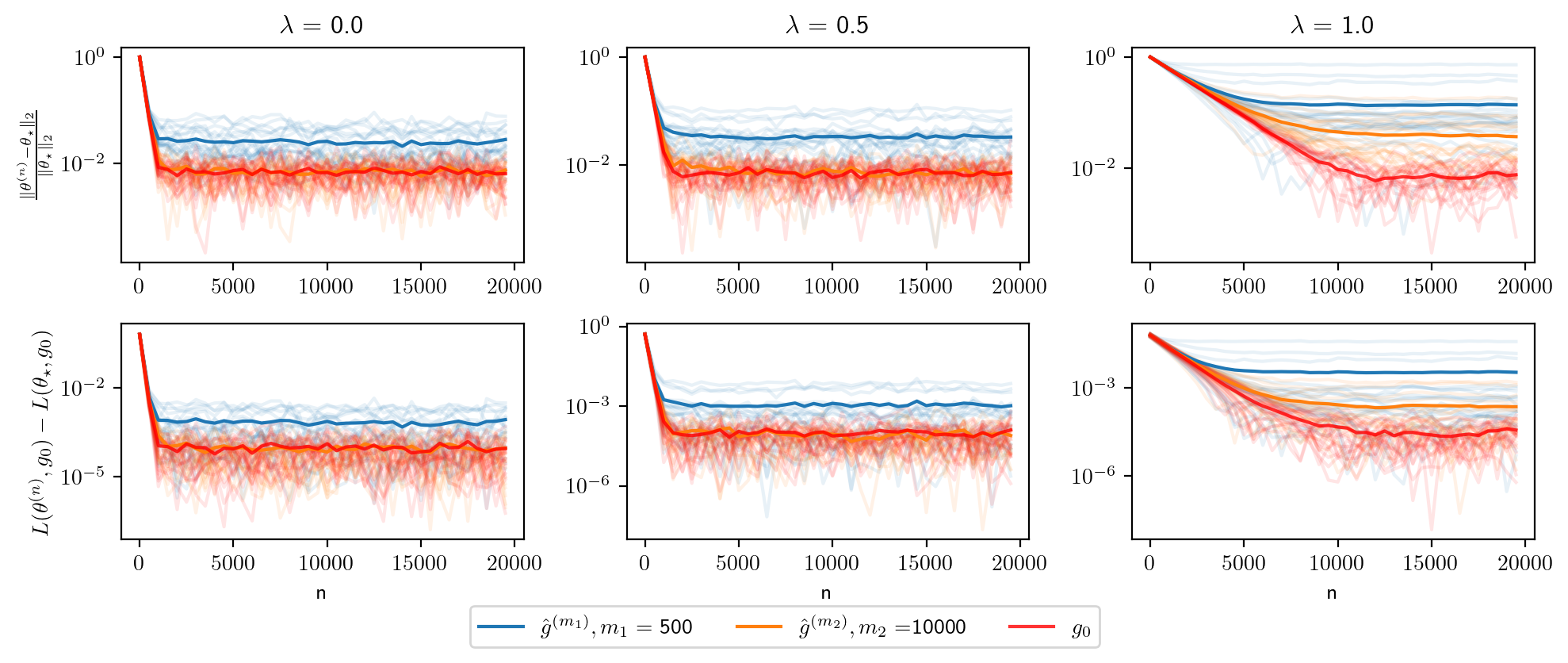}
    \caption{{\bf SGD for Orthogonal Loss with the Nuisance Fitted on Full-Batch Simulated Data.} The x-axis represents the SGD iteration. {\bf Top:} The y-axis measures the relative error. {\bf Bottom:} The y-axis measures the risk.}
    \label{fig: sgd orthogonal}
\end{figure}

\begin{figure}
    \centering
    \includegraphics[width=\linewidth]{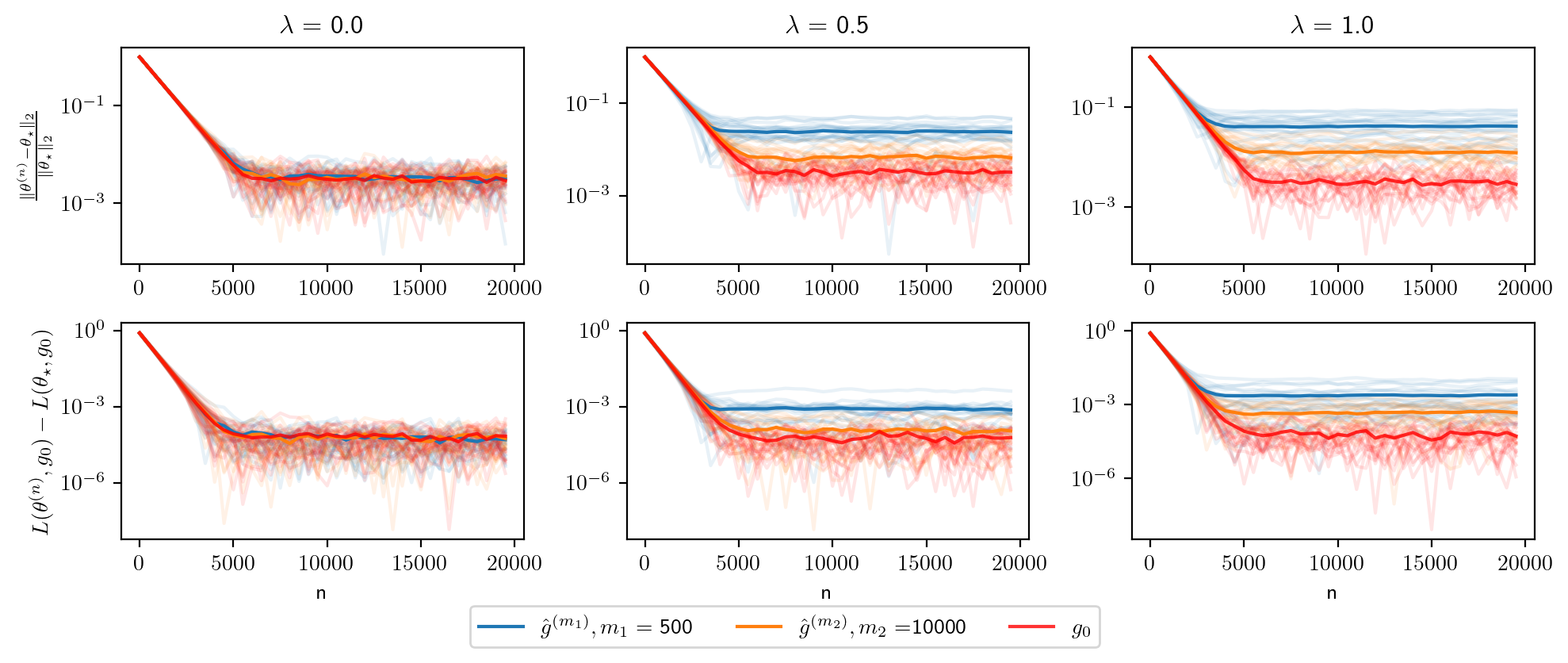}
    \caption{{\bf SGD for Non-Orthogonal Loss with the Nuisance Fitted on Full-Batch Simulated Data.} The x-axis represents the SGD iteration. {\bf Top:} The y-axis measures the relative error. {\bf Bottom:} The y-axis measures the risk.}
    \label{fig: sgd non-orthogonal}
\end{figure}

\begin{figure}
    \centering
    \includegraphics[width=\linewidth]{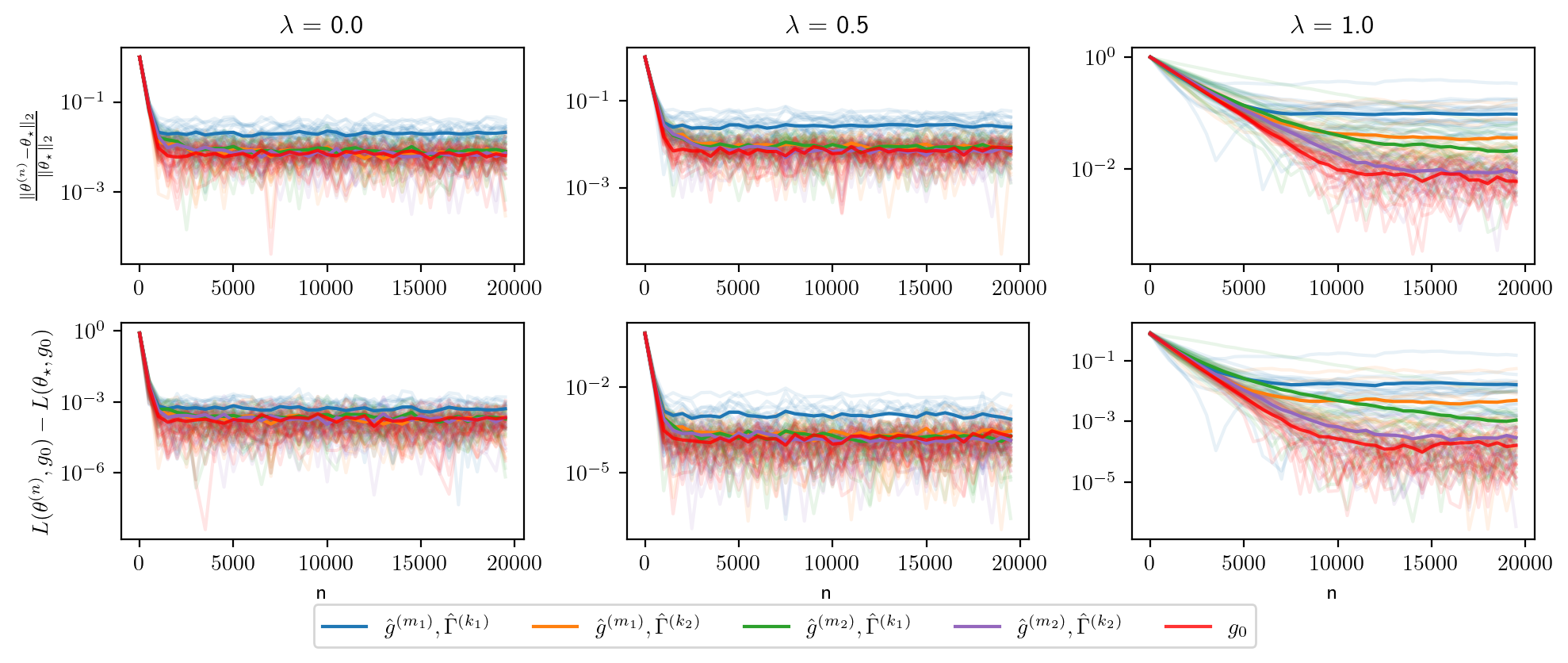}
    \caption{{\bf OSGD with the Nuisance and Operator Fitted on Full-Batch Simulated Data.} Here, $m_1 = 500$, $m_2 = 10000$, $k_1 = 300$, $k_2 = 10000$. The x-axis represents the OSGD iteration. {\bf Top:} The y-axis measures the relative error. {\bf Bottom:} The y-axis measures the risk.}
    \label{fig: osgd}
\end{figure}

\myparagraph{Results using nuisances fitted on stream data}
We then study the interleaving the nuisance and target estimations discussed in \Cref{appx discuss sec 3}. Here, Both the nuisance and the operator are learned using the same data stream and the results are shown in \Cref{fig: nuisance error stream}. Compared with \Cref{fig:nuisance estimation error}, nuisances estimated using stream data usually has larger error and need more iterations to converge due to mini-batch, learning rate, and other tuning parameters.

\begin{figure}
    \centering
    \includegraphics[width=\linewidth]{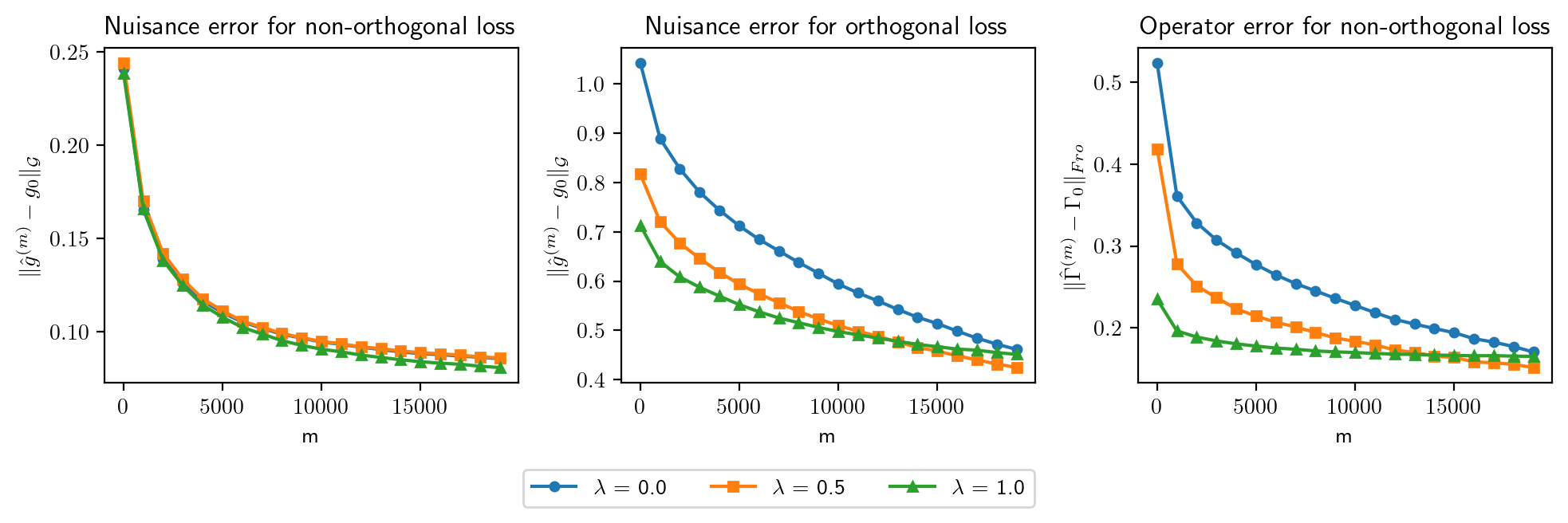}
    \caption{{\bf Nuisance and Orthogonalizing Operator Fitted on Simulated Stream Data.} The y-axis measures the corresponding error defined in \eqref{sim gnorm for ortho} - \eqref{sim operator norm} and the x-axis displays the sample size of data used to estimate the nuisance and operator.}
    \label{fig: nuisance error stream}
\end{figure}

The performances of SGDs by interleaving nuisance and target updates with stochastic gradient oracles \eqref{sim ortho score}, \eqref{sim nonortho score}, and \eqref{sim noscore} are shown in \Cref{fig: sgd orthogonal stream}, \Cref{fig: sgd non-orthogonal stream}, and \Cref{fig: osgd stream}, respectively. For all the three stochastic gradients, when $\lambda$ increases, the relative errors of the target SGD always get larger and their convergence rates become slower. There are obvious errors for SGDs using gradient oracles \eqref{sim ortho score} and \eqref{sim nonortho score} in \Cref{fig: sgd orthogonal stream} and \Cref{fig: sgd non-orthogonal stream} since nuisances are not well estimated. However, OSGD performs perfectly as shown in \Cref{fig: osgd stream}, which verifies the analysis of \Cref{theorem debiased SGD convergence rate} that using an estimated orthogonalizing operator would reduce the bias from nuisance estimation.

\begin{figure}
    \centering
    \includegraphics[width=\linewidth]{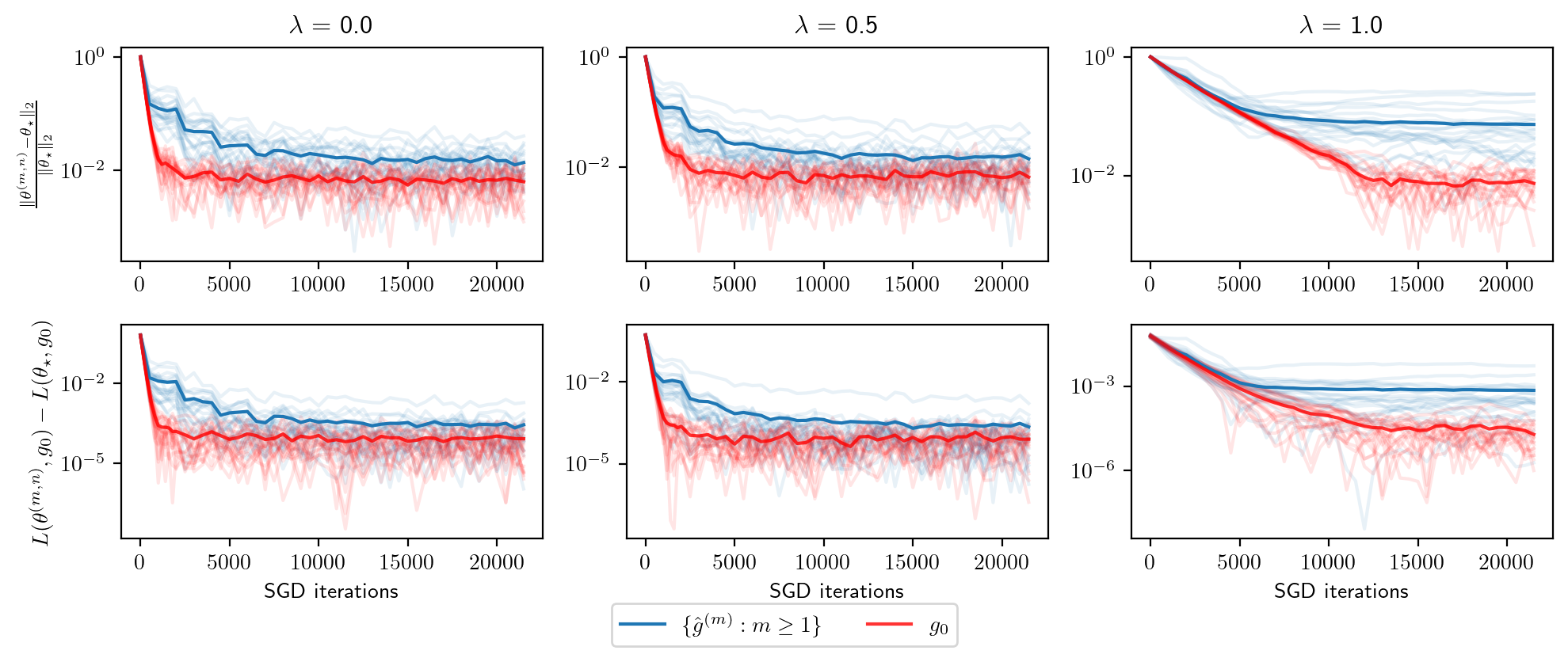}
    \caption{{\bf SGD for Orthogonal Loss with the Nuisance Fitted on Simulated Stream Data.} The x-axis represents the SGD iteration. {\bf Top:} The y-axis measures the relative error. {\bf Bottom:} The y-axis measures the risk.}
    \label{fig: sgd orthogonal stream}
\end{figure}

\begin{figure}
    \centering
    \includegraphics[width=\linewidth]{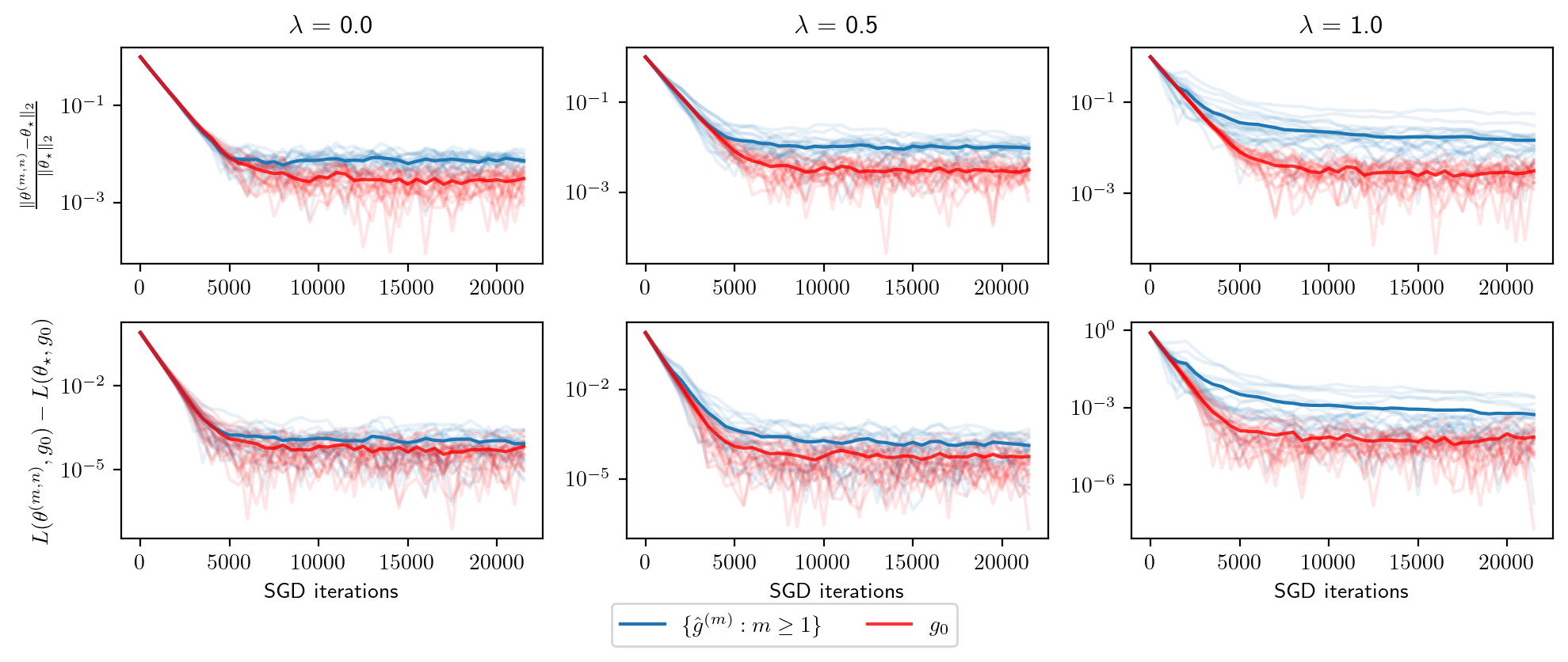}
    \caption{{\bf SGD for Non-Orthogonal Loss with the Nuisance Fitted on Simulated Stream Data.} The x-axis represents the SGD iteration. {\bf Top:} The y-axis measures the relative error. {\bf Bottom:} The y-axis measures the risk.}
    \label{fig: sgd non-orthogonal stream}
\end{figure}

\begin{figure}
    \centering
    \includegraphics[width=\linewidth]{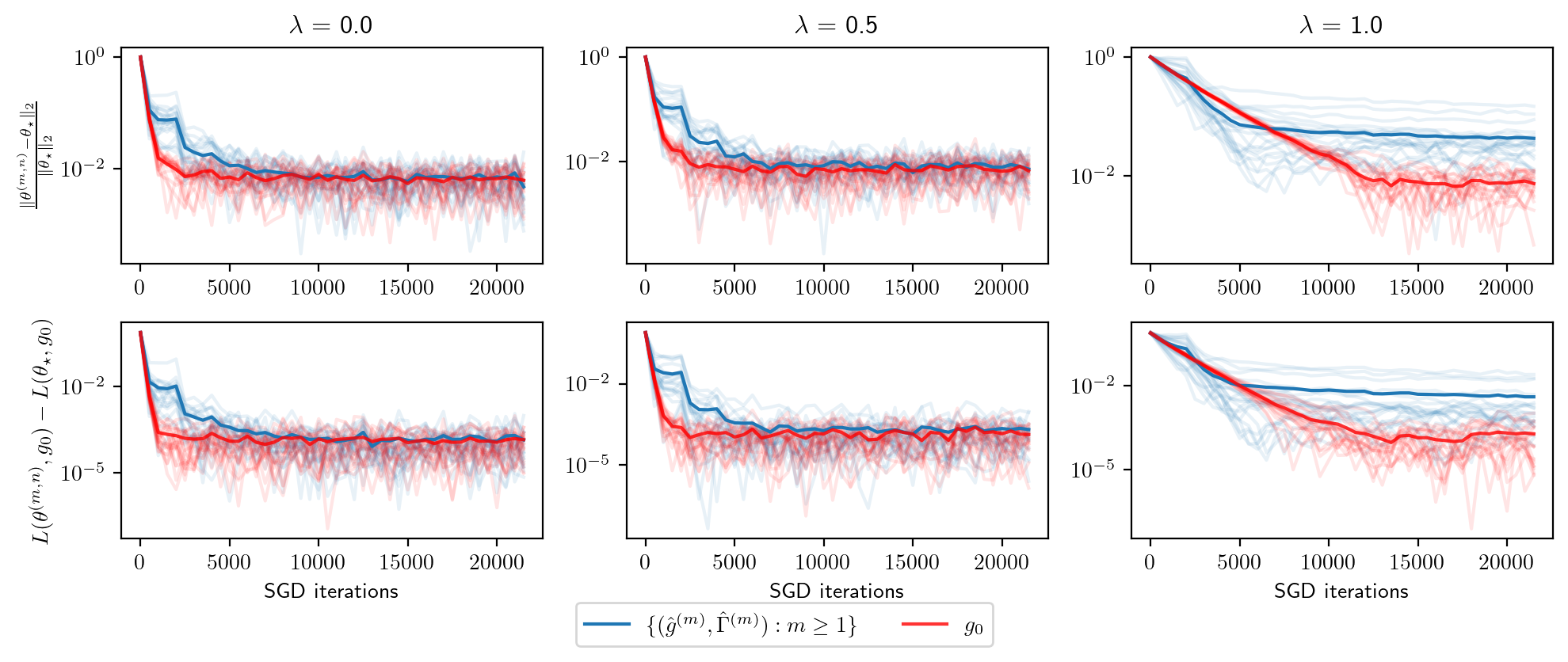}
    \caption{{\bf OSGD with the Nuisance Fitted on Simulated Stream Data.} The x-axis represents the SGD iteration. {\bf Top:} The y-axis measures the relative error. {\bf Bottom:} The y-axis measures the risk.}
    \label{fig: osgd stream}
\end{figure}

\subsection{Real Data Analysis}\label{sec: real data}
We consider the Diabetes 130-Hospitals Dataset \citep{diabetes_130-us_hospitals_for_years_1999-2008_296} as the real dataset example. We use six of these features as covariates, which are summarized in \Cref{tab:real data}. We take the binary feature ``change'' as the input $X \in \{0, 1\}$ and take the rest five features as the control $W \in \R^5$. 

\begin{table}[ht]
    \centering

    \begin{adjustbox}{max width=0.9\linewidth}
    \renewcommand{\arraystretch}{1.4}
    \begin{tabular}{cc}
    \toprule
        {\bf Feature} & {\bf Description}\\
        \midrule
        change & Indicates if there was a change in diabetic medications.\\
        time\_in\_hospital & Integer number of days between admission and discharge.\\
        num\_lab\_procedures & Integer number of lab tests performed during the encounter. \\
        num\_procedures & Integer number of procedures (other than lab tests) performed during the encounter. \\
        num\_medications & 	
Integer number of distinct generic names administered during the encounter. \\
        number\_diagnoses & Integer number of diagnoses.\\
        \bottomrule
    \end{tabular}
    \end{adjustbox}
    \vspace{6pt}
    \caption{Features used for real data analysis.}
    \label{tab:real data}
\end{table}

\subsubsection{Synthetic outcomes}
To evaluate the performance of our proposed methods, we use the synthetic outcome instead of a real outcome to examine the performance of our proposed methods. Using the synthetic outcome is common in causal inference; see \citet[Sec. 4.1]{nie2021quasi}. In this real data analysis, we generate outcome according to the following partially linear model:
\begin{align}
    Y &= \tilde \theta_0 \cdot X + \tilde \alpha_0(W) + 0.5\times \epsilon \label{real data eq Y},\\
    U &= \tilde \alpha_0(W) + 0.5 \times \xi \label{real data eq U},
\end{align}
where $\tilde \theta_0  = -1$, $\epsilon \sim \calN(0,1)$ and $\xi \sim \calN(0,1)$ are independent noises, and $\tilde \alpha_0: \R^5 \mapsto \R$ satisfies that for $w = (w_1, \dots, w_5)$,
\begin{align*}
    \tilde \alpha_0(w) = 0.5\times \cos\p{5^{-1}\sum_{i=1}^5 w_i} + 0.5\times \sin\p{5^{-1}\sum_{i=1}^5 w_i}.
\end{align*}
Similar to \Cref{subsec: sim SGD oracle}, we have $\theta_\star = \tilde \theta_0$ in this case.

\subsubsection{Real Data Results}
\myparagraph{Setup} We consider the same three stochastic gradient oracles as \Cref{subsec: sim SGD oracle} and the same two nuisance estimation methods as \Cref{subsec: sim nuisance estimation} except that we use logistic regression on full batch data and SGD of the logistic loss on stream data for estimating $\E[X\mid W]$. The setup of SGD using prefitted nuisances for this real data analysis is the same as \Cref{subsec: sim results}. For nuisance estimated using stream data, we update nuisances for 100 iterations after every 500 target SGD iterations.

\myparagraph{Evaluation} Since the true nuisances $\E[X\mid W]$ and $\E[Y \mid W]$ are unknown, we evaluate the performance of nuisance estimation $\hg\pow{m} = (\hat g_{Y}\pow{m}, \hat g_{X}\pow{m})$ for the orthogonal loss by the mean squared error:
\begin{align}\label{eq mse 1}
    \text{MSE}_1(\hat g^{(m)}) =  \max\br{\Ex{\prob}{(\hat g_{Y}\pow{m}(W) - Y)^2}, \Ex{\prob}{(\hat g_{X}\pow{m}(W) - X)^2}}.
\end{align}
We adopt the nuisance estimation error $\gnorm{\hg\pow{n} - g_0}$ defined in \eqref{eq: sim gnorm non-ortho} as the nuisance evaluation for non-orthogonal loss due to the synthetic outcome, where now $g_0 = \tilde \alpha_0$. For the operator estimation $\hat \Gamma\pow{m}: g \mapsto \E[\hg_X\pow{m}(W) g(W)]$, evaluate its performance by the mean squared error:
\begin{align}\label{eq mse 2}
    \text{MSE}_2(\hat \Gamma^{(m)}) =  \Ex{\prob}{(\hg_X\pow{m}(W) - X)^2}.
\end{align}

\myparagraph{Results using nuisances fitted on full-batch data}
We first estimate the target using prefitted nuisances and operator. The estimation errors of nuisances and the operator using full-batch real data are shown in \Cref{fig: real nuisance estimation error}, which suggests that the estimation of $\tilde \alpha_0$ converges to zero due to our design while there exists obvious bias for estimating the nuisance $(g_{0,X}, g_{0,Y})$ and the orthogonalizing operator $\Gamma_0$ possibly due to model misspecification.

\begin{figure}
    \centering
    \includegraphics[width=\linewidth]{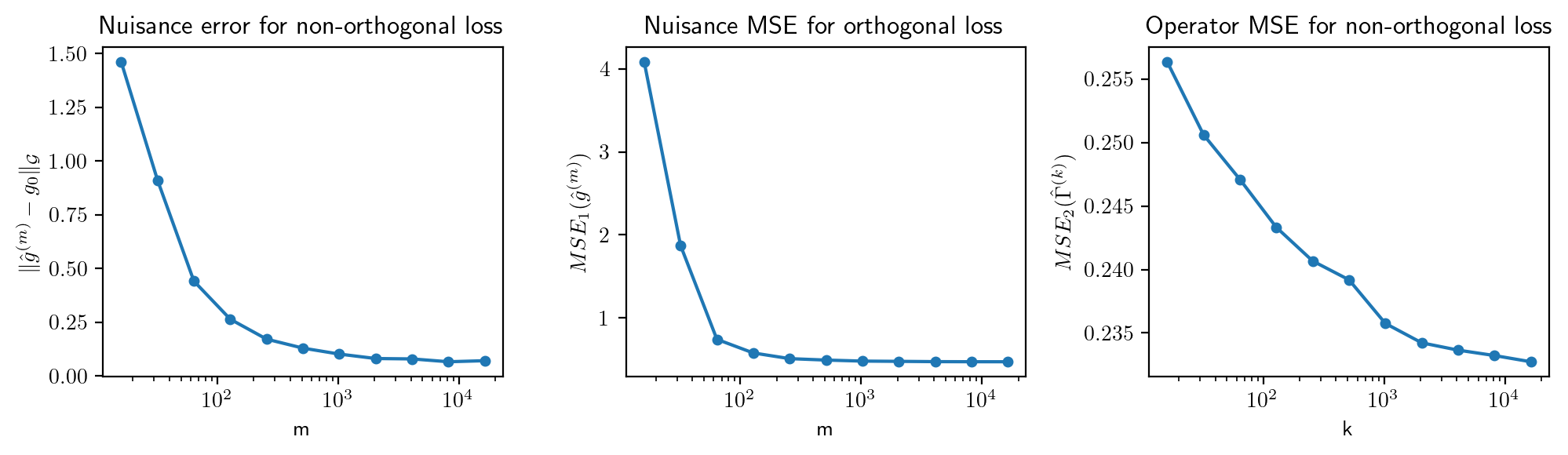}
    \caption{{\bf Nuisance and Orthogonalizing Operator Fitted on Full-Batch Real Data.} The x-axis displays the
sample size of data used to estimate the nuisance and operator. {\bf Left.} The y-axis measure the nuisance error defined in \eqref{eq: sim gnorm non-ortho}. {\bf Middle.} The y-axis measure the nuisance estimation MSE defined in \eqref{eq mse 1}. {\bf Right.} The y-axis measure the operator estimation MSE defined in \eqref{eq mse 2}.}
    \label{fig: real nuisance estimation error}
\end{figure}

The performances of SGDs using prefitted nuisances and stochastic gradient oracles \eqref{sim ortho score}, \eqref{sim nonortho score}, and \eqref{sim noscore} are shown in \Cref{fig: real sgd batch}. Overall, the relative error and the risk are small when well estimated nuisances are used. In addition, both relative errors and risks become smaller when we use more samples to estimate nuisances for the orthogonal loss and the non-orthogonal loss.

\begin{figure}
    \centering
    \includegraphics[width=\linewidth]{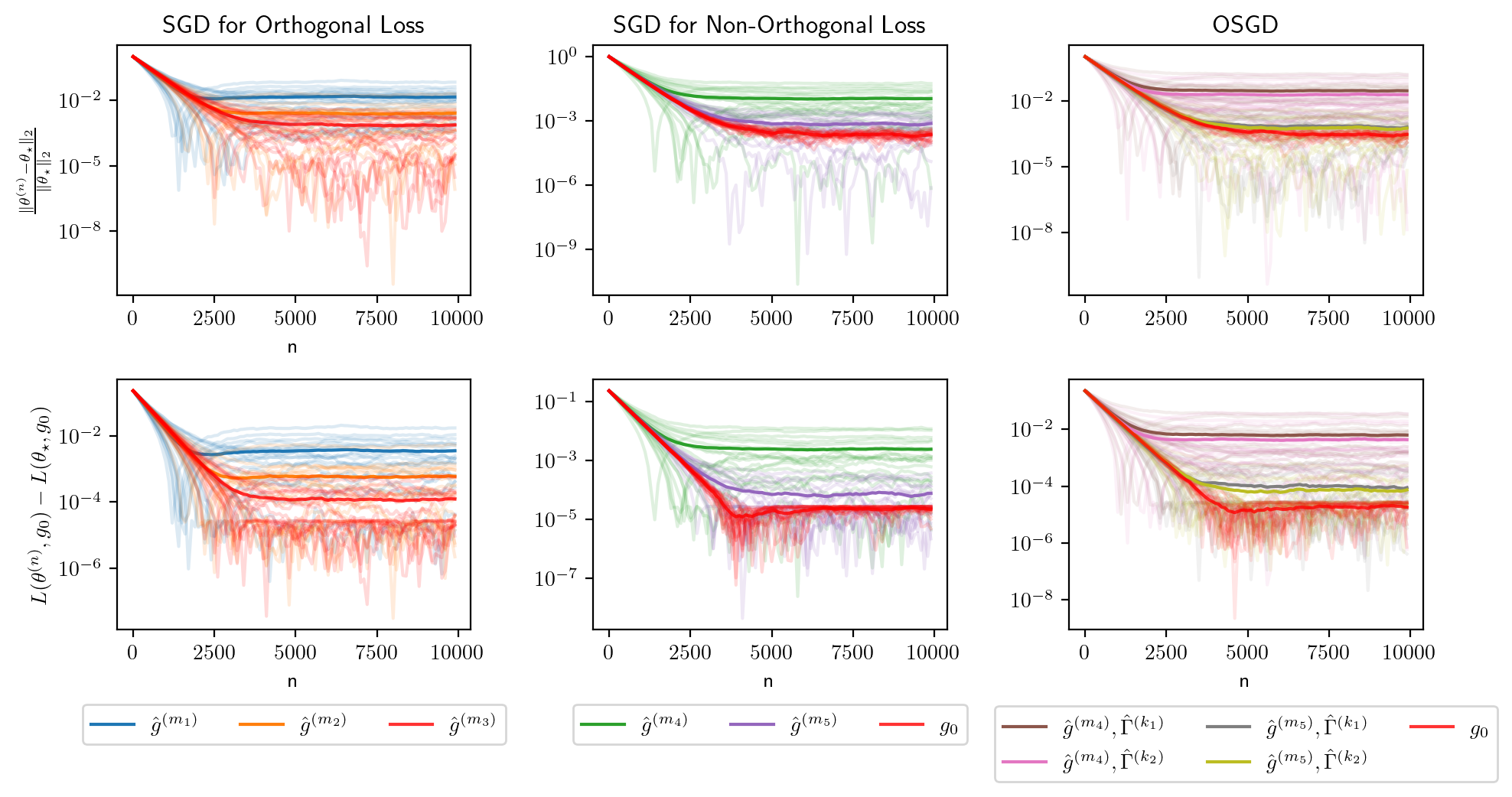}
    \caption{{\bf Stochastic Gradients with Nuisance Fitted on Full-Batch Real Data.} Here, $m_1 = 32$, $m_2 = 64$, $m_3 = 128$, $m_4 = 8$, $m_5 = 128$, $k_1 = 32$ and $k_2 = 128$. The x-axis represents the SGD iteration using corresponding stochastic gradient. {\bf Top:} The y-axis measures the relative error. {\bf Bottom:} The y-axis measures the risk.}
    \label{fig: real sgd batch}
\end{figure}

\myparagraph{Results using nuisances fitted on stream data}
We then estimate the target by interleaving the nuisance and target updates. Here, Both the nuisance and the operator are learned using the same data stream and the results are shown in \Cref{fig: real nuisance error stream}. Compared with \Cref{fig: real nuisance estimation error}, nuisances estimated using stream data converges similarly.
\begin{figure}
    \centering
    \includegraphics[width=\linewidth]{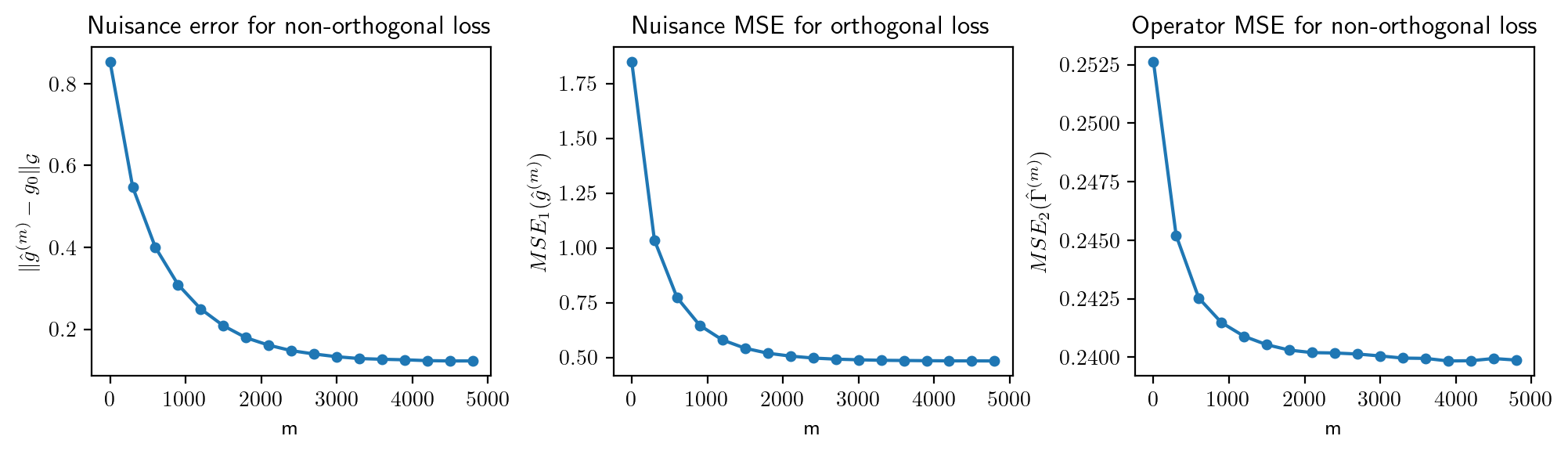}
    \caption{{\bf Estimation Errors of Nuisance and Orthogonalizing Operator Fitted on Stream Data.} The x-axis displays the
sample size of data used to estimate the nuisance and operator. {\bf Left.} The y-axis measure the nuisance error defined in \eqref{eq: sim gnorm non-ortho}. {\bf Middle.} The y-axis measure the nuisance estimation MSE defined in \eqref{eq mse 1}. {\bf Right.} The y-axis measure the operator estimation MSE defined in \eqref{eq mse 2}.}
    \label{fig: real nuisance error stream}
\end{figure}

The performances of SGDs by interleaving nuisance and target updates with stochastic gradient oracles \eqref{sim ortho score}, \eqref{sim nonortho score}, and \eqref{sim noscore} are shown in \Cref{fig: real sgd orthogonal stream}. The figure on the left in \Cref{fig: real sgd orthogonal stream} shows that the target estimation has small relative error using the estimated nuisance sequence $\{\hat g\pow{m}:m \geq 1\}$. The figure in the middle suggests that there is still some bias for the target estimation while this bias is negligible. The figure on the right shows the performance of OSGD, where the relative error of OSGD using the estimated nuisance sequence is similar to OSGD using the true nuisance, which aligns with \Cref{theorem debiased SGD convergence rate}.

\begin{figure}
    \centering
    \includegraphics[width=\linewidth]{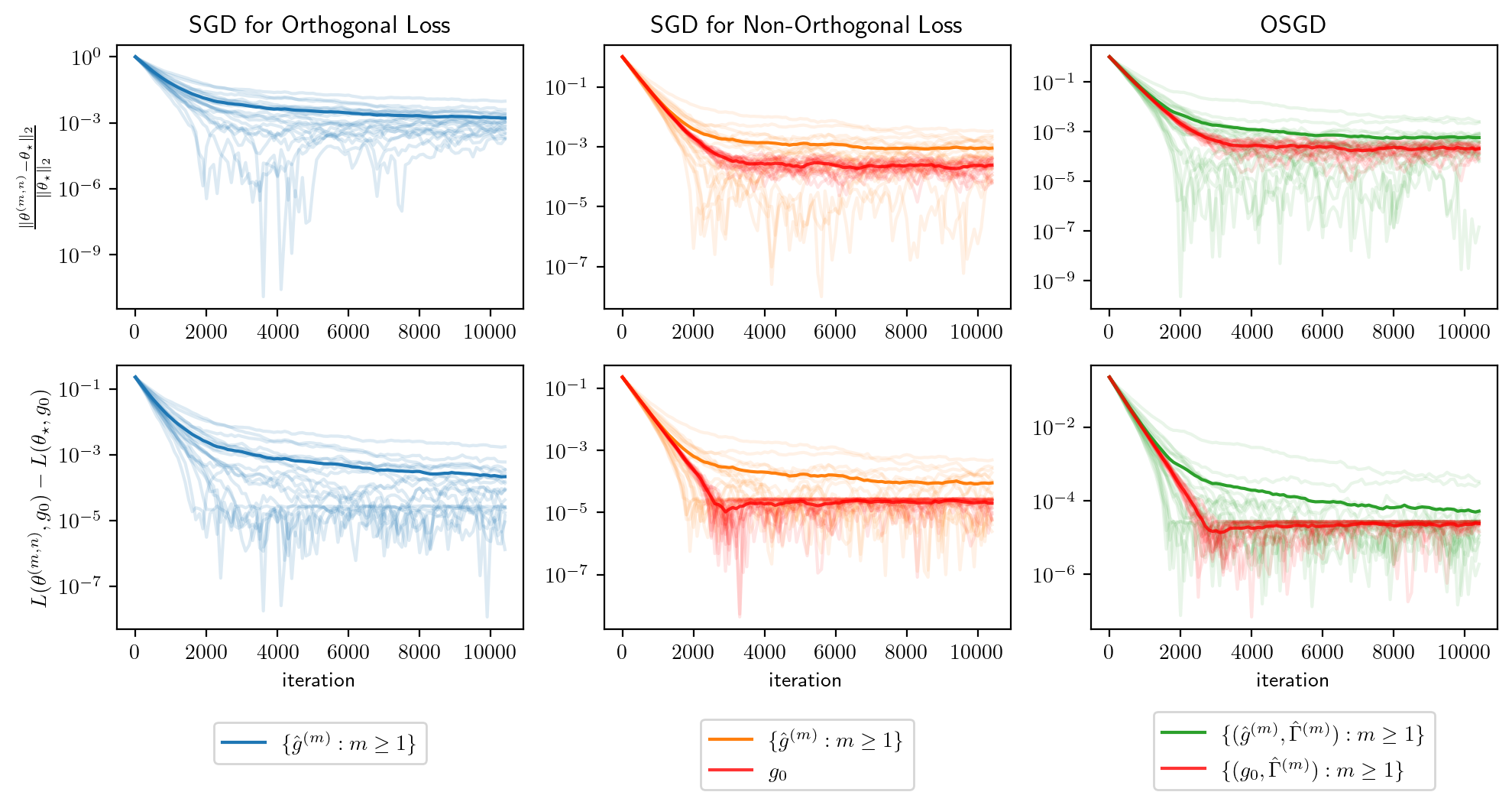}
    \caption{{\bf Stochastic Gradients with Nuisance Fitted on Real Stream Data.} The x-axis represents the SGD iteration. {\bf Top:} The y-axis measures the relative error. {\bf Bottom:} The y-axis measures the risk.}
    \label{fig: real sgd orthogonal stream}
\end{figure}

\clearpage

\section{Extension to SGD Variants}\label{appx:ASGD}
In this section, we discuss strategies for analyzing other variants of SGD under nuisances. In \Cref{sec: asgd}, we discuss the relationship between SGD with momentum and averaged SGD and provide a convergence analysis example of the averaged SGD. In \Cref{sec: adam}, we discuss Adam as a generalization of SGD with momentum and the difficulties to analyze the convergence rate of Adam.

\subsection{SGD with Momentum and Averaged SGD}\label{sec: asgd}

For the gradient oracle sequence $S\pow{n}$, SGD with momentum following the description of \citet{li2022onthelast} can be expressed as
\begin{align}
    m\pow{n+1} = \beta_n m\pow{n} + S\pow{n} \text{ and } \bar{\theta}\pow{n+1} = \bar{\theta}\pow{n} - \alpha_n m\pow{n},
\end{align}
where $\bar{\theta}\pow{n}$ is the SGD estimation sequence, $m\pow{n}$ is the momentum sequence, and $(\alpha_n)_{n \geq 0}$ and $(\beta_n)_{n \geq 0}$ can be any positive sequence. The following example shows that the averaged SGD is a special case of SGD with momentum.

\begin{examplebox}
    \textbf{Example 5} (Averaged SGD). Let $\beta_n = 1/n$ and $\alpha_n = \eta(1 - \beta_{n+1})$ for all $n \geq 1$. The momentum updates implied by this sequence are
\begin{align*}
    m\pow{n+1} = \frac{1}{n} m\pow{n} + S\pow{n} \text{ and } \bar{\theta}\pow{n+1} = \bar{\theta}\pow{n} - \eta\p{1-\frac{1}{n+1}} m\pow{n},
\end{align*}
    which implies that $\bar{\theta}\pow{n+1}$ is the averaged SGD such that
    \begin{align}\label{asgd}
        \bar{\theta}\pow{n+1} = \frac{1}{n+1}\sum_{t=0}^n \theta\pow{t}.
    \end{align}
\end{examplebox}

Example 5 demonstrates that the convergence rate of SGD with momentum can be analyzed in the same way as averaged SGD. While it is not the focus of this paper, we provide a convergence result of the averaged SGD based on the analysis of \cite{defossez2015averaged}.

\begin{proposition}[Convergence rate of averaged SGD] \label{prop asgd}
Consider the partially linear model and the non-orthogonal loss $\ell(\theta, g; z)$ in \Cref{subsubsec: non-ortho plm}. Define $\mc{D}_n = (Z_1, \ldots, Z_n)$, sampled from the product measure $\prob^n$. Choose the gradient oracle $S\pow{n}$ to be the score $\score(\theta, \hat g; Z_n)$ where $\hat g$ is estimated independently of $\calD_n$. Let $\bar \sgd\pow{n}$ be the averaged SGD defined in \eqref{asgd}. Suppose the same assumptions as \Cref{lemma: example non-ortho PLM}. If $0<\eta<\eta_{\max}$, then
    \begin{align*}
        \Ex{\prob}{\norm{\bar \sgd\pow{n} - \theta_\star}_2^2} \lesssim \frac{1}{n} + \gnorm{\hat g - g_0}^2,
    \end{align*}
where $\eta_{\max} = \sup\{\eta>0: \tr{A^\top \E_{\prob}[XX^\top] A} - \eta\Ex{\prob}{(X^\top A X)^2} > 0, \forall A \in \calS(\R^d)\}$ and $\calS(\R^d)$ is the set of all $d \times d$ symmetric matrices.
\end{proposition}

Before we prove \Cref{prop asgd}, recall the example of non-orthogonal loss for the partially linear model in \Cref{subsubsec: non-ortho plm}, where $Z = (X, W, Y) \sim \prob$ satisfies 
\begin{align}
    Y = \ip{\theta_0, X} + g_0(W) + \epsilon, \quad \Ex{\prob}{\epsilon \mid X, W} = 0.
\end{align}
The target parameter $\theta_\star = \argmin_{\theta \in \Theta} \Ex{\prob}{\ell(\theta, g; Z)}$ where $\ell$ is the non-orthogonal loss defined as
\begin{align}
    \loss(\theta, g; z) = \frac{1}{2}[y  - g(w) - \ip{\theta, x}]^2,
\end{align}
By \cref{lemma: example non-ortho PLM}, we have $\theta_\star = \theta_0$. The stochastic gradient oracle for this non-orthogonal loss is 
\begin{align*}
    \score(\theta, g; z) = -X(y  - g(w) - \ip{\theta, x}),
\end{align*}
and the SGD iteration is defined by $\sgd\pow{0} \in \Theta$ and
\begin{align}\label{appx asgd: eq sgd}
    \sgd\pow{n} = \sgd\pow{n-1} - \eta \score(\theta, \hat g; Z_{n-1}) = \sgd\pow{n-1} + \eta X_{n-1}(Y_{n-1}  - \hat g(W_{n-1}) - \ip{\theta, X_{n-1}}),
\end{align}
where $\hat g \in \Gr$ is any nuisance estimator independent of $\{Z_i\}_{i=1}^n$. 
Note that \eqref{appx asgd: eq sgd} can be written as
\begin{align*}
    \sgd\pow{n} - \theta_\star &= \p{I - \eta X_nX_n^\top}\sgd\pow{n-1} + \eta X_n\p{Y_n - \hat g\p{W_n}} - \theta_\star\\
    &=\p{I - \eta X_nX_n^\top}(\sgd\pow{n-1}-\theta_\star) + \eta X_n\p{Y_n - \hat g\p{W_n}-X_n^\top \theta_\star}\\
    &=\p{I - \eta X_nX_n^\top}(\sgd\pow{n-1}-\theta_\star) + \eta X_n\epsilon_n - \eta X_n\p{\hat g\p{W_n} - g_0\p{W_n}}.
\end{align*}
Let $\beta\pow{n} = \sgd\pow{n} - \theta_\star$, $r_n = \hat g\p{W_n} - g_0\p{W_n}$, and 
\begin{align*}
    M_{k,j} = \p{\prod_{i=k+1}^j \p{I - \eta X_iX_i^\top}}^\top \in \R^{d\times d}.
\end{align*}
By recursion, we have
\begin{align*}
    \beta\pow{n} &= \p{I - \eta X_nX_n^\top}\beta\pow{n-1} + \eta X_n\epsilon_n - \eta X_nr_n\\
    &= M_{0,n}\beta\pow{0} + \eta \sum_{k=1}^n M_{k,n} X_k\epsilon_k - \eta \sum_{k=1}^n M_{k,n} X_k r_k.
\end{align*}
Let $\bar \beta\pow{n} = \bar \sgd\pow{n} - \theta_\star = \p{n+1}^{-1}\sum_{j=0}^n \beta\pow{j}$, we have
\begin{align*}
    \bar \beta\pow{n} &=  \frac{1}{n+1}\sum_{j=0}^nM_{0,j}\beta\pow{0} + \frac{\eta}{n+1}\sum_{k=1}^n\p{\sum_{j=k}^n M_{k,j}}\p{X_k\p{\epsilon_k - r_k} + \E\sbr{X_kr_k}} \\
    &\quad - \frac{\eta}{n+1}\sum_{k=1}^n\p{\sum_{j=k}^n M_{k,j}}\E\sbr{X_kr_k}.
\end{align*}
In the above formula, first two terms are usually interpreted as the bias term and the variance term under the true nuisance, respectively according to \cite{defossez2015averaged}, and the last term can be viewed as the error term caused by the nuisance estimation. 

To analyze the bias term and the variance term, we adopt the notations of \cite{defossez2015averaged} for matrices and operators. First, Define $H = \E_{\prob}[XX^\top]$. Let $H_L$ (resp.~$H_R$) be the matrix operator representing left multiplication (resp.~right multiplication) by $H$, and $T$ be the linear operator such that for any square matrix $M \in \R^{d \times d}$, $TM = HM + MH - \eta \Ex{\prob}{(X^\top M X)XX^\top}$. Let $\rho = \max\{\Vert I - \eta H \Vert_{\text{op}}, \Vert I - \eta T \Vert_{\text{op}}\}$ where the operator norm $\Vert \cdot\Vert_{\text{op}}$ is defined as the largest singular value. Finally, let $\eta_{\max}$ be the same as in \Cref{prop asgd}. With definitions above, the asymptotic covariances of the bias  and the variance follow directly from Theorems 1 and 2 of \citet[Appx.~3]{defossez2015averaged}.

\begin{lemma}[Asymptotic covariance of the bias]\label{prop asymp bias cov}
   Let $\Xi_0=\Ex{\prob}{\beta\pow{0} {\beta\pow{0}}^T}$. If $0<\eta<\eta_{\max }$, then
\begin{align*}
    \Ex{\prob}{B_n B_n^\top}=\frac{1}{n^2 \eta^2}\left(H_L^{-1}+H_R^{-1}-\eta I\right)\left(T^{-1} \Xi_0\right)+O\left(\frac{\rho^n}{n}\norm{\Xi_0}_F\right),
\end{align*}
where $B_n = \frac{1}{n+1}\sum_{j=0}^nM_{0,j}\beta\pow{0}$.
\end{lemma}

\begin{lemma}[Asymptotic covariance of the variance]\label{prop asymp var cov}
    Let $\Sigma_0=\Var(X_n\p{\epsilon_n - r_n} + \E\sbr{X_nr_n})$. If $0<\eta<\eta_{\max }$, then
\begin{align*}
    \Ex{\prob}{V_n V_n^\top}&=\frac{1}{n}\left(H_L^{-1}+H_R^{-1}-\eta I\right) T^{-1} \Sigma_0 \\
    &\quad-\frac{1}{\eta n^2}\left(H_L^{-1}+H_R^{-1}-\eta I\right)(I-\eta T) T^{-2} \Sigma_0 +O\left(\frac{\rho^n}{n}\norm{\Sigma_0}_F\right),
\end{align*}
where $V_n = \frac{\eta}{n+1}\sum_{k=1}^n\p{\sum_{j=k}^n M_{k,j}}\p{X_k\p{\epsilon_k - r_k} + \E\sbr{X_kr_k}}$.
\end{lemma}

In fact, the convergence rate of averaged SGD depends on $\tr{B_nB_n^\top}$ and $\tr{V_nV_n^\top}$. When $\rho<1$, \Cref{prop asymp bias cov} demonstrate that $\tr{B_nB_n^\top}$ is of order $n^{-2}$, while \Cref{prop asymp var cov} shows that $\tr{V_nV_n^\top}$ is of order $n^{-1}$, which is reasonable due to the randomness of the noise $\epsilon_k - r_k$.

For the error term, we can analyze it in a similar way to the bias term. Let $\Delta = \E_{\prob}\sbr{X_nr_n}$ and 
\begin{align*}
    E_n = \frac{\eta}{n+1}\sum_{k=1}^n\p{\sum_{j=k}^n M_{k,j}}\Delta.
\end{align*}
Note that by Jensen's inequality,
\begin{align*}
    \Ex{\prob}{\norm{E_n}_2^2} &\leq  \Ex{\prob}{\frac{\eta^2}{n}\sum_{k=1}^n\left\Vert\p{\sum_{j=k}^n M_{k,j}}\Delta\right\Vert_2^2 }\\
    &= \frac{\eta^2}{n}\sum_{k=1}^n\tr{\Ex{\prob}{\p{\sum_{j=k}^n M_{k,j} \Delta}\p{\sum_{j=k}^n M_{k,j} \Delta}^\top}}.
\end{align*}
It is clear to see that the asymptotic covariance of $\Delta_{k,n} := \sum_{j=k}^n M_{k,j} \Delta$ is of the same form as the bias term in \Cref{prop asymp bias cov}. Let $G_0 = \Ex{\prob}{\Delta\Delta^\top}$ and we have
\begin{align*}
    \Ex{\prob}{\Delta_{k,n} \Delta_{k,n}^\top} = \frac{1}{\eta^2}\left(H_L^{-1}+H_R^{-1}-\eta I\right)\left(T^{-1} G_0\right)+O\left((n-k)\rho^{n-k}\norm{G_0}_F\right).
\end{align*}
Thus, the trace of the covariance summation over $k=1, \dots, n$ satisfies
\begin{align}\label{eq asgd error term}
    \frac{\eta^2}{n}\sum_{k=1}^n\tr{\Ex{\prob}{\Delta_{k,n} \Delta_{k,n}^\top}} = \tr{(H_L^{-1}+H_R^{-1}-\eta I)(T^{-1} G_0)} + O\p{\frac{\eta^2}{n}\sum_{k=0}^{n-1} k\rho^{k}\norm{G_0}_F}.
\end{align}

Gathering the bias term, the variance term, and the error term, we are now ready to proof \Cref{prop asgd}.
\begin{proof}[Proof of \Cref{prop asgd}]
By Lemma 1 of \cite{defossez2015averaged}, $0< \rho < 1$ when $0< \eta < \eta_{\max}$. Suppose that $\norm{X}_{\infty} \leq C_X$ almost surely. By Jensen's inequality we have that 

\begin{align*}
\norm{G_0}_2 = \norm{\Ex{\prob}{X_nr_n}}_2^2\leq C_X^2\Ex{\prob}{r_n}^2 \leq C_X^2\Ex{\prob}{r_n^2} = C_X^2\gnorm{\hat g - g_0}^2.
\end{align*}
Note that
\begin{align*}
    \sum_{k=0}^{n-1} k\rho^{k} = \rho \frac{\d }{\d \rho} \p{\sum_{k=0}^{n-1} \rho^{k}} = \rho \frac{\d }{\d \rho} \p{\frac{1 - \rho^n}{1 - \rho}} = \frac{\rho - n \rho^n + (n-1)\rho^{n+1}}{(1 - \rho)^2} = O(1).
\end{align*}
By \Cref{prop asymp bias cov}, \Cref{prop asymp var cov}, and \eqref{eq asgd error term}, we have
    \begin{align*}
    \Ex{\prob}{\norm{\bar \sgd\pow{n} - \theta_\star}_2^2} &\lesssim \Ex{\prob}{\norm{B_n}_2^2} + \Ex{\prob}{\norm{V_n}_2^2} + \Ex{\prob}{\norm{E_n}_2^2}\\
    &\lesssim \tr{\Ex{\prob}{B_n B_n^\top}} + \tr{\Ex{\prob}{V_n V_n^\top}} + \frac{\eta^2}{n}\sum_{k=1}^n\tr{\Ex{\prob}{\Delta_{k,n} \Delta_{k,n}^\top}}\\
    &\lesssim \frac{1}{n} + \gnorm{\hat g - g_0}^2.
\end{align*}
\end{proof}

\subsection{Adam}\label{sec: adam}

The primary updates for Adam under nuisance estimate $\hg$ are given by the following recursive equations. Below, we let $i \in \br{1, \ldots, d}$ denote a particular dimension of the finite-dimensional parameter of interest. Following the description of \citet{defossez2022asimple}, for the gradient oracle sequence $S\pow{n}$ the Adam generates the target estimator $\tilde \theta\pow{n}$ as below:
\begin{align}
    m\pow{n}_i &= \beta_1 m\pow{n-1}_i + S_{i}\pow{n}\label{adam eq 1}\\
    v\pow{n}_i &= \beta_2 v\pow{n-1}_i + \p{S_{i}\pow{n}}^2\label{adam eq 2}\\
    \tilde \theta\pow{n}_i &= \tilde \theta\pow{n-1}_i - \eta \frac{m\pow{n}_i}{\sqrt{\epsilon + v\pow{n}_i}}, \label{adam eq 3}
\end{align}
where $\beta_2 \in (0, 1]$, $\beta_1 \in [0, \beta_2]$ are the momentum and variance parameters, $m\pow{n}, v\pow{n} \in \R^d$ are the momentum and variance sequences, and $\epsilon > 0$ is a numerical stability parameter. Adam differs from the SGD with momentum by adding a variance sequence $v\pow{n}$. When $v\pow{n}$ is chosen to be a constant, then \eqref{adam eq 1} and \eqref{adam eq 3} would reduce to the special case of SGD with momentum where $\beta_n$ and $\alpha_n$ are constant. 

The analysis of Adam is often done in the case of smooth non-convex optimization, in which it is shown that the gradient of the objective tends to zero \citep{ward2020adagrad, defossez2022asimple}. Specifically, \citet{defossez2022asimple} consider a momentum-free Adam ($\beta_1 = 0$) to analyze the essential ingredients that differ from momentum: the variance pre-conditioning and element-wise updates, which suggests that under the true nuisance, i.e., $S\pow{n} = \score(\theta\pow{n}, g_0; Z_n)$ for an \iid sample $\br{Z_i}_{i=1}^n \sim \prob^n$, the convergence result of Adam satisfies
\begin{align*}
    \Ex{\prob^n}{\norm{S_\theta(\tilde \theta\pow{n}, g_0)}_2^2} \lesssim \frac{1}{\sqrt{n}} \p{1+  \log\p{\frac{1}{\epsilon}}}.
\end{align*}
Note that this result is not comparable to our convergence criterion (in terms of iterations), which differs non-trivially from a stationarity or function value analysis. While the convergence of Adam without nuisance has been studied in the literature, it still remains unclear that how to analysis Adam under an estimated nuisance $\hg$ and what should be the nuisance effect on the gradient norm criterion.

\end{document}